%% file: template.tex
\RequirePackage{fix-cm}
\documentclass[final]{svjour3}                     
\smartqed  
\usepackage{graphicx}
\usepackage{booktabs}
\usepackage{lipsum}
\usepackage{amsfonts}
\usepackage{graphicx}
\usepackage{epstopdf}
\usepackage{enumerate}
\usepackage{enumitem}
\usepackage{cite}
\usepackage{subcaption}

\usepackage{tabularx}
\usepackage{arydshln}
\usepackage{multirow}
\usepackage{caption}
\usepackage{microtype}

\usepackage{algorithm}
\usepackage{algpseudocode}
\usepackage{amsmath}
\usepackage{hyperref}
\hypersetup{
    colorlinks=true,
    linkcolor=blue,
    citecolor=blue,
    urlcolor=blue,
}
\allowdisplaybreaks[1]

\usepackage{mathtools}

\graphicspath{{pics/}{figs/}}

\newcolumntype{x}[1]{>{\centering\let\newline\\\arraybackslash\hspace{0pt}}p{#1}}

\usepackage{latexsym}
\usepackage{amssymb}
\begin{document}

\title{Free-Boundary Quasiconformal Maps via a Least-squares Operator in Diffeomorphism Optimization
}

\titlerunning{Optimization for free boundary diffeomorphism}        
\author{Zhehao Xu         \and
        Lok Ming Lui$^{*}$\thanks{Corresponding author: lmlui@math.cuhk.edu.hk} 
}


\institute{Zhehao Xu \at
              Department of Mathematics, The Chinese University of Hong Kong, Hong Kong, China \\
              \email{zhxu@math.cuhk.edu.hk}           
           \and
           Lok Ming Lui \at
              Department of Mathematics, The Chinese University of Hong Kong, Hong Kong, China \\
              \email{lmlui@math.cuhk.edu.hk}
}
\date{Received: date / Accepted: date}

\maketitle
\begin{abstract}

Free-boundary diffeomorphism optimization, an important and widely occurring task in geometric modeling, computer graphics, and biological imaging, requires simultaneously determining a planar target domain and a locally bijective map with well-controlled distortion. We formulate this task through the least-squares quasiconformal (LSQC) operator and establish key structural properties of the LSQC minimizer, including well-posedness under mild conditions, invariance under similarity transformations, and resolution-independent behavior with stability under mesh refinement. We further analyze the sensitivity of the LSQC solution with respect to the Beltrami coefficient, establishing stability and differentiability properties that enable gradient-based optimization over the space of Beltrami coefficients. To make this differentiable formulation practical at scale and to facilitate the optimization process, we introduce the Spectral Beltrami Network (SBN), a multiscale mesh–spectral surrogate that approximates the LSQC solution operator in a single differentiable forward pass. This yields SBN-Opt, an optimization framework that searches over admissible Beltrami coefficients and pinning conditions to solve free-boundary diffeomorphism objectives with explicit distortion control. Extensive experiments on equiareal parameterization and inconsistent surface registration demonstrate consistent improvements over traditional numerical algorithms.


\end{abstract}
\input{sections/sec1-introduction}
\input{sections/sec4-math_background}

\input{sections/sec5-LSQC_properties}
\input{sections/sec5b-LSQC_sensitivity}
\input{sections/sec6-Spectral_Beltrami_Network}

\input{sections/sec7-SBN-Opt}
\input{sections/sec8-implementation}

\input{sections/sec9-experiment}
\input{sections/sec10-conclusion}
\bibliographystyle{spmpsci}      
\bibliography{references}

\end{document}

%% file: sections/sec1-introduction.tex
\section{Introduction}

Diffeomorphic mappings are indispensable in fields such as geometry processing, computer graphics, and medical imaging, where one must deform surfaces in a topology-preserving way so that measurements, annotations and features remain comparable after deformation. Practical examples abound in parameterization and registration problems:  population cartograms rescale countries to equalize density\cite{choi2018density,lyu2024bijective}; surgeons “unfold” the cerebral cortex into a flat view for easier navigation\cite{buchsbaum2012neural}; and manufacturers flatten curved sheets to design cutting patterns for free-form parts\cite{pietroni2022computational}; atlas-based registration and medical image segmentation require invertible deformations to align anatomical structures or biological features\cite{vercauteren2009diffeomorphic,yang2017quicksilver}. In many scenarios, the task involves finding the desired planar representation of an open surface. Crucially, the shape of the planar image is often needed to optimize. The algorithm must infer the boundary shape so that the final map captures target properties such as area distribution, curvature, or feature alignment. We refer to this setting as the \textit{free-boundary diffeomorphism problem.}

Given a surface $\mathcal{M}$, a representative mathematical formulation is to find an optimal target domain $\Omega^*$ and an optimal mapping $f^*$ defined on $\mathcal{M}$:
\begin{equation}
    (f^*,\Omega^*) 
    = \textbf{argmin}_{\Omega,\; f:\mathcal{M}\rightarrow \Omega} E(f,\Omega)
    \quad \text{subject to } f \in \mathcal{C},
    \label{eq: formulation of optimization problem}
\end{equation}
where $E$ is a mapping objective and $\mathcal{C}$ denotes admissible maps.
In applications, bijectivity is essential for the map. Beyond bijectivity, the map is usually required to have \emph{well-controlled geometric distortion} so that it is physically plausible. 
Accordingly, the feasible set $\mathcal{C}$ should be understood as diffeomorphisms together with an explicit distortion constraint or regularization (e.g., bounds on local anisotropy), rather than arbitrary smooth bijections. This makes free-boundary diffeomorphism optimization challenging for two reasons: (i) \textbf{bijectivity} must be preserved but naive optimization over vertex coordinates or displacement fields is prone to foldovers; and (ii) \textbf{distortion control} must be enforced efficiently, ideally through a representation in which local distortion is explicit and can be constrained or penalized without projections or delicate step-size tuning. Therefore, our goal is to find a representation in which diffeomorphisms and their local distortion are encoded by an explicit quantity, and to optimize in that representation rather than in vertex coordinates. Specifically, we aim to identify a distortion-aware quantity that corresponds to diffeomorphisms, design algorithms to reconstruct the associated bijection from it, and then build an optimization framework directly over this quantity to control distortion efficiently and make free-boundary diffeomorphism optimization more stable and easier.
\paragraph{Related computational paradigms for diffeomorphic optimization.}
Diffeomorphism optimization problems arise pervasively in surface parameterization and deformable registration, and the computational task can be viewed as optimizing over a (discretized) space of topology-preserving maps, ideally with (i) \emph{bijectivity} and (ii) \emph{controlled distortion}. Existing methods can be grouped into two broad paradigms, distinguished by what is taken as the optimization variable: (A) direct optimization of the map itself (e.g., vertex positions or displacement fields), and (B) optimization of a generating quantity from which a diffeomorphism is obtained by construction (e.g., velocity fields, momenta, or spline generators). 
\paragraph{(A) Diffeomorphisms by directly optimizing the map.}
In this paradigm, one treats the map $f$ as the primary unknown and then augments the task loss objective with regularizations or designs mechanisms intended to avoid local foldovers and reduce geometric distortion. 
In surface parameterization, a common pipeline first computes a bijective embedding  and then improves the parameterization by minimizing a distortion functional subject to injectivity constraints \cite{jiang2017simplicial,su2020efficient,rabinovich2017scalable,smith2015bijective}. The computational process from the initial embedding to the final 2D parameterization is a typical diffeomorphism optimization problem. Some methods first
initialize using Tutte’s embedding or its variants and
then optimize a flip-prevented distortion energy. For example, Smith et al. \cite{smith2015bijective} combined a symmetric Dirichlet energy with a barrier preventing boundary self-intersection to reduce the distortion while ensuring no flips. \cite{jiang2017simplicial, su2020efficient} added an auxiliary ``scaffold'' domain so that global non-overlap can be reduced to a local injectivity condition and also used a distortion metric or a barrier function. Most works need to impose distortion energies and design tailored numerical optimization strategies as well as mesh postprocessing methods. Same ideas can be seen in the deformable image registration research \cite{liu2024finite,jia2021learning,balakrishnan2019voxelmorph}. Many deep registration methods predict a dense displacement field in one forward pass and use smoothness regularization to discourage irregular deformations. VoxelMorph \cite{balakrishnan2019voxelmorph} and VR-Net \cite{jia2021learning} pursued unsupervised registration via a similarity term plus deformation regularization. 

Optimizing the map is flexible and direct, but bijectivity and distortion control are usually achieved through auxiliary mechanisms (barriers, projections, penalty weights, line-search safeguards). This typically (i) yields sensitive hyperparameters, (ii) needs tailored optimization strategies and (iii) lacks direct control over geometric distortion (e.g., controlling local conformal distortion versus imposing energies on the singular values or Jacobian.)

\paragraph{(B) Diffeomorphisms by construction from generators.}
In this paradigm, one does not optimize the map $f$ (or displacement field) directly. Instead, one optimizes a generating quantity from which a diffeomorphism is reconstructed, typically through integrating an ODE/PDE-defined flow or exponentiating a stationary velocity field. Classical representatives include LDDMM and related geodesic-shooting formulations \cite{joshi2000landmark,zhang2017frequency,beg2005computing}, where the deformation is the flow of a velocity field; in practice the optimization can be posed either in terms of the velocity $v_t$ or in terms of its dual momentum variable. Demons-based methods \cite{vercauteren2009diffeomorphic, vercauteren2008symmetric} further reduced complexity by optimizing a stationary vector field $v$ and reconstructing the displacement field by squaring and scaling. Recent learning-based registration has adopted this generator viewpoint to obtain topology-preserving maps more reliably than direct displacement prediction. For example, Quicksilver \cite{yang2017quicksilver} predicted the momentum-parameterization of an LDDMM model in a patch-wise manner, aiming to approximate solutions of variational LDDMM optimization while retaining its diffeomorphic structure. Mok \cite{mok2020fast} proposed a fast symmetric diffeomorphic registration network that estimates forward and inverse transformations simultaneously, enabling fast computation of the inverse transformation. Han et al. \cite{han2023diffeomorphic} went a step further toward function-space parameterizations, which modeled a continuous velocity field using a neural network and then performed optimization in this generator space. 

While generator-based parameterizations offer a principled route to diffeomorphisms and often improve robustness under large deformations, they typically regulate deformation through smoothness or kinetic-energy priors on the generator (velocity/momentum/spline coefficients). As a consequence, geometric distortion of the resulting map is controlled at most indirectly. This motivates seeking a representation in which the optimized quantity corresponds directly to (and quantitatively controls) geometric distortion, which we introduce next via quasiconformal geometry.

Quasiconformal (QC) geometry offers a principled representation of distortion-controlled diffeomorphisms via the Beltrami coefficient $\mu$, which encodes the conformality deviation and satisfies the correspondence between the set $\{\mu:\|\mu\|_{\infty}<1\}$ and the group of orientation-preserving quasiconformal maps. Discrete conformal and quasiconformal mappings have been central in surface processing~\cite{gu2002computing,gu2004genus,kharevych2006discrete,lui2013texture}. Lui et al. developed the Linear Beltrami Solver (LBS)~\cite{lui2013texture}, which accommodates large deformations and guarantees bijectivity, and has been applied broadly in medical imaging and surface matching~\cite{lui2014teichmuller, choi2015flash,lyu2024spherical,lyu2024bijective}. A key limitation is that LBS typically requires explicit boundary correspondences, a requirement that is not applicable in free-boundary settings. To remove prescribed boundary constraints, Qiu and Lui introduced the weighted least-squares quasiconformal energy~\cite{qiu2019computing}, which allows a boundary-free formulation and has enabled tasks such as inconsistent surface registration~\cite{qiu2020inconsistent,choi2020shape} and partial welding~\cite{zhu2022parallelizable}. 

\paragraph{\textit{Methodology}} Our idea is to seek a distortion-related representation of the mappings, which belongs to the second category of approaches. Given a planar domain $U$, We build on the $\ell^2$ norm of the Beltrami equation deviation and propose the unweighted least-squares quasiconformal energy (LSQC)
\begin{equation}
E_{\text{LSQC}}(f,\mu) = \frac{1}{2}\int_{U}\|f_{\bar{z}}-\mu f_{z}\|^2\,dA.
\end{equation}
A key viewpoint in this paper is to treat LSQC not merely as an energy, but as a \emph{solution operator} that maps Beltrami coefficients to free-boundary homeomorphisms. Concretely, as shown in the Section \ref{sec: Properties of Least Squares Quasi-conformal Energy} for an admissible Beltrami coefficient $\mu$ with $\|\mu\|_{\infty}<1$, together with a minimal pinning condition (two pinned points to fix the similarity gauge and ensure non-triviality of the solution), LSQC yields a unique discrete map $f:U\to \mathbb{C}$ whose image domain $\Omega=f(U)$ is not prescribed but determined inherently by the Beltrami coefficient. Thus, we can view LSQC as an operator $\mathcal{F}$ that receives as input $\mu$,$p_i$ in the mesh as well as their associated destinations $q_i\in \mathbb{C}$, and outputs a piecewise linear mapping $f$. In particular, given the Beltrami coefficient induced by a (piecewise linear) homeomorphism and two-point pinning, the LSQC exactly reconstructs the map. This leads to the following modeling insight:

\smallskip
\noindent\textbf{Representation insight.}
\emph{Up to a similarity transformation, a free-boundary quasiconformal homeomorphism can be represented by its Beltrami coefficient $\mu$ together with the images of two pinned points.}
\smallskip 

\noindent We will establish that the LSQC solution map has the structure required for optimization over Beltrami coefficients.
\begin{enumerate}
    \item \textbf{Structure and well-posedness.} Under mild conditions (including minimal pinning), the LSQC solution is well-defined and unique; it is invariant under similarity transforms and exhibits resolution-independent behavior that allows stable transfer across discretization levels.
    \item \textbf{Sensitivity.} The LSQC minimizer depends stably on $\mu$ (in particular, it satisfies Lipschitz-type stability bounds under perturbations of $\mu$), and it is differentiable with respect to $\mu$ under standard regularity assumptions. These results justify gradient-based optimization over Beltrami coefficients.
\end{enumerate}
    
Therefore, instead of optimizing directly over vertex coordinates (where injectivity is fragile), after obtaining a conformal parameterization  $U=\phi(\mathcal{M})$ of the surface $\mathcal{M}$, we can reformulate the problem ~\eqref{eq: formulation of optimization problem} as an optimization problem over the structured Beltrami space,
\begin{equation}
    f^* = \textbf{argmin}_{\mu,p_i \in U.q_i \in \mathbb{C}, f : U \rightarrow \mathbb{C}} E(f,\mu, p_i,q_i) \\ \text{ subject to } f,\mu \in \mathcal{C},
\end{equation}
where $\mathcal{C}$ is the space of admissible Beltrami coefficient such that $\|\mu\|_{\infty}<1, f=\mathcal{F}(\mu,p_i,q_i)$.

Although LSQC provides an elegant operator viewpoint, differentiating through a classical LSQC solver requires backpropagation through large sparse linear systems, which is computationally expensive for a high-resolution surface mesh. This computational barrier is one reason why many existing QC-based optimization pipelines update Beltrami variables through alternating or splitting strategies (often ADMM-like) \cite{lui2015splitting,lam2014landmark}, rather than end-to-end gradients. To make the LSQC operator usable as a differentiable computational primitive at scale, we introduce the \textbf{S}pectral \textbf{B}eltrami \textbf{N}etwork (\textbf{SBN}), a multiscale mesh-spectral surrogate that approximates the free-boundary LSQC solution map in a single differentiable forward pass. After training, we freeze SBN and treat the Beltrami coefficient $\mu$ (together with pinning) as the explicit optimization variables, thereby shifting the optimization from unstable vertex-coordinate deformation spaces to the structured admissible set $\{\mu:\|\mu\|_{\infty}<1\}$, which corresponds to quasiconformal mappings in QC theory~\cite{gardiner2000quasiconformal}.
This yields \textbf{SBN-Opt}, an optimization framework that searches over admissible Beltrami coefficients and pinning conditions to solve free-boundary diffeomorphism objectives with explicit distortion control.
\begin{equation}
    \min_{(\mu_v,p_i,\phi,s,r)} E_1\Big(g\big(\mathcal{F}_{\theta^*}(\mu_v,p_1,p_2),\phi,s,r\big)\Big) + E_2(\mu)
\end{equation}
Unlike prior QC optimization methods that incorporate $\mu$ through splitting/ADMM-style alternations, SBN-Opt enables gradient-based optimization directly in Beltrami coefficients through a differentiable surrogate of the LSQC operator. We validate the resulting theory-algorithm pipeline on tasks including equiareal parameterization and inconsistent surface registration.
\paragraph{\textit{Contributions}} Our work provides a foundational framework for solving free-boundary diffeomorphism problems by re-characterizing them as optimization tasks within the space of Beltrami coefficients. Our contributions are summarized as follows:
\begin{enumerate}
    \item \textbf{Foundational Analysis of the LSQC Operator.} We establish that the numerical Least-squares Quasiconformal (LSQC) energy defines a well-posed computational operator for free-boundary problems. Specifically, we prove that the discrete LSQC minimizer exists, is unique given at least two pinned points, and can exactly reconstruct the piecewise linear homeomorphism given its corresponding face-wise constant Beltrami fields. Furthermore, we prove that this operator possesses crucial geometric and computational invariances, including similarity-invariance and resolution-independence.
    \item \textbf{Theoretical Legitimacy of the Optimization Map.} We provide the first rigorous analysis of the LSQC solution. We prove that the discrete LSQC minimizer depends Lipschitz-continuously on the input Beltrami field. Crucially, we establish the Fréchet differentiability of the mapping with respect to the Beltrami coefficients. These results demonstrate that gradient-based optimization in the space of Beltrami coefficient is mathematically sound and that small perturbations in the coefficients yield controlled, predictable changes in the resulting mapping.
    \item \textbf{The Spectral Beltrami Network (SBN) as a Differentiable Primitive.} A major challenge in optimizing free-boundary diffeomorphisms is that backpropagating gradients through a numerical LSQC solver is computationally expensive and thus  practically infeasible for iterative gradient-based pipelines. We introduce the SBN, a novel multiscale mesh-spectral architecture designed to act as a differentiable surrogate for the LSQC operator. By integrating multiscale message-passing with Mesh Spectral Layers (MSL), the SBN captures both local geometric interactions and global dependencies via transforming in the eigenspace of mesh Laplacian. This transforms the computationally expensive, implicitly-defined LSQC solver into a differentiable computational primitive capable of approximating the mapping in a single, fully-differentiable forward pass.
    \item \textbf{SBN-Opt: A Novel Paradigm for Free-Boundary Optimization.} We propose SBN-Opt, an optimization framework that shifts the problem from searching in the unstable space of vertex coordinates to the well-structured space of Beltrami coefficients. Building on the LSQC's properties indicating that$ (\mu,\text{two pins})$ represents the target homeomorphism, we propose \textbf{SBN-Opt}, reframing free-boundary diffeomorphism optimization as direct, gradient-based optimization in Beltrami space, with pins treated as optimization variables as well. This is fundamentally different from prior quasiconformal pipelines that typically update $\mu$ via splitting methods\cite{qiu2019computing,qiu2020inconsistent}, because here the objective is optimized end-to-end through a differentiable surrogate with explicit distortion regularizers on $\mu$. 
    \item \textbf{Empirical validation on parameterization and registration tasks.} We validate the full pipeline on equiareal parameterization and inconsistent surface registration. Extensive experiments show that SBN-Opt consistently produces mappings with significantly lower geometric distortion and higher accuracy than traditional numerical algorithms, meanwhile maintaining bijectivity and free-boundary characteristics.
\end{enumerate}

\paragraph{Roadmap} The rest of the paper is organized as follows. Section \ref{sec: math background}, Section \ref{sec: Properties of Least Squares Quasi-conformal Energy} and Section \ref{sec: LSQC_sensitivity} develop the theoretical foundation. After introducing the required quasiconformal and discrete differential geometry background, we establish the structural properties of LSQC that make it suitable for optimization, including well-posedness under mild conditions, similarity-gauge invariance, and resolution-independent behavior, followed by a sensitivity analysis showing Lipschitz stability and Fréchet differentiability of the LSQC minimizer with respect to the Beltrami coefficient. Section \ref{sec: Spectral Beltrami Network} and Section \ref{sec: neural lsqc} turn these results into a practical differentiable computational pipeline: we introduce the Spectral Beltrami Network (SBN) as a surrogate for the LSQC operator and then present SBN-Opt, an optimization framework over admissible Beltrami coefficients and pinning variables that enables explicit distortion control in free-boundary diffeomorphism objectives. Section \ref{sec: implementation} describes implementation details, and Section \ref{sec: experiments} reports experimental results on equiareal parameterization and inconsistent surface registration. Section \ref{sec: conclusions} concludes with a discussion of achievements and future directions.

%% file: sections/sec4-math_background.tex
\section{Mathematical Background}
\label{sec: math background}
\subsection{Quasiconformal Mappings}
Quasiconformal mappings generalize conformal mappings by extending orientation-preserving homeomorphisms to those with bounded conformality distortions.
Mathematically, suppose $\mathcal{U}$ is a domain in $\mathbb{C}$, a mapping $f \colon \mathcal{U} \rightarrow \mathbb{C}$ is a quasi-conformal map if it satisfies the Beltrami equation
\begin{equation}\label{eq:beltrami}
    \frac{\partial f}{\partial \bar{z}} = \mu(z)\frac{\partial }{\partial z}
\end{equation}
for some measurable complex-valued function $\mu$ with $\| \mu \|_{\infty}<1$, where $\frac{\partial f}{\partial z} = \frac{1}{2}(f_{x}-\sqrt{-1}f_{y})$ and $\frac{\partial f}{\partial \bar{z}} = \frac{1}{2}(f_{x}+\sqrt{-1}f_{y})$. $\mu$ is called the complex dilation or Beltrami coefficient (BC) of $f$,  which measures the local deviation from a conformal map. 

Let $f \colon (x,y) \rightarrow (u,v) $ and $\mu = \rho + \sqrt{-1} \tau$, and denote by $\mathbf{S}(2)$ the space of all $2 \times 2$ symmetric positive definite matrices whose determinant is 1, then 
\begin{equation}
    D_f(z)^{T}D_{f}(z) = |\det D_f(z)|Q(z) 
\end{equation}
where $Q= (Q_{ij}) \colon \Omega \rightarrow \mathbf{S}(2)$. Right multiplying both sides by $D_f(z)^{-1}$,  we obtain a linear system which is the alternative formulation of the Beltrami equation:
\begin{equation} \label{eq:4.4}
    \begin{bmatrix}
    u_x & u_y \\
    v_x & v_y
    \end{bmatrix}^T = \text{sgn}\left(J_f(x)\right) \cdot 
    \begin{bmatrix}
    q_{11} & q_{12} \\
    q_{12} & q_{22}
    \end{bmatrix}
    \begin{bmatrix}
    v_y & -u_y \\
    -v_x & u_x
    \end{bmatrix},
\end{equation}
and from this we have $\mu = \frac{q_{11}-q_{22}+2\sqrt{-1}q_{12}}{q_{11}+q_{22}+2\text{sgn}(\det{D_f})}$. With the constraint that $Q(z) \in \mathbf{S}(2)$, we observe that $|\mu(z)|<1$ if and only if $\det D_f(z)>0$. By inverse function theorem, researchers can achieve local bijectivity and prevent folding by restricting $\| \mu \|_{\infty}<1$ in the triangular mesh \cite{lui2014teichmuller, choi2015flash}.

Suppose $\mu(z)$ is a measurable complex-valued function defined in
a domain $\mathcal{U}\subseteq \mathbb{C}$ for which $\| \mu \|_{\infty}<1$, we have the following existence theorem.
\begin{theorem}
\label{thm: Riemann}
    (\textbf{Measurable Riemann Mapping Theorem}) For any function $\mu : \mathcal{U} \to \mathbb{C}$ on with bounded essential supremum norm $\|\mu\|_\infty < 1$, there is a quasiconformal map $\phi$ on $\overline{\mathcal{U}}$ satisfying the Beltrami equation $\phi_{\bar{z}} = \mu \phi_z$ for almost all $z \in \mathcal{U}$. Moreover, $\phi$ is unique up to post-composition with conformal isomorphisms and $\phi$ depends holomorphically on $\mu$.
\end{theorem}

When comes to discretization, it is a natural manner to ask whether there is a numerical method to compute the quasiconformal mapping given the BC. The answer is yes, thanks to the Least-squares Quasiconformal Energy.

Apart from Equation \ref{eq:4.4}, the Beltrami equation can be rewritten as 
\begin{equation} \label{eq:4.5}
    \begin{bmatrix}
        v_y \\
        -v_x
    \end{bmatrix} =
    \frac{1}{1-|\mu|^{2}}
    \begin{bmatrix}
        (\rho -1)^{2}+\tau^{2} & -2\tau \\
        -2\tau & (1+\rho)^{2}+\tau^{2}
    \end{bmatrix}
    \begin{bmatrix}
        u_x \\
        u_y
    \end{bmatrix}.
\end{equation}
Here we assume $|\mu|<1$ and for convenience, we denote the square matrix on the right-hand side by $\mathbf{A}$. Inspired by the least-squares conformal energy \cite{levy2002least}, Qiu  \cite{qiu2019computing} proposed the concept least-squares quasiconformal (LSQC) energy of a map.
\begin{definition}
\label{def: lsqc energy}
    Let \( \mu = \rho + \sqrt{-1} r \) be a complex-valued function defined on the domain \( \mathcal{U} \). The LSQC energy of the map \( z = (x,y) \mapsto (u,v) \) against the BC \( \mu \) is defined to be 
\begin{equation}
\label{eq:4.6}
    E_{LSQC}(u,v,\mu) = \frac{1}{2} \int_{\mathcal{U}} \|P\nabla u + JP\nabla v\|^2 \, dx \, dy,
\end{equation}

where
\begin{equation*}
    P = 
\frac{1}{\sqrt{1-|\mu|^2}}
\begin{bmatrix}
1 - \rho & -\tau \\ 
-\tau & 1 + \rho
\end{bmatrix},
J =
\begin{bmatrix}
0 & -1 \\ 
1 & 0
\end{bmatrix},
\end{equation*}

so that \( P^{T}P = \mathbf{A} \) as in Equation \ref{eq:4.5} and $P^{T}JP=J$.
\end{definition}
The weight term $\frac{1}{\sqrt{1-|\mu|^2}}$ introduces numerical singularities as $|\mu| \to 1$ and renders the dependence of the energy on $\mu$ highly non-linear. Therefore, we adopt a simplified, unweighted formulation for the remainder of this work. We define the unweighted LSQC energy as the $\ell^2$-norm of the Beltrami equation mismatch.
\begin{definition}
    \label{def: unweighted lsqc energy}
    Let \( \mu = \rho + \sqrt{-1} \tau \) be a complex-valued function defined on the domain \( \mathcal{U} \). The LSQC energy of the map \( z = (x,y) \mapsto (u,v) \) against the BC \( \mu \) is defined to be 
    $$\tilde{E}_{LSQC}(f, \mu) =  \frac{1}{2}\int_{\mathcal{U}} |f_{\bar{z}} - \mu f_z|^2 dA.$$ 
    Equivalently, 
    \begin{equation}
    \label{eq:4.7}
        \tilde{E}_{LSQC}(u,v, \mu) = \frac{1}{2} \int_{\mathcal{U}} \|\tilde{P}\nabla u + J\tilde{P}\nabla v\|^2 \, dx \, dy,
    \end{equation}
    where 
    $\tilde{P}=
    \begin{bmatrix}
    1 - \rho & -\tau \\ 
    -\tau & 1 + \rho
    \end{bmatrix}.$
    Unless otherwise specified, all references to $E_{LSQC}$ in the subsequent sections refer to this unweighted formulation.

\end{definition}
\subsection{Discrete Setup and Notation for LSQC}
Let \( \mathcal{T} = (\mathcal{V},\mathcal{F}) \) be a triangulation of a bounded planar domain $\mathcal{U}$, with
vertex set \(\mathcal{V} = \{1,\dots,N\}\), face (triangle) set \(\mathcal{F}\). For each face \(T\in\mathcal{F}\), we have vertices \(p_i = (x_i,y_i)\), \(i=1,2,3\),  \(d_T\) (twice the area), perimeter \(h_T\) and the radius of the inscribed circle \(r_T\). For the finite element space, we denote the Real P1 FE space:
\[ V_h := \{ \phi_h : \text{continuous, piecewise linear on }\mathcal{T} \}. \]
with Nodal basis \(\{\varphi_i\}_{i=1}^N\subset V_h\) with \(\varphi_i(x_j) = \delta_{ij}\). The Complex-valued P1 FE map
is \( f_h = u_h + \sqrt{-1} v_h,\,u_h,v_h\in V_h. \) and its full coefficient vector
\[ X(f_h) := (u_1,\dots,u_N,\,v_1,\dots,v_N)^\top \in \mathbb{R}^{2N}. \]
As shown in the next section, to obtain a nontrivial LSQC minimizer, we need to pin at least two vertices in $\mathcal{T}$.  We denote the pinned vertex index set by \( P\subset \mathcal{V},\, |P|\ge 2. \) and prescribed images of pinned vertices, \( q_p = (q_{p,x},q_{p,y})\in\mathbb{R}^2,\, p\in P \). Thus, we need to introduce the zero-at-pins subspace in function and coefficient spaces,
\[ V_{h,0} := \{ f_h = u_h + i v_h : u_h,v_h\in V_h,\ f_h(p)=0\ \forall p\in P\}, \]
\[ \mathcal{X}_0 := \{ X\in\mathbb{R}^{2N} : X_p = X_{N+p} = 0\ \forall p\in P\}. \]
and fixed pin map \(f_{h,\mathrm{pin}}\) such that
\[ f_{h,\mathrm{pin}}(p)=q_p\quad(p\in P),\qquad f_{h,\mathrm{pin}}(v)=0\quad(v\notin P) \]
with its full coefficient vector \(\delta(P,q)\in\mathbb{R}^{2N}\):
\[ \delta_j := \begin{cases} q_{p,x}, & j=p,\ p\in P,\\[2pt] q_{p,y}, & j=N+p,\ p\in P,\\[2pt] 0, & \text{otherwise}. \end{cases} \]
Any total map satisfying pin constraints can be written as  \(f_h^{\mathrm{tot}} = f_h + f_{h,\mathrm{pin}},\,f_h\in V_{h,0}, \) with \(X(f_h^{\mathrm{tot}}) = X(f_h) + \delta(P,q). \)
Let \( \widehat X \in \mathbb{R}^{2N-2|P|} \) denote the coefficient vector at free vertices only, obtained from \(X\in\mathcal{X}_0\) by deleting entries indexed by \(p\) and \(N+p\), \(p\in P\), and  correspondingly there is a fixed injection
\(\mathrm{lift}: \mathbb{R}^{2N-2|P|}\to\mathbb{R}^{2N} \) such that \( X = \mathrm{lift}(\widehat X) \in\mathcal{X}_0, \) i.e. \(\mathrm{lift}\) re-inserts zeros in the pinned coordinates.

In discretization, the BC means facewise constant Beltrami field
\[ \mu = (\mu_T)_{T\in\mathcal{F}},\quad \mu_T = \rho_T + \sqrt{-1}\tau_T,\quad \|\mu\|_\infty:=\max_{T\in\mathcal{F}}|\mu_T|. \]
and denote \( S := \{\mu : \|\mu\|_\infty < 1\} \). 
Given a piecewise linear map $f = u+\sqrt{-1}v $ and consider a triangle $T=(x_j,y_j)_{j=1,2,3}$ of $\mathbb{R}^2$, we have 
$$
\label{partial derivative}
\begin{pmatrix}
    \partial u / \partial x \\
    \partial u / \partial y
\end{pmatrix}= 
\frac{1}{d_T}
\begin{pmatrix}
    y_2 - y_3 & y_3-y_1 & y_1-y_2\\
    x_3-x_2 & x_2-x_1 & x_1-x_2
\end{pmatrix}
\begin{pmatrix}
    u_1\\u_2\\u_3
\end{pmatrix}
$$
where $d_T=(x_1y_2-y_1x_2)+(x_2y_3-y_2x_3)+(x_3y_1-y_3x_1)$ is twice the area of the triangle. The Beltrami equation (Equation \ref{eq:beltrami}) can be rewritten as $(1-\mu)f_x + \sqrt{-1}(1+\mu)f_y=0$ and collected as a complex-valued linear system.
$$
0 = \frac{\sqrt{-1}}{d_T}(W_1  \quad W_2 \quad W_3)(U_1 \quad U_2 \quad U_3)^\top
$$
where $U_j=u_j+\sqrt{-1}v_j$ and 
\begin{align}
\label{coefficient of M}
    W_1 & = (1+\mu)(x_3-x_2)+\sqrt{-1}(1-\mu)(y_3-y_2) \\
    W_2 & = (1+\mu)(x_1-x_3)+\sqrt{-1}(1-\mu)(y_1-y_3) \\
    W_3 & = (1+\mu)(x_2-x_1)+\sqrt{-1}(1-\mu)(y_2-y_1).
\end{align}
As a result, 
    \begin{align*}
    4\,E_{LSQC}(\textbf{U}=(U_1,\ldots,U_{|\mathcal{V}|})) &= \sum\limits_{T \in \mathcal{T}}d_T \Big|\left( (1-\mu)f_x + \sqrt{-1}(1+\mu)f_y \right)\Bigg|_T \Big|^{2} \\
    &=\sum_{T\in \mathcal{T}}\frac{1}{d_T} \Big|
    (W_{1,T}  \quad W_{2,T} \quad W_{3,T})(U_{1,T} \quad U_{2,T} \quad U_{3,T})^\top\Big|^2
\end{align*}
and we may write as $4\,E_{LSQC}(\textbf{U})=\| \mathcal{M}\textbf{U}\|^2$ with $\mathcal{M}=(m_{ij}) \in \mathbb{C}^{|\mathcal{F}| \times N}$ defined as
\[
\mathcal{M}_{ij} = 
\begin{cases} 
 W_{j,T_i} & \text{if j is a vertex of the face } T_i, \\ 
0 & \text{otherwise}.
\end{cases}.
\]
Using the complex Beltrami equation on each face, we assemble
\(  A_{\mathrm{full}}(\mu)\in\mathbb{R}^{2|\mathcal{F}|\times 2N}\) obtained by splitting real/imaginary parts of the complex system \(\mathcal{M}\textbf{U} = 0.\) Each row corresponds to a real equation; each row has only six nonzero vertex blocks. With pinned vertex index set $P$ given, we can obtain a "reduced" LSQC matrix \(A(\mu)\) by restricting the columns of \(A_{\mathrm{full}}(\mu)\) to free vertices:
\[ A(\mu) := A_{\mathrm{full}}(\mu)\,\Pi \quad\in\mathbb{R}^{2|\mathcal{F}|\times (2N-2|P|)}, \] where \(\Pi:\mathbb{R}^{2N-2|P|}\to\mathbb{R}^{2N}\) is the matrix of the injection \(\mathrm{lift}\).
Thus, \[ A_{\mathrm{full}}(\mu)\,\mathrm{lift}(\widehat X) = A(\mu)\,\widehat X. \]
Concretely, for a given \(\mu\) and pinned point conditions $P,q$, the LSQC energy on total maps can be written in reduced variables as \[ E_{LSQC}(\widehat X;\mu) := \frac12\big\|A(\mu)\widehat X + A_{\mathrm{full}}(\mu)\delta(P,q)\big\|_2^2, \]
and the LSQC minimizer \(\widehat X(\mu) = \arg\min_{\widehat X} J(\widehat X;\mu), \) is the unique solution of the normal equations \[ A(\mu)^TA(\mu)\,\widehat X(\mu) = -A(\mu)^T A_{\mathrm{full}}(\mu)\,\delta(P,q). \] (non-singularity of $A(\mu)$ is shown in the next section). We define \[N(\mu) := A(\mu)^T A(\mu)\in\mathbb{R}^{(2N-2|P|)\times(2N-2|P|)} \]
Concretely, We decomposed $\mathbf{U}=(\mathbf{U}_f^\top, \mathbf{U}_{p}^\top)^\top$, where $\mathbf{U}_f$ are free points, i.e. variables of the optimization problem, and $\mathbf{U}_f$ are the points pinned. Similarly, we can decompose $\mathcal{M}$ in blocked matrices as 
\(\mathcal{M} = \begin{pmatrix}
    \mathcal{M}_f, \mathcal{M}_p
\end{pmatrix}\)
with $\mathcal{M}_f \in \mathbb{C}^{|\mathcal{F}| \times (N-|P|)} $ and $\mathcal{M}_p \in \mathbb{C}^{|\mathcal{F}| \times |P|} $.
Let $^1$ and $^2$ be real and imaginary parts of a complex number, and rewrite the $|\mathcal{F}|$ complex linear equations to $2\,|\mathcal{F}|$ real equations, we have 
\[
E_{LSQC}(\mathbf{U};\mu) = \| A(\mu) \mathbf{u}-\mathbf{b} \|^2
\]
where $\mathbf{u}=((\mathbf{U}_{f}^{1}) ^\top,(\mathbf{U}^2_f)^\top)^\top$ and 
\[
A(\mu) = \begin{pmatrix}
    \mathcal{M}_f^1 & -\mathcal{M}_f^2 \\
    \mathcal{M}_f^2 & \mathcal{M}_f^1
\end{pmatrix}, 
\mathbf{b} = - \begin{pmatrix}
    \mathcal{M}_p^1 & -\mathcal{M}_p^2 \\
    \mathcal{M}_p^2 & \mathcal{M}_p^1
\end{pmatrix}
\begin{pmatrix}
    \mathbf{U}_p^1\\
    \mathbf{U}_p^2
\end{pmatrix}.
\]

Although LSQC has been used in numerical computations, its theoretical properties, suitability for neural simulation and stability estimate have not been systematically studied. We address this gap in the following section.

%% file: sections/sec5-LSQC_properties.tex
\section{Properties of Least-squares Quasiconformal Energy}
\label{sec: Properties of Least Squares Quasi-conformal Energy}

We now analyze the properties of LSQC energy that make it uniquely suitable for computing the corresponding quasiconformal mappings.

In discretization, the numerical solution $U$ is unique when at least two points’ destinations are prescribed. Here is the statement with proof.
\begin{proposition}
\label{full rank of A}
    Suppose $\|\mu\|_{\infty}<1$, and the triangulation mesh is connected without dangling triangles. As long as $p \geq 2$, then $\mathcal{A}=A(\mu)$ has full rank.
\end{proposition}

\begin{proof}
   The idea is that start with the pinned points, we can construct the triangular mesh $\mathcal{T}$ incrementally via two operations:
     \begin{enumerate}[leftmargin=30pt]
        \item glue: add a new vertex and connect it to two neighboring vertices.
        \item join: connect two existing vertices.
    \end{enumerate}
    
    Correspondingly, the matrix $\mathcal{M}$ is constructed incrementally as well. Let $\mathcal{M}^{(i)} \in \mathbb{C}^{|\mathcal{F}_i| \times |\mathcal{V}_i|}$ be the $\mathcal{M}$ based on the constructed triangulation after $i-$th  step.  Two key observations are that $|\mathcal{F}_i|\geq |\mathcal{V}_i|-2$ and the rank of $\mathcal{M}_f^{(i)}$ is always  its column number. If the $(i+1)-$th step is a joining operation, then compared to $\mathcal{M}_f^{(i)}$, $\mathcal{M}_f^{(i+1)}$ just has one more row, and joining has no affect on the rank.  If the $(i+1)-$th step is a gluing operation, then $\mathcal{M}_f^{(i+1)}$ has one more row and column, and the new column has only one nonzero element in the last row thus the rank still equals to the column number. The premise of the argument is that for each triangle $T$, correspond $W_{j,T}$ are nonzero. This is still true when $\| \mu \|_{\infty} <1 $. Let $\mu = \rho + \sqrt{-1}\tau$ and $G(\mu_T) = \begin{pmatrix}
        1+\rho_T & \tau_T \\
        \tau_T & 1-\rho_T
    \end{pmatrix} $, then $W_{j,T} = \begin{pmatrix}
        1 & \sqrt{-1}
    \end{pmatrix}G(\mu_T)\begin{pmatrix}
        x_{j+2}-x_{j+1} \\ y_{j+2}-y_{j+1}
    \end{pmatrix}$. Since $\| \mu \|_{\infty} <1$ and $\mathcal{T}$ is a non-degenerate triangulation, then $G(\mu_T)$ is non-singular and $W_{j,T}$ are all nonzero. Last, since $\mathcal{M}_f$ is a full-rank complex matrix, $\mathcal{A}$ also has full rank.
\end{proof}
A natural corollary is the existence and uniqueness of minimizer of LSQC energy with $p\geq 2$ points pinned, given by $\mathbf{x}=(\mathcal{A}^\top \mathcal{A})^{-1}\mathcal{A}^\top \mathbf{b}$. On the other hand, pinning too many points would overdetermine the system and the solution may not be compatible with the provided BCs. Therefore, it is preferred to only pin two points that are far from each other to avoid an unnatural local scale change. Consequently, if $|\mathcal{F}|=|\mathcal{V}|-2$ the numerical result $\mathbf{U}$ is an analytic solution to the Beltrami equation, provided face-wise constant $\mu$.

Besides, suppose the provided face-wise constant BC is induced by a face-wise linear homeomorphism, the LSQC algorithm can reconstruct the homeomorphism given the induced BC and two points of the homeomorphism.
\begin{corollary}
    Suppose $\mathcal{T}$ is a triangular mesh, $F$ is a face-wise linear homeomorphism of $\mathcal{T}$ and $F(p_1)=q_1, F(p_2)=q_2$, let $\mu$ be the face-wise constant BC induced by $F$, then the solution of the LSQC system derived by $\mu$ and pinning $F(p_1)=q_1, F(p_2)=q_2$ is exactly the $F$.
\end{corollary}
\begin{proof}
    With two points pinned, the LSQC system has only one minimizer satisfying the pin-point condition. Since $\mu$ is induced by $F$,  $(1-\mu)F_x+\sqrt{-1}(1+\mu)F_y = 0$ on each face and $E_{\text{LSQC}}(F)=0$. Meanwhile, $F$ satisfies the pin-point condition, and thus it is the solution to the LSQC system.
\end{proof}

In addition, the solution is invariant by a similarity transformation, which fits the implication of the composition formula $\mu_{g \circ f} = \frac{\mu_{f}+\overline{f_{z}}/f_{z}(\mu_g \circ f)}{1+\overline{f_z}/f_z\overline{\mu_f}(\mu_g \circ f)}$ that post-composition of a conformal mapping $g$ does not change the BCs of the original function $f$.
\begin{proposition}
\label{prop: independent of similarity transformation}
    Given $\mu$ and pinned points $\mathbf{U}_p$ and denote the corresponding free-point solution by $\mathbf{U}_{f}$, if the pinned points are transformed to $z\mathbf{U}_p+\mathbf{T}$ where $z \in\mathbb{C}, \mathbf{T} = (z',\ldots,z')^\top\in \mathbb{C}^{|\mathcal{V}|}$, then the corresponding free-point solution is $z\mathbf{U}_f+\mathbf{T}$.
\end{proposition}
\begin{proof}
    If $z=e^{i\theta}$ with $ \theta \in \mathbb{R}$, then 
    \begin{align*}
        \operatorname{min}_{\tilde{\mathbf{U}}}E_{\text{LSQC}}( \tilde{\mathbf{U}}, \mathbf{U}_p) & = E_{\text{LSQC}}(\mathbf{U}_f, \mathbf{U}_p)  =  (\mathcal{M}\mathbf{U})^*\mathcal{M}\mathbf{U} \\ & = z^*z(\mathcal{M}\mathbf{U})^*\mathcal{M}\mathbf{U}  = E_{\text{LSQC}}(z\mathbf{U}_{f}, z\mathbf{U}_p) \\
        &\geq \operatorname{min}_{\tilde{\mathbf{U}}}E_{\text{LSQC}}( \tilde{\mathbf{U}}, z\mathbf{U}_p)
    \end{align*}
    On the other hand, denote the $\operatorname{argmin}_{\tilde{\mathbf{U}}}E_{\text{LSQC}}( \tilde{\mathbf{U}}, z\mathbf{U}_p)$ by $\hat{\mathbf{U}}_{f}$, we have 
    \[\operatorname{min}_{\tilde{\mathbf{U}}}E_{\text{LSQC}}( \tilde{\mathbf{U}}, z\mathbf{U}_p) =
    E_{\text{LSQC}}(z^*\hat{\mathbf{U}}_{f}, \mathbf{U}_p) \geq \operatorname{min}_{\tilde{\mathbf{U}}}E_{\text{LSQC}}( \tilde{\mathbf{U}}, \mathbf{U}_p)
    \]
    Therefore, $\operatorname{min}_{\tilde{\mathbf{U}}}E_{\text{LSQC}}( \tilde{\mathbf{U}}, \mathbf{U}_p)=\operatorname{min}_{\tilde{\mathbf{U}}}E_{\text{LSQC}}( \tilde{\mathbf{U}}, z\mathbf{U}_p) $ and $\hat{\mathbf{U}}_{f}=z\mathbf{U}_{f}$.
    If $z \in \mathbb{R}$,  then the minimizer is  $(\mathcal{A}^\top \mathcal{A})^{-1}\mathcal{A}^\top (z\mathbf{b}) = z (\mathcal{A}^\top \mathcal{A})^{-1}\mathcal{A}^\top \mathbf{b} = z\mathbf{U}_f$.
    Lastly, the invariance by a translation is followed by the observation that 
    $\sum_{i=1}^{3}w_{i,T}=0$ and $(1,1, \ldots, 1)^\top \in \text{ker}\mathcal{M}$.
\end{proof}

In all, LSQC also has its own properties related to solution's uniqueness and existence given points pinned, as well as invariant by a family of transformations, which is very analogous to Theorem \ref{thm: Riemann}. Moreover, the solution of the minimization problem is independent of the resolution of the mesh.
\begin{proposition}
\label{prop: resolution independence}
    Suppose we have a triangulation $\mathcal{T}=(\mathcal{V},\mathcal{F})$ with face-wise constant $\mu = \{\mu_{T}: T \in \mathcal{F}\}$, and pinned points $\mathbf{U}_p$, and denote the corresponding solution by $\mathbf{U} = (\mathbf{U}_f, \mathbf{U}_p)$. Let $T=(\mathbf{v}_1,\mathbf{v}_2,\mathbf{v}_3)$ be a face in $T$ and we split $T$ into three triangles with $\mathbf{v}=\sum_{i=1}^{3}\alpha_j\mathbf{v}_j~ (\alpha_j>0)$ being the introduced vertex and $T_i$ being the new triangles not containing $\mathbf{v}_i$ as vertex. If $\mu_{T_i}=\mu_{T}~(i=1,2,3)$ and BCs on other faces remain, then the minimizer of this new minimization problem is $\mathbf{U}^{+}=(\mathbf{U}^\top, \mathbf{U}_{v})^\top$ where $\mathbf{U}_v=\sum_{i=1}^{3}\alpha_i\mathbf{U}_i$.
\end{proposition}
\begin{proof}
    There are four cases that $n~(n=0,1,2,3)$  of 3 vertices in $T$ are pinned points. The case that $n=3$ is trivial and let's first look into the case that $n=0$. In this case, inserting $v$ would augment $\mathcal{M}_f$ by two rows and one columns (one triangle become three triangles and add one more free point), and augment $\mathcal{M}_p$ by one null row. Denote the new blocked matrices by $\mathcal{M}_{f}^{+}$ and $\mathcal{M}_{p}^{+}$, and structure of these four matrices are:
    $$\mathcal{M}_f =
\left(
\begin{array}{ccccc}
& & &  & \\
    & & \mathcal{N}_f  & & \\ 
    & & &  & \\
    \hdashline
   & & \mathcal{F} & &
\end{array}
\right),\quad
\mathcal{M}_f^+ =
\left(
\begin{array}{ccc:c}
    &\mathcal{N}_f  &  & \begin{array}{c} 0 \\ \vdots \\ 0 \end{array} \\ \hdashline
   & \mathcal{L} & &\mathcal{P}
\end{array}
\right),
$$
$$\mathcal{M}_p =
\left(
\begin{array}{ccc}
& &  \\
    & \mathcal{N}_p &  \\ 
    & &   \\
    \hdashline
   &  \mathbf{0} & 
\end{array}
\right),\quad
\mathcal{M}_p^+ =
\left(
\begin{array}{ccc}
& &  \\
    & \mathcal{N}_p &  \\ 
    & &   \\
    \hdashline
   &  \mathbf{0}  & \\
   &  \mathbf{0}  & \\
   &  \mathbf{0}  & \\
\end{array}
\right).
$$
where $\mathcal{N}_f \in \mathbb{C}^{(|\mathcal{F}|-1) \times (|\mathcal{V}|-|P|)}$, $\mathcal{N}_p \in \mathbb{C}^{(|\mathcal{F}|-1) \times |P|}$,  $\mathcal{F} \in \mathbb{C}^{1 \times (|\mathcal{V}|-|P|)}$ with three nonzero elements $f_j=W_{j,T}/\sqrt{d_T}$, $\mathcal{L}$ and $\mathcal{P}$ both have three rows with nonzero elements defined as:
\[
l_{ij} = w_{j,T_{i}} = \frac{1}{\sqrt{\alpha_i}}(\alpha_if_j-\alpha_jf_i), ~p_i=\frac{f_i}{\sqrt{\alpha_i}}.
\]
Since $\mathbf{U}_f = -(\mathcal{M}_f^*\mathcal{M}_f)^{-1}\mathcal{M}_{f}^*\mathcal{M}_p\mathbf{U}_p$, we have  
\begin{align*}
    \mathcal{M}_f^*\mathcal{M}_f\mathbf{U}_f &=\left(\begin{array}{c:c}
 \mathcal{N}_f^*  & \mathcal{F}^* 
\end{array}
\right)
\left(\begin{array}{c}
 \mathcal{N}_fU_f \\ 
 \sum\limits_{j=1}^{3}f_j\mathbf{U}_j
\end{array}
\right) = \mathcal{N}_f^* \mathcal{N}_f\mathbf{U}_f+ (\sum\limits_{j=1}^{3}f_j\mathbf{U}_j)\mathcal{F}^*=-\mathcal{M}_f^*\mathcal{M}_p\mathbf{U}_p 
\end{align*}
To prove the statement, we need to show ${\mathcal{M}_f^{+}}^*\mathcal{M}_{f}^{+} \begin{pmatrix}
    \mathbf{U}_f\\
    \mathbf{U}_v
\end{pmatrix} = -{\mathcal{M}_{f}^+}^*\mathcal{M}^{+}_p\mathbf{U}_p$. And for the left hand side, 
\begin{align*}
    {\mathcal{M}_f^{+}}^*\mathcal{M}_{f}^{+} \begin{pmatrix}
    \mathbf{U}_f\\
    \mathbf{U}_v
\end{pmatrix} &= {\mathcal{M}_f^{+}}^* \begin{pmatrix}
    \mathcal{N}_f\mathbf{U}_f\\
    \mathcal{L}\mathbf{U}_f+(\sum_i\alpha_i\mathbf{U}_i)\mathcal{P}
\end{pmatrix} = {\mathcal{M}_f^{+}}^* \begin{pmatrix}
    \mathcal{N}_f\mathbf{U}_f\\
    \sqrt{\alpha_1}\sum_if_i\mathbf{U}_i \\
    \sqrt{\alpha_2}\sum_if_i\mathbf{U}_i \\
    \sqrt{\alpha_3}\sum_if_i\mathbf{U}_i  
\end{pmatrix} \\
&=\begin{pmatrix}
    \mathcal{N}_f^*\mathcal{N}_f\mathbf{U}_f\\
    0
\end{pmatrix}+ \begin{pmatrix}
    \mathcal{L}^*\\
    \mathcal{P}^*
\end{pmatrix}\begin{pmatrix}
    \sqrt{\alpha_1}\sum_if_i\mathbf{U}_i \\
    \sqrt{\alpha_2}\sum_if_i\mathbf{U}_i \\
    \sqrt{\alpha_3}\sum_if_i\mathbf{U}_i
\end{pmatrix}
\end{align*}
and ${\mathcal{M}_{f}^+}^*\mathcal{M}^{+}_p\mathbf{U}_p= \begin{pmatrix}
    \mathcal{M}_{f}^*\mathcal{M}_p\mathbf{U}_p\\
    0
\end{pmatrix}$. So, we are left to prove\[
\begin{pmatrix}
    \mathcal{L}^*\\
    \mathcal{P}^*
\end{pmatrix}
\begin{pmatrix}
    \sqrt{\alpha_1}\sum_if_i\mathbf{U}_i \\
    \sqrt{\alpha_2}\sum_if_i\mathbf{U}_i \\
    \sqrt{\alpha_3}\sum_if_i\mathbf{U}_i
\end{pmatrix}= \begin{pmatrix}
    (\sum\limits_{j=1}^{3}f_j\mathbf{U}_j)\mathcal{F}^*\\
    0
\end{pmatrix}. 
\] For this, we have 
\begin{align*}
    \sqrt{\alpha_2}(\sum_{i=1}^3f_i\mathbf{U}_i)l_{21}^*+\sqrt{\alpha_3}(\sum_{i=1}^3f_i\mathbf{U}_i)l_{31}^* &=(\sum_{i=1}^3f_i\mathbf{U}_i)(\alpha_2 f_1^*-\alpha_1f_2^*+\alpha_3f_1^*-\alpha_1f_3^* )\\
    &=(\sum_{i=1}^3f_i\mathbf{U}_i)(\alpha_2 f_1^* + \alpha_3 f_1^*+\alpha_1 f_1^*) = (\sum_{i=1}^3f_i\mathbf{U}_i) f_1^*
\end{align*}
where the last second equality is due to $\sum\limits_{i=1}^{3}f_i=0$. Similarly, we have 
\begin{align*}
    \sqrt{\alpha_3}(\sum_{i=1}^3f_i\mathbf{U}_i)l_{32}^*+\sqrt{\alpha_1}(\sum_{i=1}^3f_i\mathbf{U}_i)l_{12}^* &= (\sum_{i=1}^3f_i\mathbf{U}_i) f_2^*\\
    \sqrt{\alpha_1}(\sum_{i=1}^3f_i\mathbf{U}_i)l_{13}^*+\sqrt{\alpha_2}(\sum_{i=1}^3f_i\mathbf{U}_i)l_{23}^* &= (\sum_{i=1}^3f_i\mathbf{U}_i) f_3^*.
\end{align*} Lastly, $P^*\begin{pmatrix}
    \sqrt{\alpha_1}\sum_if_i\mathbf{U}_i \\
    \sqrt{\alpha_2}\sum_if_i\mathbf{U}_i \\
    \sqrt{\alpha_3}\sum_if_i\mathbf{U}_i
\end{pmatrix} = (\sum_if_i\mathbf{U}_i)(\sum_if_i^*)=0$.
For the case that $n=1$, we assume $\mathbf{v}_1$ to be the pinned points. Structures of the four blocked matrices are:
$$\mathcal{M}_f =
\left(
\begin{array}{ccccc}
& & &  & \\
    & & \mathcal{N}_f  & & \\ 
    & & &  & \\
    \hdashline
   \mathbf{0}& f_2 & \mathbf{0} &f_3 &\mathbf{0}
\end{array}
\right),\quad
\mathcal{M}_f^+ =
\left(
\begin{array}{ccc:c}
    &\mathcal{N}_f  &  & \begin{array}{c} 0 \\ \vdots \\ 0 \end{array} \\ \hdashline
   & \mathcal{L} & &\mathcal{P}
\end{array}
\right),
$$
$$\mathcal{M}_p =
\left(
\begin{array}{ccc}
& &  \\
    & \mathcal{N}_p &  \\ 
    & &   \\
    \hdashline
   \mathbf{0}& f_1  & \mathbf{0}
\end{array}
\right),\quad
\mathcal{M}_p^+ =
\left(
\begin{array}{ccc}
& &  \\
    & \mathcal{N}_p &  \\ 
    & &   \\
    \hdashline
   & \mathbf{0}& \\
   \mathbf{0}& l_{21}  & \mathbf{0} \\
  \mathbf{0}& l_{31}  & \mathbf{0} \\
\end{array}
\right).
$$
where the last three rows in $\mathcal{M}_f^{+}$ and $\mathcal{M}_p^{+}$ from top to bottom corresponds to $T_1, T_2, T_3$. In this case, three rows of $\mathcal{L}$ have 2,1 and 1 nonzero entries which are $(l_{12},l_{13}), l_{23} \text{ and }l_{32}$, respectively. And we have
\begin{align*}
    \mathcal{M}_f^*\mathcal{M}_p\mathbf{U}_p &= \mathcal{N}_f^*\mathcal{N}_p\mathbf{U}_p+f_1\mathbf{v}_1 \begin{pmatrix}
        \mathbf{0}\\
        f_2^*\\
        \mathbf{0}\\
        f_3^*\\
        \mathbf{0}
    \end{pmatrix} \\
    {\mathcal{M}^+_f}^*\mathcal{M}_p^+\mathbf{U}_p &=\begin{pmatrix}
        \mathcal{N}_f^*\mathcal{N}_p\mathbf{U}_p \\
        \mathbf{0}
    \end{pmatrix} + l_{21}\mathbf{v}_1\begin{pmatrix}
        \mathbf{0}\\
        l_{23}^*\\
        \mathbf{0}\\
        p_2^* 
    \end{pmatrix} + l_{31}\mathbf{v}_1\begin{pmatrix}
        \mathbf{0}\\
        l_{32}^*\\
        \mathbf{0}\\
        p_3^* 
    \end{pmatrix}\\
    \mathcal{M}_f^*\mathcal{M}_f\mathbf{U}_f&= \mathcal{N}_f^*\mathcal{N}_f\mathbf{U}_f + (f_2\mathbf{v}_2+f_3\mathbf{v}_3)\begin{pmatrix}
        \mathbf{0}\\
        f_2^*\\
        \mathbf{0}\\
        f_3^* \\
        \mathbf{0}
    \end{pmatrix}\\
     {\mathcal{M}^+_f}^*{\mathcal{M}^+_f}\begin{pmatrix}
         \mathbf{U}_f\\
         \sum_j\alpha_j\mathbf{v}_j
     \end{pmatrix} &= {\mathcal{M}_f^{+}}^*\begin{pmatrix}
         \mathcal{N}_f\mathbf{U}_f\\
         l_{12}\mathbf{v}_2+l_{13}\mathbf{v}_3+p_1\sum_{j}\alpha_{j}\mathbf{v}_j\\
         l_{23}\mathbf{v}_3+p_2\sum_{j}\alpha_{j}\mathbf{v}_j\\
         l_{32}\mathbf{v}_2+p_3\sum_{j}\alpha_{j}\mathbf{v}_j
     \end{pmatrix}\\
     &= \left(
\begin{array}{ccc:c}
\\
    &\mathcal{N}_f  &  & \mathcal{L}^*\\ \\\hdashline
   & \mathbf{0}   & &\mathcal{P}^*
\end{array}
\right) \begin{pmatrix}
         \mathcal{N}_f\mathbf{U}_f\\
         \sqrt{\alpha_1}\sum_jf_j\mathbf{v}_j\\
         l_{23}\mathbf{v}_3+p_2\sum_{j}\alpha_{j}\mathbf{v}_j\\
         l_{32}\mathbf{v}_2+p_3\sum_{j}\alpha_{j}\mathbf{v}_j
     \end{pmatrix}
\end{align*}
Again we have $\mathcal{M}^*_f\mathcal{M}_p\mathbf{U}_p = -     {\mathcal{M}^+_f}^*\mathcal{M}_p^+\mathbf{U}_p$, to show $ {\mathcal{M}^+_f}^*{\mathcal{M}^+_f}\mathbf{U}_f = - {\mathcal{M}^+_f}^*\mathcal{M}_p^+\mathbf{U}_p$ and to validate this case, it is equivalent to show
\[
\left(
\begin{array}{c}
     \\
     \mathcal{L}^*
     \\
     \\
     \hdashline
     \mathcal{P}^*
\end{array}
\right) \begin{pmatrix}
    \sqrt{\alpha_1}\sum_jf_j\mathbf{v}_j\\
     l_{23}\mathbf{v}_3+p_2\sum_{j}\alpha_{j}\mathbf{v}_j\\
     l_{32}\mathbf{v}_2+p_3\sum_{j}\alpha_{j}\mathbf{v}_j
\end{pmatrix} + l_{21}\mathbf{v}_1\begin{pmatrix}
        \mathbf{0}\\
        l_{23}^*\\
        \mathbf{0}\\
        p_2^*         0
    \end{pmatrix} + l_{31}\mathbf{v}_1\begin{pmatrix}
        \mathbf{0}\\
        l_{32}^*\\
        \mathbf{0}\\
        p_3^* 
    \end{pmatrix} = (\sum\limits_{j=1}^{3}f_j\mathbf{v}_j)\begin{pmatrix}
        \mathbf{0}\\
        f_2^*\\
        \mathbf{0}\\
        f_3^*\\
        \mathbf{0}
    \end{pmatrix}
\]
and we have 
\begin{align*}
    &\quad (\sqrt{\alpha_1}\sum_jf_j\mathbf{v}_j)l_{12}^*+ (l_{32}\mathbf{v}_2+p_3\sum_{j}\alpha_{j}\mathbf{v}_j+l_{31}\mathbf{v}_1)l_{32}^* \\
    &= (\sum_jf_j\mathbf{v}_j)(\alpha_1f_2^*-\alpha_2f_1^*)
     + \frac{1}{\alpha_3}(\alpha_3f_2^*-\alpha_2f_3^*)\cdot\\
     &\quad \left[ (\alpha_3f_2-\alpha_2f_3)\mathbf{v}_2+f_3\sum_{j}\alpha_{j}\mathbf{v}_j+(\alpha_3f_1-\alpha_1f_3)\mathbf{v}_1\right]\\
     &= (\sum_jf_j\mathbf{v}_j)\left[(\alpha_1f_2^*-\alpha_2f_1^*) + (\alpha_3f_2^*-\alpha_2f_3^*)\right] \\
     &= (\sum_jf_j\mathbf{v}_j)(\alpha_1f_2^*+\alpha_2f_2^*+\alpha_3f_2^*) = (\sum_jf_j\mathbf{v}_j)f_2^*
\end{align*}
where the last two equalities are due to $\sum_jf_j=0$. Similarly, we have
\[
\quad (\sqrt{\alpha_1}\sum_jf_j\mathbf{v}_j)l_{13}^*+ (l_{23}\mathbf{v}_3+p_2\sum_{j}\alpha_{j}\mathbf{v}_j+l_{21}\mathbf{v}_1)l_{23}^* 
= (\sum_jf_j\mathbf{v}_j)f_3^*,\]
And for the last row,
\begin{align*}
    & \quad p_1^{*}(\sqrt{\alpha_1}\sum_jf_j\mathbf{v}_j)+p_2^*(l_{23}\mathbf{v}_3+p_2\sum_{j}\alpha_{j}\mathbf{v}_j) + p_3^*(l_{32}\mathbf{v}_2+ \\
    & \quad \quad p_3\sum_{j}\alpha_{j}\mathbf{v}_j) + l_{21}\mathbf{v}_1p_2^* + l_{31}\mathbf{v}_1p_3^*\\
    &= f_1^* (\sum_jf_j\mathbf{v}_j) + p_2^*(l_{23}\mathbf{v}_3+p_2\sum_{j}\alpha_{j}\mathbf{v}_j+l_{21}\mathbf{v}_1) +
    p_3^*(l_{32}\mathbf{v}_2+p_3\sum_{j}\alpha_{j}\mathbf{v}_j+ l_{31}\mathbf{v}_1) \\
    &= f_1^* (\sum_jf_j\mathbf{v}_j) + f_2^* (\sum_jf_j\mathbf{v}_j) + f_3^* (\sum_jf_j\mathbf{v}_j) = ( f_1 + f_2 +f_3)^*(\sum_jf_j\mathbf{v}_j)=0
\end{align*}
Using the same argument, we can easily validate the case $n=2$.
\end{proof}

This provides a theoretic support for us to compute the deformation of a finer mesh via interpolation on the coarser mesh. 

%% file: sections/sec5b-LSQC_sensitivity.tex
\section{Sensitivity and differentiability of the LSQC minimizer with respect to the Beltrami coefficient}
\label{sec: LSQC_sensitivity}
In the following argument, the maximum norm $\|\cdot\|_{max}$ of a matrix is defined as $\|M\|_{max}=\max_{i,j}|M_{ij}|$
\begin{theorem}
\label{thm 3.1}
    Let $\mathcal{T}$ be a shape‑regular triangulation of the unit disk with vertex set $V$ and face set $F$, namely there exists a constant $\kappa>0$ such that for every triangle $T\in F$, the radius of the inscribed circle $r_{T}$ satisfies that $$r_{T}\geq\frac{h_T}{\kappa},$$where $h_T$ is the perimeter of $T$.
Fix two distinct vertices $p_1,p_2\in V$ and their images $q_1,q_2\in\mathbb{R}^2$. For each facewise constant Beltrami coefficient field $\mu$ with $\lVert\mu\rVert_\infty\le L<1$, let $X(\mu)$ denote the coefficient vector of the unique LSQC minimizer on $T$ subject to $F(p_i)=q_i$. Then there exists a constant $C>0$, depending on $L$, $\kappa$, and the pinned configuration $(p_1,p_2,q_1,q_2)$, such that for all such $\mu,\nu$: $$
\lVert X(\mu) - X(\nu) \rVert_{2}
\;\le\; C\, \lVert \mu - \nu \rVert_{\infty},$$
where the left‑hand side is the $\ell^2$ norm of the difference of vertex positions, i.e.
$$\lVert X(\mu) - X(\nu) \rVert_{2} := \Big(\sum_{v\in V_{\text{free}}} \|X(\mu)(v) - X(\nu)(v) \|^2 \Big)^{\frac{1}{2}}$$
and the right‑hand side is the facewise $\ell^\infty$ norm of the Beltrami coefficients, i.e.
$$\lVert \mu - \nu \rVert_{\infty}:=\max_{T\in F}|\mu_{T}-\nu_{T}|$$
\end{theorem}
To prove it, we need several facts and Lemmas. The first observation is about the coefficients of $A(\mu)$. By the formula in Equation \ref{coefficient of M} we can compute the exact value of nonzero entries in ${A}_{\mathrm{full}}(\mu)$. 
For a triangle $T=(x_i,y_i)_{i=1,2,3}$,  the row corresponding to the real part, its coefficient of $u_1$ is $$\frac{1}{\sqrt{d_T}}[(1+\rho_T)(x_3-x_2)+\tau_T(y_3-y_2)],$$ and the coefficient of $v_1$ is $$\frac{1}{\sqrt{d_T}}[-\tau_T(x_3-x_2)-(1-\rho_T)(y_3-y_2)],$$ both are affine functions in $(\rho_T,\tau_T)$. 
Next, when we remove the columns corresponding to the pinned vertices we obtain $$E_{LSQC}(X)=\|A(\mu)X-b(\mu)\|_2^2,$$ where $A(\mu)\in\mathbb{R}^{m\times n}$ with $m=2|\mathcal{F}|$, $n=2(|\mathcal V|-2)$ and $b(\mu)=-A_{\mathrm{full}}(\mu)\delta(P,q)\in\mathbb{R}^m$ incorporates the contribution of the pinned vertices.
The LSQC minimizer $X(\mu)$ is the unique solution to the least‑squares problem  $$X(\mu)=\arg\min_X \|A(\mu)X-b(\mu)\|_2^2,$$ equivalently to the normal equations 
$$N(\mu)X(\mu) = c(\mu), \quad N(\mu):=A(\mu)^\top A(\mu),\quad c(\mu):=A(\mu)^\top b(\mu).$$
And following several lemmas are related to the Lipschitz dependence of $N(\mu),A(\mu),c(\mu),b(\mu)$ with respect to $\mu$ or estimate of their norm.
\begin{lemma}[uniform coercivity and boundedness of $N(\mu)$]
\label{lemma: uniform coercivity and boundedness of N}
    Let $K := \{\mu :\|\mu\|_\infty \le L\}.$ Then there exists positive constants $L_{N,1}$ and $L_{N,2}$ depending only on $L$ such that $L_{N,1} \le \|N(\mu)\|_2 \le L_{N,2}, \forall \mu \in K$
\end{lemma}
\begin{proof}
    By Proposition \ref{full rank of A} in the paper, for any $\mu$ with $\|\mu\|_\infty < 1$, the LSQC matrix $A(\mu)$ has full column rank after removing pinned columns. In particular, for each $\mu\in K$, $A(\mu)$ has rank $n$, hence $\sigma_{\min}(A(\mu))>0$. The map $\mu \mapsto A(\mu)$ is continuous as a map from the compact space $K$ into $\mathbb{R}^{m\times n}$ equipped with $\|\cdot\|_{\max}$ norm, because by the formula of matrix entries given above, each entry of $A(\mu)$ is an affine (indeed linear) function of the local coefficients $(\rho_T,\tau_T)$. It is a standard result that the smallest  $\sigma_{\min}(\cdot)$ is continuous as a function of a matrix in any matrix norm, so $\mu \mapsto \sigma_{\min}(A(\mu))$ is continuous on $K$. Since $K$ is compact and $\sigma_{\min}(A(\mu))>0$ everywhere on $K$, we have $$\sigma_\star := \inf_{\mu\in K} \sigma_{\min}(A(\mu)) > 0.$$ Similarly, the largest singular value $\sigma_{\max}(A(\mu))=\|A(\mu)\|_2$ is continuous, so $$\sigma^{\ast} := \sup_{\mu\in K} \sigma_{\max}(A(\mu)) < \infty.$$ Therefore $$\sigma_{\ast}^2 \|z\|_2^2 \le z^\top N(\mu)z \le \sigma^{\ast}\,^2 \|z\|_2^2 \quad\forall z,\ \forall\mu\in K$$ Thus $N(\mu)$ is uniformly coercive and uniformly bounded on $K$. In particular, $$\|N(\mu)^{-1}\|_2 \le \frac{1}{\sigma_{\ast}^2}=c_1 \quad \forall\mu\in K,$$ such an upper bound is a constant dependent on $L$.
\end{proof}

\begin{lemma}[Lipschitz dependence of $A(\mu)$ on $\mu$]
\label{Lipschitz dependence of A}
    Let $K := \{\mu :\|\mu\|_\infty \le L\}.$ Then there exists positive constants $L_{A}$ depending only on $\kappa$ such that $ \|A(\mu)-A(\nu)\|_2 \le L_A \|\mu-\nu\|_\infty, \forall \mu,\nu \in K$
\end{lemma}
\begin{proof}
    Recall the formula of entries in $A(\mu)$, Fix a triangle $T$ with vertices $(x_1,y_1),(x_2,y_2),(x_3,y_3)$ and area $d_T$. Let $\mu_T=\rho_T + i\tau_T,\quad \nu_T=\tilde\rho_T + i\tilde\tau_T$. Write $\Delta\rho := \rho-\tilde\rho,\quad \Delta\tau := \tau-\tilde\tau,\quad \Delta\mu := \mu_T - \nu_T$. For the coefficient of $u_1$ in the “real part row” we have: $$A(\mu)_{T,u_1} - A(\nu)_{T,u_1} = \frac{1}{\sqrt{d_T}}\Big[\Delta\rho_T\,(x_3-x_2) + \Delta\tau_T\,(y_3-y_2)\Big].$$ So 
    \begin{align*}\big|A(\mu)_{T,u_1} - A(\nu)_{T,u_1}\big| &\le \frac{1}{\sqrt{d_T}} \sqrt{(\Delta\rho_T)^2+(\Delta\tau_T)^2}\,\sqrt{(x_3-x_2)^2+(y_3-y_2)^2}\\ &= \frac{|e_{32}|}{\sqrt{d_T}}\,|\Delta\mu_T|\end{align*}
    Exactly the same inequality holds for the coefficient of $v_1$ and for the two coefficients in the “imaginary part row” of the same triangle, just with slightly different linear combinations of $\Delta\rho_T$ and $\Delta\tau_T$, but always of the form $\frac{1}{\sqrt{d_T}}\big(a\Delta\rho_T+b\Delta\tau_T\big)$, with $(a,b)$ proportional to $(x_3-x_2,y_3-y_2)$, up to signs and $\pm1$ factors. Hence they are all bounded by the same factor $\frac{|e_{32}|}{\sqrt{d_T}}|\Delta\mu_T|$.
    Similarly, the coefficients for $(u_2,v_2)$ will involve the edge $e_{13}$ and the coefficients for $(u_3,v_3)$ will involve $e_{21}$, with the same structure. Thus, for every nonzero entry of $A(\mu)$ coming from triangle $T$: $$|A(\mu)_{ij}-A(\nu)_{ij}| \le \max\left\{ \frac{|e_{32}|}{\sqrt{d_T}}, \frac{|e_{13}|}{\sqrt{d_T}}, \frac{|e_{21}|}{\sqrt{d_T}} \right\} \,|\Delta\mu|.$$
    Now, at the global level, $|\Delta\mu|\le\|\mu-\nu\|_\infty$ for all faces, so $|A(\mu)_{ij}-A(\nu)_{ij}| \le C_T\,\|\mu-\nu\|_\infty,$ with $C_T := \max\left\{ \frac{|e_{32}|}{\sqrt{d_T}}, \frac{|e_{13}|}{\sqrt{d_T}}, \frac{|e_{21}|}{\sqrt{d_T}} \right\}$. Moreover, for each edge $e$ of a triangle $T$, we have $$\frac{e}{\sqrt{d_T}}=\frac{e}{\sqrt{r_T h_T}}\le\frac{\frac{h_T}{2}}{\sqrt{r_T h_T}}\le \frac{\sqrt{\kappa}}{2}$$
    Therefore $$\|A(\mu)-A(\nu)\|_{\max} := \max_{i,j}|A(\mu)_{ij}-A(\nu)_{ij}| \le \frac{\sqrt{\kappa}}{2}\,\|\mu-\nu\|_\infty$$
    Next, for a general matrix $B$ , we have $\|B\|_2\le\sqrt{\|B\|_1\|B\|_{\infty}}$. Since there are at most 6 nonzero entries per row in $A(\mu)-A(\nu)$, $\|A(\mu)-A(\nu)\|_{\infty}\le 6 \|A(\mu)-A(\nu)\|_{\max}\leq 3\sqrt{\kappa}\|\mu-\nu\|_\infty$. 
    Besides, since $\mathcal{T}$ is $\kappa$-shape regular, the minimum angle is bounded below by a function of $\kappa$ and each vertex is at most adjacent to $C(\kappa)$ faces. Then,  $\|A(\mu)-A(\nu)\|_{1}\le C(\kappa) \|A(\mu)-A(\nu)\|_{\max}\leq C(\kappa)\frac{\sqrt{\kappa}}{2}\|\mu-\nu\|_\infty$. Combining these two, we have  
    
    $$ \|A(\mu)-A(\nu)\|_2 \le L_A\,\|\mu-\nu\|_\infty,\,L_A=\sqrt{\frac{3}{2}C(\kappa)\kappa}$$
\end{proof}

\begin{lemma}[Lipschitz dependence of $b(\mu)$ on $\mu$]
\label{Lipschitz dependence of b(mu)}
    Let $K := \{\mu :\|\mu\|_\infty \le L\}.$ There exists a constant $L_b$ depending on $\kappa,q_i$ such that $$\|b(\mu)-b(\nu)\|_2\le L_b \|\mu-\nu\|_\infty.$$
\end{lemma}
\begin{proof}
   Because each entry of $A_{\text{full}}(\mu)$ is affine in $(\rho_T,\tau_T)$, and $q_1,q_2$ are fixed, each entry of $b(\mu)$ is also an affine function of the $(\rho_T,\tau_T)$. More precisely, $$b(\mu)=-\Big(A(\mu)(:,p_{1,x})\, q_{1,x}+A(\mu)(:,p_{1,y})\, q_{1,y}+A(\mu)(:,p_{2,x})\, q_{2,x}+A(\mu)(:,p_{2,y})\, q_{2,y}\Big)$$ where $A(\mu)(:,p_{1,x})$ means the column vector in $A(\mu)$ corresponding to the variable $p_{1,x}$. Then $b(\mu)-b(\nu)=-\sum\big(A(\mu)-A(\nu)\big)(:,p_{1,x})q_{1,x}$. In the proof of Lemma \ref{Lipschitz dependence of A}, we have shown $\|A(\mu)-A(\nu)\|_{max}\le \frac{\sqrt{\kappa}}{2}\|\mu-\nu\|_{\infty}$
   Thus, for each column index $i$:  $|b(\mu)_i - b(\nu)_i| \le C_0(\kappa,q) \,\|\mu-\nu\|_\infty$ where $C_0(\kappa,q)=2\sqrt{\kappa}\,\max\{|q_{1,x}|,|q_{1,y}|,|q_{2,x}|,|q_{2,y}|\}$.
    Therefore $$\|b(\mu)-b(\nu)\|_2 \le \sqrt{4C(\kappa)}\,\|b(\mu)-b(\nu)\|_{\max} \le L_b \|\mu-\nu\|_\infty,\, L_b = 4\sqrt{C(\kappa)\kappa}\,\max(q).$$
\end{proof}

\begin{lemma}[Lipschitz dependence of $N(\mu)$ and $c(\mu)$]
\label{Lipschitz dependence of N,c}
     Let $K := \{\mu :\|\mu\|_\infty \le L\}.$ There exists a constant $L_N,L_c$ depending on $\kappa,L,q_i$ such that $$\|N(\mu)-N(\nu)\|_2 \le L_N \|\mu-\nu\|_\infty, \, \|c(\mu)-c(\nu)\|_2 \le L_c \|\mu-\nu\|_\infty .$$
\end{lemma}
\begin{proof}
    Recall: $N(\mu)=A(\mu)^\top A(\mu),\quad c(\mu)=A(\mu)^\top b(\mu)$. we have \begin{align*}N(\mu)-N(\nu)&=A(\mu)^\top A(\mu)-A(\nu)^\top A(\nu)\\ &= A(\mu)^\top(A(\mu)-A(\nu)) + (A(\mu)-A(\nu))^\top A(\nu).\end{align*}
    Hence, \begin{align} \|N(\mu)-N(\nu)\|_2 &\le \|A(\mu)^\top(A(\mu)-A(\nu))\|_2 +\|(A(\mu)-A(\nu))^\top A(\nu)\|_2 \\ &\le \|A(\mu)\|_2\,\|A(\mu)-A(\nu)\|_2+\|A(\mu)-A(\nu)\|_2 \,\|A(\nu)\|_2 \\ &= \big(\|A(\mu)\|_2+\|A(\nu)\|_2\big)\,\|A(\mu)-A(\nu)\|_2. \end{align}
    For $\|A(\mu)\|_2$, each element of $A(\mu)$ is in the form like 
    $\frac{1}{\sqrt{d_T}}[(1+\rho_T)(x_3-x_2)+\tau_T(y_3-y_2)]$, and following the idea in Lemma \ref{Lipschitz dependence of A}, we have $\|A(\mu)\|_{max}\leq \frac{1+L}{2}\sqrt{\kappa}$ and thus $$\|A(\mu)\|_{2}\leq \sqrt{6C(\kappa)}\|A(\mu)\|_{max}\leq \sqrt{\frac{3}{2}C(\kappa)\kappa}(1+L)=:M_A.$$
    so $$\|N(\mu)-N(\nu)\|_2 \le L_N\|\mu-\nu\|_\infty,\, L_N =3C(\kappa)\kappa(1+L).$$
    As for the Lipschitz bound for $c(\mu)$, we have \begin{align}c(\mu)-c(\nu) &= A(\mu)^\top b(\mu) - A(\nu)^\top b(\nu)\\ &= A(\mu)^\top\big(b(\mu)-b(\nu)\big)+\big(A(\mu)-A(\nu)\big)^\top b(\nu).\end{align}
    We already have \begin{enumerate}
        \item $\|A(\mu)\|_2\le\sqrt{6C(\kappa)}c(L)\sqrt{\kappa}$,
        \item $\|b(\mu)-b(\nu)\|_2\le L_b\|\mu-\nu\|_\infty$, (Lemma \ref{Lipschitz dependence of A})
        \item $\|A(\mu)-A(\nu)\|_2\le L_A\|\mu-\nu\|_\infty$. (Lemma \ref{Lipschitz dependence of A})
    \end{enumerate}
    It remains to bound $\|b(\nu)\|_2$ uniformly for $\nu\in K$. This is again by Cauchy Inequality:\begin{align*}\|b(\nu)\|_2&\le 4\sqrt{C(\kappa)} \Big(\max_{T}\max_{e}{\frac{e}{\sqrt{d_T}}}\Big) \sqrt{(1+\rho_T)^2+\tau_T^2}\max(q)\\ & \le 2\sqrt{C(\kappa)\kappa}(1+L)\max(q) =:M_{b}\end{align*}
Therefore, $$ \|c(\mu)-c(\nu)\|_2 \le L_c \|\mu-\nu\|_\infty, \,L_c = M_AL_b+M_bL_A=3\sqrt{6}C(\kappa)\kappa(1+L)\max(q)$$
\end{proof}
Within these Lemmas, here is the proof of the Theorem \ref{thm 3.1}.
\begin{proof}
 Since $N(\mu)X(\mu) = c(\mu), N(\nu)X(\nu) = c(\nu)$, it is straightforward to obtain   $$X(\mu)-X(\nu) = N(\mu)^{-1}\big(c(\mu)-c(\nu)\big)+N(\mu)^{-1}\big(N(\nu)-N(\mu)\big)X(\nu).$$
Therefore, $$\|X(\mu)-X(\nu)\|_2 \le \|N(\mu)^{-1}\|_2\, \Big(\|c(\mu)-c(\nu)\|_2 + \|N(\nu)-N(\mu)\|_2\,\|X(\nu)\|_2 \Big).$$
We already have 
\begin{enumerate}
    \item $\|N(\mu)^{-1}\|_2 \le 1/c_1$ (Lemma \ref{lemma: uniform coercivity and boundedness of N})
    \item $\|c(\mu)-c(\nu)\|_2 \le L_c\|\mu-\nu\|_\infty$(Lemma \ref{Lipschitz dependence of N,c})
    \item $\|N(\nu)-N(\mu)\|_2 \le L_N\|\mu-\nu\|_\infty$(Lemma \ref{Lipschitz dependence of N,c})
\end{enumerate}
The last step is to find a uniform upper bound on  $\|X(\nu)\|_2$ for $\nu\in K$. For each $\nu$,  $$X(\nu) = \operatorname{argmin}_X \|A(\nu)X - b(\nu)\|_2^2,$$ so, in particular, $$\|A(\nu)X(\nu)-b(\nu)\|_2 \le \|A(\nu)X_0 - b(\nu)\|_2 $$  for any fixed reference vector $X_0$.  (e.g., $X_0=0$ or any map consistent with pins). Thus, $$\|A(\nu)X(\nu)\|_2 \le C' \|b(\nu)\|_2.$$ for some constant $C'>0$. Meanwhile, $$\sigma_{\ast}^2\|X(\nu)\|_2^2 \le X(\nu)^\top N(\nu)X(\nu) = \|A(\nu)X(\nu)\|_2^2.$$
Since $\|b(\nu)\|_2\le M_b$, we have $$\|X(\nu)\|_2 \le \sqrt{\frac{C'}{\sigma_{\ast}^2}}M_b=:M_X$$ and $M_X$ is a constant dependent on $\kappa,L,q$. Putting everything together, we have $$\|X(\mu)-X(\nu)\|_2 \le \frac{1}{c_1} \Big(L_c + L_N M_X\Big)\,\|\mu-\nu\|_\infty.$$
\end{proof}
\begin{remark}
Theorem \ref{thm 3.1} shows that the discrete LSQC minimizer (and hence the induced mapping) depends Lipschitz continuously on the Beltrami field $\mu$. In the context of optimization, this means that if $\mu$ is used as the optimization variable, then small perturbations of $\mu$ cannot produce arbitrarily large changes in the resulting mapping. In particular, if the task loss $\mathcal{L}$ is continuous in the mapping $f$, the composed loss $\mathcal{L}(\mathcal{F}(\mu,P,Q))$ is continuous in $\mu$. This justifies optimization over $\mu$.
\end{remark}
Combined with Proposition \ref{prop: resolution independence}, we have an estimate on the error of approximating the LSQC solution on a fine mesh via interpolation on the solution on its coarser submesh. Let $\mathcal{T}_h=(V_h,F_h)$ be a triangulation and $\mathcal{T}_H=(V_H,F_H)$ be a coarse submesh such that each face $K\in F_H$ is the union of a set of fine faces, $S(K)\subset F_h,\, \bar{K}=\cup_{T\in S(K)}\bar{T}$ and $V_H \subset V_h$. Fix two pinned vertices $p_1,p_2\in V_H$ and images $q_1,q_2\in \mathbb{C}$ and assume the same pin constraints are used in both meshes. Let $\mu_h\in\mathbb C^{F_h}$ be a given fine, facewise-constant Beltrami field with $\|\mu_h\|_\infty\le L<1$. Given any complex-valued vertex data $U_H:V_H\to\mathbb C$, define its interpolation $I_{H\to h}U_H:V_h\to\mathbb C$ by:
\begin{itemize}
    \item if $v\in V_H$, $(I_{H\to h}U_H)(v)=U_H(v)$;
    \item if $v\in V_h\setminus V_H$, let $K=(v_1,v_2,v_3)\in F_H$ be the unique coarse triangle containing $v$, write barycentric coordinates $v=\alpha_1 v_1+\alpha_2 v_2+\alpha_3 v_3$, then  
    \[  
    (I_{H\to h}U_H)(v)=\alpha_1U_H(v_1)+\alpha_2U_H(v_2)+\alpha_3U_H(v_3). \]
\end{itemize}
Let $\mu_H\in\mathbb C^{F_H}$ be any coarse Beltrami field. Define its constant-on-parent prolongation to the fine faces by  
\[  
(\mathcal P\mu_H)(T) := \mu_H(K)\quad \text{for the unique }K\in F_H\text{ such that }T\in\mathcal S(K).  
\]
\begin{corollary}
    Let $U_h(\mu_h)$ be the fine-mesh LSQC minimizer under pins $p_i\mapsto q_i$ and $U_H(\mu_H)$ be the coarse-mesh LSQC minimizer under the same pins. Then, letting $C_h$ be the Lipschitz constant of Theorem \ref{thm 3.1} on the fine mesh $\mathcal T_h$, we have
\[  
\|U_h(\mu_h)-I_{H\to h}U_H(\mu_H)\|_{2}  
\le 
C_h\|\mu_h-\mathcal P\mu_H\|_{\infty}.  
\]
where $C_h$ depends on $(L,\kappa_h,p_1,p_2,q_1,q_2)$.
\end{corollary}

In practice, it is possible that a Beltrami field $\mu$ is not admissible, which means that with pinned points $p_i,q_i$ given, there does not exist a piecewise linear mapping $f$ such that $E_{LSQC}(f;\mu)=0$ and thus the Beltrami field $\widetilde{\mu}$ induced by $X^{*}(\mu)$ is not identical to $\mu$. Therefore, with pinned points $p_i,q_i$ given, it is natural to introduce a concept \textbf{Beltrami Map} $\mathcal{B}:\mathcal{S}=\{\mu:\|\mu\|_{\infty}<1\}\to \mathbb{C}^{|F|}$  which is defined as $$\mathcal{B}(\mu):=\widetilde{\mu}(X(\mu)),$$ namely $\mathcal{B}(\mu)$ is the Beltrami field induced by the LSQC minimizer $X(\mu)$. We can show that under some mild condition, the error between $\mathcal{B}(\mu)$  and $\mu$ is upper bounded by minimal LQSC energy.
\begin{theorem}
    Let $\mathcal{T}$ be a triangulation of the unit disk $\mathbb{D}$, which satisfies the regularity condition in Theorem \ref{thm 3.1}, and $f$ be the LSQC solution given Beltrami field $\mu$, pinned points $p_1,p_2$ and their images $q_1,q_2$. Suppose $\mathcal{B}(\mu) \in \mathcal{S}$ and $f$ is bijective, then \[
    \|\big(\mathcal{B}(\mu)-\mu\big)\circ f^{-1}\|_{L^2(f(\mathbb{D}))}^2 \leq E_{LSQC}(f;\mu)
    \]
\end{theorem}
\begin{proof}
\begin{align*}
E_{LSQC}(f;\mu)&=\sum Area(T)\,|f_{\bar{z}}\Big|_T-\mu_T f_{z}\Big|_T|^2= \sum Area(T)\,|f_{z}\Big|_T|^2\,|\mathcal{B}(\mu)\Big|_{T}-\mu_{T}|^2\\
&=\sum\frac{Area(f(T))}{\det(f\Big|_T)}\,|f_{z}\Big|_T|^2\,|\mathcal{B}(\mu)\Big|_{T}-\mu_{T}|^2 \\ 
&=\sum\frac{Area(f(T))}{|f_z\Big|_T|^2-|f_{\bar{z}}\Big|_T|^2}
\,|f_{z}\Big|_T|^2\,|\mathcal{B}(\mu)\Big|_{T}-\mu_{T}|^2 \\
&=\sum \frac{Area(f(T))}{1-|\mathcal{B}(\mu)\Big|_T|^2}\,|\mathcal{B}(\mu)\Big|_{T}-\mu_{T}|^2 \\
& \geq \sum Area(f(T))\,|\mathcal{B}(\mu)\Big|_{T}-\mu_{T}|^2 = \|\big(\mathcal{B}(\mu)-\mu\big)\circ f^{-1}\|_{L^2(f(\mathbb{D}))}^2
\end{align*}

\end{proof}

\begin{remark}
Let $\mu\in L^\infty(\mathbb{D})$ with
$\|\mu\|_\infty\le k<1$. Consider a family of shape-regular triangulations
$\{\mathcal T_h\}_{h\downarrow 0}$ whose union exhausts $\mathbb{D}$, and let $\mathcal V_h$ denote the
complex-valued $P^1$ finite element space on $\mathcal T_h$ with the same two-point pinning
used to fix the similarity gauge. At the continuum level, the Measurable Riemann Mapping Theorem guarantees the existence of
a (normalized) quasiconformal solution $F\in W^{1,2}_{\mathrm{loc}}(U)$ to
$F_{\bar z}=\mu F_z$ a.e. in $U$ and $ E_{\mathrm{LSQC}}(F,\mu)=0$.
If, in addition, one can choose a sequence of piecewise linear maps $F_h\in \mathcal V_h$ (for instance, suitable finite element interpolants of smooth
approximations of $F$) such that $F_h\to F$ in $W^{1,2}(U)$ while 0 and 1 always remain fixed,
then the residual converges to zero: $E_{\mathrm{LSQC}}(F_h,\mu)\;\rightarrow\;0$. At the same time, supppose the discrete coefficients $\mu_h$ are obtained from a fixed measurable field $\mu\in L^\infty(U)$ (e.g., by facewise averaging or $L^2$ projection), so that $\mu_h\to\mu$ in a suitable sense.
Consequently, the discrete minimal energy $\min_{f\in \mathcal V_h} E_{\mathrm{LSQC}}(f,\mu_h) \leq E_{\mathrm{LSQC}}(F_h,\mu_h) \approx E_{\mathrm{LSQC}}(F_h,\mu)\;\rightarrow\;0$.
Combining this with Theorem~3, we obtain the heuristic implication that under mesh refinement, even if the Beltrami coefficient $\mathcal B(\mu)$ of the discrete LSQC minimizer $f_h$ cannot recover exactly the prescribed coefficient $\mu$ in some case, their discrepancy in an $\ell^2$ sense would gradually decay to 0.
\end{remark}

Next we show the Fr$\mathrm{\acute{e}}$chet differentiability of the LSQC minimizer with respect to BC $\mu$. 

\begin{theorem}[Fr$\mathrm{\acute{e}}$chet differentiability of LSQC minimizer]
\label{differentiability of the discrete LSQC minimizer}
Suppose the triangular mesh satisfies the same regularity condition in Theorem \ref{thm 3.1}. Given pinned point index set $P$ and corresponding images $q_i$, for each \(\mu\in\mathcal{S}\), define \(\widehat X(\mu)\in\mathbb{R}^{2N-2|P|}\) as the unique minimizer of $$(\widehat X;\mu) = \frac12\|A(\mu)\widehat X + A_{\mathrm{full}}(\mu)\delta(P,q)\|_2^2,$$ Then:
\begin{enumerate}
    \item There exists a unique \(C^\infty\) map $\mu \mapsto \widehat X(\mu)\in\mathbb{R}^{2N-2|P|}$ satisfying the normal equation $$A(\mu)^T A(\mu)\,\widehat X(\mu) = -A(\mu)^T A_{\mathrm{full}}(\mu)\,\delta(P,q).$$
    \item Defining the LSQC minimizer \(f_h(\mu)\in V_{h,0}^{\mathbb{R}}\) by
 $X(f_h(\mu)) := \mathrm{lift}(\widehat X(\mu))$, the map $\mu \mapsto f_h(\mu)$ is Fr$\mathrm{\acute{e}}$chet differentiable (indeed \(C^\infty\)) as a map from \(\mathcal{S}\) into \((V_{h,0},\|\cdot\|_{0,h})\), where $ \|f_h\|_{0,h} := \|X(f_h)\|_2.$
 \item  Further given a $L<1$, for any direction \(\dot\mu\), there is a constant \(C(L,\kappa,p,q)>0\), the derivative \(\dot f_h := D f_h(\mu)[\dot\mu]\) satisfies that
$$\|\dot f_h\|_{0,h} \le C(L,\kappa,p,q)\,\|\dot\mu\|_\infty\quad\text{for all }\mu\text{ with }\|\mu\|_\infty\le L.$$
\end{enumerate}

\end{theorem}
\begin{proof}
Define $ \widehat F:\ \mathcal{S}\times\mathbb{R}^{2N-2|P|} \to\mathbb{R}^{2N-2|P|} $ by $$ \widehat F(\mu,\widehat X) := A(\mu)^T\big(A(\mu)\widehat X + A_{\mathrm{full}}(\mu)\delta(P,q)\big). $$
Then the reduced LSQC minimizer \(\widehat X(\mu)\) is characterized by $$ \widehat F(\mu,\widehat X(\mu)) = 0. $$
It is natural that each coordinate of $\widehat F(\mu,\widehat X) $ is a polynomial in \((\mu,\widehat X)\). Therefore \(\widehat F\) is a \(C^\infty\) map from \( (\mathcal{S},\|\cdot\|_{\infty})\times (\mathbb{R}^{2N-2|P|},\|\cdot\|_2 )\) to \( (\mathbb{R}^{2N-2|P|},\|\cdot\|_2 )\). Fix \(\mu_0\in\mathcal{S}\) and let \(\widehat X_0 = \widehat X(\mu_0)\) be the LSQC minimizer coefficients at \(\mu_0\), i.e. $$ \widehat F(\mu_0,\widehat X_0) = 0. $$
Compute the derivative of \(\widehat F\) with respect to \(\widehat X\):

For any direction \(\dot{\widehat X}\in\mathbb{R}^{2N-2|P|}\), $$ D_{\widehat X}\widehat F(\mu_0,\widehat X_0)[\dot{\widehat X}] = A(\mu_0)^TA(\mu_0)\,\dot{\widehat X}. $$

So $$ D_{\widehat X}\widehat F(\mu_0,\widehat X_0) = N(\mu_0) = A(\mu_0)^TA(\mu_0). $$ By Proposition \ref{full rank of A}, \(N(\mu_0)\) is a linear isomorphism.
By the implicit function theorem \cite{lang2012fundamentals}, there exists a neighborhood \(U\subset\mathcal{S}\) of \(\mu_0\) and a unique \(C^1\) map
$$ \mu\in U \mapsto \widehat X(\mu)\in\mathbb{R}^{2N-2|P|} $$ such that $$ \widehat F(\mu,\widehat X(\mu)) = 0 \quad\forall\mu\in U. $$
Since \(\widehat F\) is \(C^\infty\), the solution map \(\widehat X(\mu)\) is actually \(C^\infty\) in \(\mu\).
Now define the LSQC minimizer \(f_h(\mu)\in V_{h,0}\) by $$ X(f_h(\mu)) := \mathrm{lift}(\widehat X(\mu)). $$
Then the map $$ \mu\mapsto f_h(\mu) $$ is Fréchet differentiable (and \(C^\infty\)) from \(\mathcal{U}\) into \((V_{h,0},\|\cdot\|_{0,h})\).

Suppose $\|\mu\|_{\infty}\le L$ for some $L<1$ and $\dot\mu\in\mathbb{C}^{|F|}$, there exists $\epsilon>0$ such that $\|\mu+\epsilon\dot\mu\|_{\infty}\le \frac{1+L}{2}$ and write \(\mu(t)=\mu+t\dot\mu\), \(\widehat X(t)=\widehat X(\mu(t)),0<t<\epsilon\), then $$ N(\mu(t))\,\widehat X(t) = -A(\mu(t))^T A_{\mathrm{full}}(\mu(t))\,\delta(P,q), \,0<t<\epsilon. $$
Differentiate both sides at \(t=0\), we have $$ D_\mu N(\mu)[\dot\mu]\,\widehat X(\mu)- N(\mu)\,\widehat X(\dot f_h) = -D_\mu\big(A(\mu)^T A_{\mathrm{full}}(\mu)\big)[\dot\mu]\,\delta(P,q),$$
then $$ \widehat X(\dot f_h) = - N(\mu)^{-1}\Big(D_\mu N(\mu)[\dot\mu]\,\widehat X(\mu)+ D_\mu\big(A(\mu)^T A_{\mathrm{full}}(\mu)\big)[\dot\mu]\,\delta(P,q)\Big). $$
This equation holds because $\|\mu+t\dot\mu\|_{\infty}\leq \frac{1+L}{2}<1, \, \forall t \in [0,\epsilon]$.
By Theorem \ref{thm 3.1}, $\|N(\mu)^{-1}\|\le\frac{1}{L_{N,1}}$ for some constant $L_{N,1}$ depending on $L$. On each face \(T\), the entries of \(A_{\mathrm{full}}(\mu)\) are linear combinations of geometric coefficients (edge vectors, normalized by \(\sqrt{d_T}\)), multiplied by \((1\pm\mu_T)\) or their real or imaginary parts. Concretely, a typical (complex) coefficient corresponding to a face \(T\) is $$\frac{1}{\sqrt{d_T}}\big[(1+\rho_T)(x_3-x_2)+\tau_T(y_3-y_2)\big],\quad \mu_T=\rho_T+i\tau_T.$$ Therefore, each entry \(a_{ij}(\mu)\) of \(A_{\mathrm{full}}(\mu)\) is affine in \(\mu_T\) and the derivative with respect to \(\mu_T\) yields a constant (in \(\mu\)), times the same geometric factors. Let $\bar A_{\mathrm{full}}(\dot\mu) := D_\mu A_{\mathrm{full}}(\mu)[\dot\mu]$ denote the “directional derivative matrix” in direction \(\dot\mu\). Then \(\bar A_{\mathrm{full}}(\dot\mu)\) has the same sparsity pattern as \(A_{\mathrm{full}}(\mu)\) and a typical nonzero entry is $$\frac{1}{\sqrt{d_T}}\big[\dot{\rho}_T(x_3-x_2)+\dot{\tau}_T(y_3-y_2)\big],\quad\dot\mu_T=\dot\rho_T+i\dot\tau_T.$$
and by Lemma \ref{Lipschitz dependence of A}, we have $$ \Big|\frac{1}{\sqrt{d_T}}\big[\dot{\rho}_T(x_3-x_2)+\dot{\tau}_T(y_3-y_2)\big]\Big| \le C(\kappa)\,|\dot\mu_T|$$
for some constant \(C(\kappa)\) depending only on the shape regularity \(\kappa\). Therefore, there exists a constant $C_0(\kappa)$ such that the spectral norm $\bar A(\dot\mu) := D_\mu A(\mu)[\dot\mu]$ and $\bar A_{\mathrm{full}}(\dot\mu$ are both bounded by $C_0(\kappa)\|\dot\mu\|_{\infty}$. Moreover, for bounded \(\|\mu\|_\infty\le L\), the entries of \(A(\mu)\) and \(A_{\mathrm{full}}(\mu)\) are uniformly bounded in terms of \(\kappa\) and \(L\), so there exists \(C_{1}(\kappa,L)\) such that  $$ \|A(\mu)\|_2 \le C_{1}(\kappa,L),\quad \|A_{\mathrm{full}}(\mu)\|_2 \le C_{1}(\kappa,L).$$
Hence  \begin{align*} \|D_\mu N(\mu)[\dot\mu]\|_2 &=\|\bar A(\dot\mu)^T A(\mu) + A(\mu)^T \bar A(\dot\mu)\|_2 \\ &\le \|\bar A(\dot\mu)^T A(\mu)\|_2 + \|A(\mu)^T \bar A(\dot\mu)\|_2\\ &\le \|\bar A(\dot\mu)\|_2\,\|A(\mu)\|_2 + \|A(\mu)\|_2\,\|\bar A(\dot\mu)\|_2\\  &\le 2\,C_0(\kappa)\,C_{1}(\kappa,L)\,\|\dot\mu\|_\infty, \end{align*} 
and we denote $c_2(\kappa,L)=C_0(\kappa)C_1(\kappa,L)$. Similarly,  $$\big\|D_\mu\big(A(\mu)^T A_{\mathrm{full}}(\mu)\big)[\dot\mu]\big\|_2 \le C_3(\kappa,L)\,\|\dot\mu\|_\infty.$$ 
Combining all, we have  \begin{align*} \|\widehat X(\dot f_h)\|_2 &\le \frac{1}{L_{N,1}} \Big( \|D_\mu B(\mu)[\dot\mu]\|_2\,\|\widehat X(\mu)\|_2+\|D_\mu(A(\mu)^T A_{\mathrm{full}}(\mu))[\dot\mu]\|_2\,\|\delta(P,q)\|_2\Big)\\ &\le \frac{1}{L_{N,1}} \Big( C_2(\kappa,L)\,\|\dot\mu\|_\infty\,\|\widehat X(\mu)\|_2+ C_3(\kappa,L)\,\|\dot\mu\|_\infty\,\|\delta(P,q)\|_2
\Big)\\ &= \frac{\|\dot\mu\|_\infty}{L_{N,1}} \Big( C_2(\kappa,L)\,\|\widehat X(\mu)\|_2+ C_3(\kappa,L)\,\|\delta(P,q)\|_2\Big). \end{align*}
Since the supremum \(\sup\limits_{\|\mu\|_\infty\le L}\|\widehat X(\mu)\|_2\) is finite and depends only on \((L,\kappa,p,q)\). Thus $$ \|\widehat X(\dot f_h)\|_2 \le C(L,\kappa,p,q)\,\|\dot\mu\|_\infty. $$ where \(C(L,\kappa,p,q)= C_2(\kappa,L)\sup\limits_{\|\mu\|_\infty\le L}\|\widehat X(\mu)\|_2+ C_3(\kappa,L)\,\|\delta(P,q)\|_2 \) is a finite constant. 
\end{proof}

Next, if we define the $H^1$-seminorm of $f_h$ by $$|f_h|_{1,h}^{2}:=\sum_{T}d_T\big(|f_{h,x}^T|^2+|f_{h,y}^T|^2\big)$$ and we have a useful inequality.
\begin{lemma}[gradient-vertex inequality]
Suppose the triangular mesh satisfies the same regularity condition in Theorem \ref{thm 3.1}. There exists a constant $C(\kappa)$ such that $$|f|_{1,h}^2\le C(\kappa)\|f\|_{0,h}^2$$
\end{lemma}
\begin{proof}
Let $u=\Re(f),v=\Im(f)$, because derivatives are linear in \(u\) and \(v\), \[ |f_h|_{1,h}^2 = |u_h|_{1,h}^2 + |v_h|_{1,h}^2, \] with \[ |u_h|_{1,h}^2 := \sum_T d_T\,(|u_x^T|^2+|u_y^T|^2), \quad |v_h|_{1,h}^2 := \sum_T d_T\,(|v_x^T|^2+|v_y^T|^2), \]
and the coefficient norm is \[ \|f_h\|_{0,h}^2 := \|X(f_h)\|_2^2 = \sum_{i=1}^N (u_i^2 + v_i^2). \]
It suffices to prove \[ |u_h|_{1,h} \;\le\; C_1(\kappa)\,\|u_h\|_{0,h}, \qquad |v_h|_{1,h} \;\le\; C_1(\kappa)\,\|v_h\|_{0,h} \] for real scalar P1 functions. Since 
\[ d_T\,|u_x^T|^2 = |D\cdot u_T|^2, \quad D = \frac{1}{\sqrt{d_T}}(y_3-y_2,\;y_1-y_3,\;y_2-y_1), \] where \(u_T:=(u_1,u_2,u_3)^\top\) is the vector of nodal values on \(T\).

A similar expression holds for \(d_T\,|u_y^T|^2\) with another geometric vector \(E\) of the form \[ E = \frac{1}{\sqrt{d_T}}(x_2-x_3,\;x_3-x_1,\;x_1-x_2). \]

Thus \[ d_T\big(|u_x^T|^2+|u_y^T|^2\big) = |D\cdot u_T|^2 + |E\cdot u_T|^2 = \|(D,E)u_T\|_2^2, \] where \((D,E)\) is the \(2\times 3\) matrix with rows \(D\) and \(E\). We have \[ d_T\big(|u_x^T|^2+|u_y^T|^2\big) = \|(D,E)u_T\|_2^2 \le \|(D,E)\|_2^2\,\|u_T\|_2^2, \] where \(\|\cdot\|_2\) is the operator (spectral) norm from \(\mathbb{R}^3\) to \(\mathbb{R}^2\). By an estimate derived in the proof of Lemma \ref{Lipschitz dependence of A}, there exists a constant \(C_{\text{loc}}(\kappa)\) such that \[ \|(D,E)\|_2 \;\le\; C_{\text{loc}}(\kappa). \] 
Hence \[ d_T\big(|u_x^T|^2+|u_y^T|^2\big) \le C_{\text{loc}}(\kappa)^2 \,\|u_T\|_2^2. \]
That is, for each element \(T\), \[ d_T\big(|u_x^T|^2+|u_y^T|^2\big) \le C_{\text{loc}}(\kappa)^2\,(u_1^2+u_2^2+u_3^2). \]
Sum  over all triangles \(T\): \[ |u_h|_{1,h}^2 = \sum_{T\in\mathcal{F}} d_T\big(|u_x^T|^2+|u_y^T|^2\big) \le C_{\text{loc}}(\kappa)^2 \sum_{T\in\mathcal{F}}\sum_{i\in T} u_i^2. \]

Each vertex \(i\) belongs to a finite number of triangles \(\mathcal{F}_i\), bounded uniformly in terms of the shape-regularity which is determined by $\kappa$. Let \(M(\kappa)\) be the maximal number of elements sharing a vertex, then each nodal square \(u_i^2\) appears at most \(M(\kappa)\) times in the double sum \(\sum_{T}\sum_{i\in T}\). Hence \[ \sum_{T\in\mathcal{F}}\sum_{i\in T} u_i^2 \le M(\kappa)\,\sum_{i=1}^N u_i^2. \]
Thus \[ |u_h|_{1,h}^2 \le C_{\text{loc}}(\kappa)^2\,M(\kappa)\,\sum_{i=1}^N u_i^2 = \big(C_{\text{loc}}(\kappa)\sqrt{M(\kappa)}\big)^2\,\|u_h\|_{0,h}^2. \]
Taking square roots, \[ |u_h|_{1,h} \le C_1(\kappa)\,\|u_h\|_{0,h}, \quad C_1(\kappa) := C_{\text{loc}}(\kappa)\sqrt{M(\kappa)}. \]
Exactly the same reasoning applies to \(v_h\), so \[ |v_h|_{1,h} \le C_1(\kappa)\,\|v_h\|_{0,h}. \]
Finally, for \(f_h = u_h + \sqrt{-1} v_h\), \[ |f_h|_{1,h}^2 = |u_h|_{1,h}^2 + |v_h|_{1,h}^2 \le C_1(\kappa)^2(\|u_h\|_{0,h}^2 + \|v_h\|_{0,h}^2) = C_1(\kappa)^2\,\|f_h\|_{0,h}^2, \] so \[  |f_h|_{1,h} \le C_1(\kappa)\,\|f_h\|_{0,h}.\] 
\end{proof}
With this, here is a corollary by combining Theorem \ref{thm 3.1} and Theorem \ref{differentiability of the discrete LSQC minimizer} with this lemma.
\begin{corollary}
    Under the same assumptions in Theorem \ref{thm 3.1} and Theorem \ref{differentiability of the discrete LSQC minimizer}, there exists $C_0=C_0(\kappa,L,p,q)$ such that$$|f_h(\mu)-f_h(\nu)|_{1,h} \le C_0\,\|\mu-\nu\|_\infty.$$
    Similarly, for the Fréchet derivative \(D f_h(\mu)\) there exists $C_1=C_1(\kappa,L,p,q)$ such tha satisfies
    $$|D f_h(\mu)[\dot\mu]|_{1,h}\le C_1\,\|\dot\mu\|_\infty.$$ 
\end{corollary}

In all, for each triple $(\mu,P,Q)$ where $P$ is the index set of pinned points and $Q$ consists of images of pinned point in $P$ with $|P|=|Q|\ge2$, let
$\mathcal{F}(\mu,P,Q)$ denote the LSQC minimizer mapping, i.e. the unique solution of the LSQC linear system with Beltrami field $\mu$ and pinned constraints $p_{i} \mapsto q_{i}$, and $\mathcal{L}_1:V_{h}\to\mathbb{R}$ be the task loss, the optimization problem can be solved as
$$\min_{P,Q,\mu}\mathcal{L}_1\big(\mathcal{F}(\mu,P,Q)\big)+\mathcal{L}_2(\mu).$$
In practice,  $|P|=2$ and we can show the existence of the minimizer $(\mu^*,P^*,Q^*)$ for the problem under some mild conditions.
\begin{theorem}[Existence of a minimizer for the optimization over Beltrami fields]
Suppose that the triangular mesh $\mathcal{T}=(V,F)$ satisfies the regularity condition in Theorem \ref{thm 3.1}, $\mathcal{L}_1,\mathcal{L}_2$ is continuous, $\mathcal{P}$ is the collection of all two-element subsets of $(1,2,...,|V|)$ and $\mathcal{Q}$ is a compact subset of $\mathbb{C}^2$ . Let $\mathcal{M}=\{\mu:\|\mu\|_{\infty}\le L\}$ for some $L<1$ then the optimization problem
\[
\min\limits_{\substack{\mu \in \mathcal{M},\\P\in \mathcal{P},\,Q\in \mathcal{Q}}} \mathcal{L}_1\big(\mathcal{F}(\mu,P,Q)\big)+\mathcal{L}_2(\mu).
\]
admits a minimizer.
\end{theorem}
\begin{proof}
    For a fixed $P\in\mathcal{P}$, by Proposition \ref{prop: independent of similarity transformation} and Theorem \ref{thm 3.1}, the map $(\mu,Q)\to \mathcal{F}(\mu,P,Q)$ is continuous on $\mathcal{M}\times \mathcal{Q}$ and consequently $(\mu,Q)\to \mathcal{L}_1\big(\mathcal{F}(\mu,P,Q)\big)+\mathcal{L}_2(\mu)$ is continuous. Since $\mathcal{M}\times \mathcal{Q}$ is compact, \[
\min\limits_{\substack{\mu \in \mathcal{M},\,Q\in \mathcal{Q}}} \mathcal{L}_1\big(\mathcal{F}(\mu,P,Q)\big)+\mathcal{L}_2(\mu).
\]
has a minimizer. Lastly since $|\mathcal{P}|=\big(\substack{|V|\\2}\big)<\infty$, the optimization problem admits a minimizer. 
\end{proof}

\begin{remark}
It is a common choice to augment the task-specific loss with regularization terms $\mathcal{L}_2(\mu)$ of the form \[ \alpha \int_{\Omega} |\mu|^{2}+\beta \int_{\Omega} |\nabla\mu|^{2}. \]
The first term penalizes the magnitude of the Beltrami coefficient and therefore the overall distortion. The second term penalizes the spatial variation of \(\mu\), promoting a smoother distortion field. The stability and differentiability results make this choice mathematically well justified: the LSQC minimizer depends Lipschitz continuously  on \(\mu\). Thus, controlling \(\mu\) in suitable Sobolev norms (such as \(\lVert \mu \rVert_{L^{2}}\) and \(\lVert \nabla\mu \rVert_{L^{2}}\)) yields control over the corresponding mapping in discrete \(H^{1}\) norms. In other words, the regularizer \(\int |\nabla\mu|^{2}\) acts as a smoothness prior on the mapping itself.
\end{remark}

%% file: sections/sec6-Spectral_Beltrami_Network.tex
\section{Spectral Beltrami Network}
\label{sec: Spectral Beltrami Network}
With these advantages, LSQC should have been a perfect tool to derive a free-boundary quasiconformal map. However, there is a drawback in this method. Since $\mathcal{F}$ involves solving a huge sparse linear system, it is computationally expensive or impossible to backpropagate the gradient of loss function $\mathcal{L}$ w.r.t the mapping $\frac{\partial \mathcal{L}}{\partial f}$  to that w.r.t BC $\frac{\partial \mathcal{L}}{\partial \mu}$,  destinations of pinned points $\frac{\partial \mathcal{L}}{\partial q_i}$, strictly following the chain rule. Moreover, the choice of the pinned points is a discrete manner and there lacks a clear mechanism to update it according to the loss values.
The Alternating Direction Method of Multipliers (ADMM) partially mitigates this by updating 
$\mu$ based on the BCs reconstructed from the newly updated mapping $f$, but it does not address the lack of explicit gradients for pinned points $p_i$ and destinations $q_i$. Previous approaches, such as those by Qiu \cite{qiu2019computing, qiu2020inconsistent} circumvent this by imposing destinations of certain points in advance, which is often unrealistic in real-life scenarios.

To overcome these limitations, we propose the Spectral Beltrami Network (SBN), a differentiable neural surrogate for $\mathcal{F}$ that enables gradient backpropagation through $\mu$ and $p$. The network operates on the unit disk $\mathbb{D}$. This is a canonical choice: by the uniformization theorem, many surfaces of interest (after reducing to disk topology) admit a conformal parameterization to $\mathbb{D}$, and in practice there have been many efficient algorithms \cite{choi2015fast,delillo2004schwarz} for computing the disk or circular conformal parameterizations.
We represent the domain as a triangular mesh viewed as a directed graph $\mathcal{G} = (V,E,F)$ with nodes $V$, directed edges $E$, and faces $F$. SBN leverages three nested meshes $\mathcal{G}^{i} = (V^{i},E^{i},F^{i})$ $(i=1,2,3)$ of decreasing resolution. Given BC $\mu_{v}$ on all $v \in V^{1}$ and two points $p_1, p_2 \in V^{1}$, the network predicts a mapping $F$ approximating the numerical LSQC solution $f$ of $\{\mu_T \colon T \in F^{1}, \mu_{T} = \frac{1}{3}\sum_{v \in T}\mu_{v}\}$ with fixed constraints $f(p_i)=p_i,i=1,2$.
Enforcing selected two points to be fixed is based on two practical considerations. First, it constrains the output range, preventing training instabilities from unbounded target positions. Second, arbitrary target positions $q_1,q_2$ can be recovered via post-composition with a similarity transformation $g$ satisfying $g(p_j)=q_{j}$ $(j=1,2)$. By Proposition \ref{prop: independent of similarity transformation}, $\mathcal{F}(\mu,p_1,p_2,q_1,q_2)=g\circ \mathcal{F}(\mu,p_1,p_2,p_1,p_2)$, ensuring no loss of generality.

Therefore, in downstream applications after the SBN is well trained, parameters to be optimized include BC, coordinates of two fixed points, and parameters for scaling, rotation and translation. The challenge lies in incorporating the BC and pinned point information into the network and building efficient short- and long-range interactions between vertices, depicting the dynamics of LSQC energy.

\subsection{Encoding input information}
\label{subsec: encode pde}
SBN operates on three types of graphs: meshes $(\mathcal{G}^i, i=1,2,3)$, downsampling graphs $(\mathcal{G}^{\text{down},i} = (V_i \cup V_{i+1}, E^{i,i+1}))$, and upsampling graphs $(\mathcal{G}^{\text{up},i} = (V_i \cup V_{i+1}, E^{i+1,i}))$ for $i=1,2$. Each vertex $v \in V_1$ is assigned BC $\mu_v$ and encoded with a 13-dim feature vector
$\tilde{v} = [\mu_v, K_v, \arg(\mu_v), |\mu_v|, v - p_1, \|v - p_1\|_2, v - p_2, \|v - p_2\|_2, v]$
where $K_v = \frac{1 + |\mu_v|}{1 - |\mu_v|}$ captures distortion and the pinned-point features encode positional relationships. An encoder MLP maps $\tilde{v}$ to 24-dim embeddings. For coarser meshes $(V_2, V_3)$, if $v = \sum_{j=1}^3 \lambda_j v_j$ with $v_j \in V_1$ and $\sum_j \lambda_j = 1$, then $\mu_v = \sum_{j=1}^3 \lambda_j \mu_{v_j}$ by barycentric interpolation, yielding analogous 24-dim embeddings.
Edge features are computed as $e^{k,l}_{i,j} = f^{en}_e(\tilde{v}_j-\tilde{v}_i, \|v_j-v_i\|_{2})$ for edges from $v_i \in V^k$ to $v_j\in V^l$, where $f_e^{en}$ is an encoding MLP capturing directed vertex differences. Superscripts $(k,l)$ are omitted when clear from context.

\subsection{Message Passing Mechanism}
We adopted the multiscale message passing from \cite{fortunato2022multiscale} and improved the way of hierarchical edge construction. For downsampling graph $\mathcal{G}^{\text{down},i}=(V^{i} \cup V^{i+1},E^{i,i+1})$, each node $v_k \in V^{i}$ is connected to the three corners of its containing triangle in $F^{i+1}$ (identified by minimal $|\alpha_{k,l}| + |\alpha_{k,m}| + |\alpha_{k,n}|$ in barycentric coordinates $v_k = \alpha_{k,l}v_l + \alpha_{k,m}v_m + \alpha_{k,n}v_n$). Considering the free-boundary nature of LSQC, it is crucial to enhance boundary information exchange and thus we add edges from fine-mesh boundary nodes to any coarse-mesh vertex within a distance threshold. Upsampling graph $\mathcal{G}^{\text{up},i} = (V^{i} \cup V^{i+1}, E^{i+1,i})$ reverses edge directions. See Figure \ref{fig: Message Passing Mechanism}.
The general update Rules for node and edge embeddings are as follows: For graph $\mathcal{G}$ with edges $(v_i, v_j)$ where $v_i \in V^k, v_j \in V^l$:
$$
e_{i,j} = f^{k,l}_{e}(e_{i,j}, v_i, v_j, \theta_{e}), \quad  v_j = f^{k,l}_{v}\left(v_j, \sum_{e_{i,j}\in E^{k,l}}e_{i,j}, \theta_{v}\right).
$$
Update Rules are applied sequentially to $\mathcal{G}^{k}$ (intra-mesh, $k=l$), $\mathcal{G}^{\text{down},k}$ (fine-to-coarse, $k<l$), and $\mathcal{G}^{\text{up},k}$ (coarse-to-fine, $k>l$). Here $f_{e}, f_{v}$ are residual MLPs \cite{he2016deep} with GraphNorm \cite{cai2021graphnorm}; parameters $\theta_e, \theta_v$ encode PDE constraints for robustness. This multiscale mechanism extends interaction ranges, allowing each node to reach the entire mesh through hierarchical edges and enhanced boundary broadcasting. Figure \ref{fig:mp_block_arch} is the architecture of a message passing block.
\subsection{Mesh Spectral Layer}
A critical limitation persists: purely local message passing struggles to propagate information between distant nodes, especially in dense meshes where signals attenuate over multiple iterations. This is particularly problematic for our free-boundary problem, where pinned-point positions exert global influence on the mapping. Experimentally, we observed that multiscale message passing alone risks neural collapse, producing near-identity outputs.
\begin{figure}
    \centering
    \includegraphics[width=\linewidth]{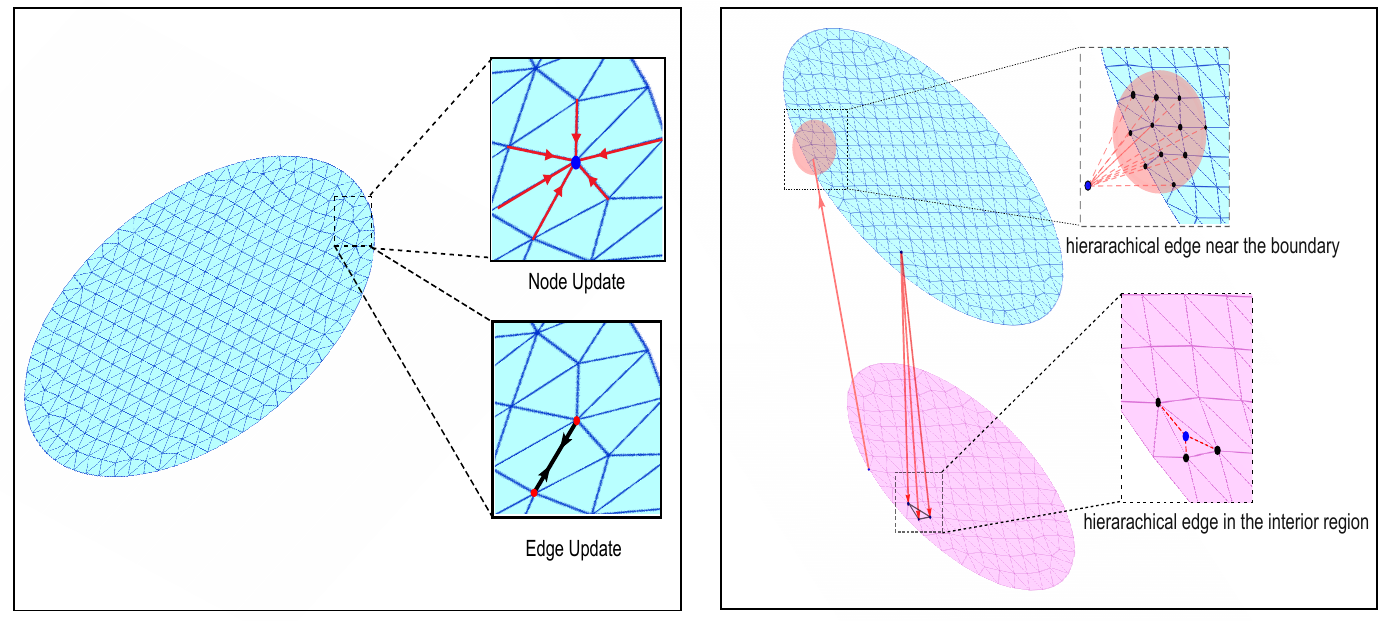}
    \caption{  Illustration of the embedding update rule and the hierarchical edge construction.}
    \label{fig: Message Passing Mechanism}
\end{figure}
\begin{figure}
\subfloat[ Architecture of Message Passing Block.]   
  {
      \centering
		\includegraphics[width=0.47\textwidth]{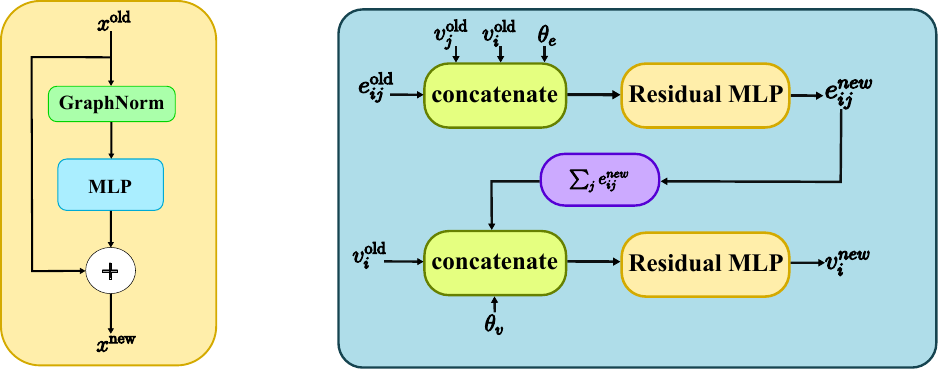}
        \label{fig:mp_block_arch}
  }
  \hfill
  \subfloat[Architecture of Mesh Spectral Layer.]
  {
      \centering
		\includegraphics[width=0.47\textwidth]{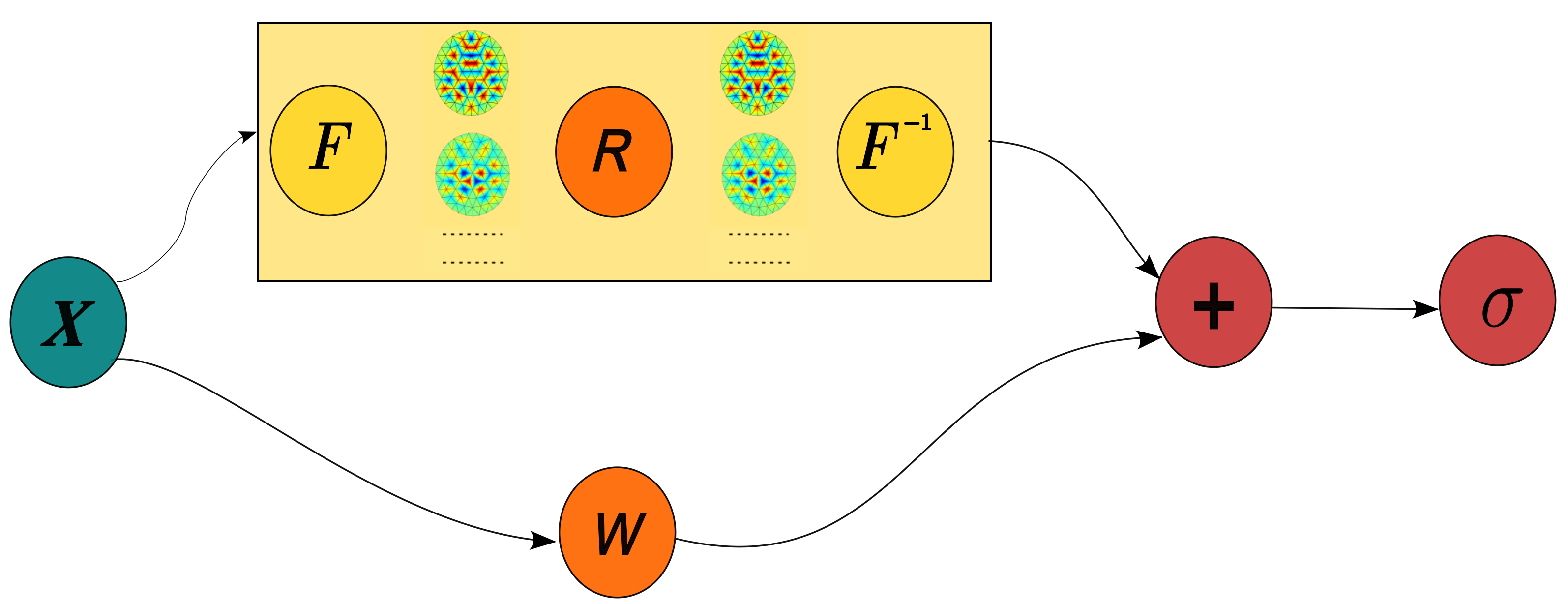}
        \label{fig: MeshSpectralLayer}
  }
  \caption{ Architecture of key modules in the Spectral Beltrami Network.}   
\end{figure}

Inspired by Fourier Neural Operator (FNO)\cite{li2020fourier}, we proposed the Mesh Spectral Layer (MSL) to capture global dependencies via the mesh Laplacian's eigenbasis. For a triangular mesh $(V,E,F)$ with Laplacian $\Delta = \mathbf{M}^{-1}\mathbf{L}$, where
    $$
\mathbf{L}_{ij} = 
\begin{cases} 
\frac{1}{2}(\cot \alpha_{ij} + \cot \beta_{ij}) & \text{if } j \sim i \\ 
-\sum_{j' \in \mathcal{N}(i)} w_{ij'} & \text{if } j = i \\ 
0 & \text{otherwise}
\end{cases},
\quad \mathbf{M}_{ik} = \begin{cases} 
\frac{1}{3}\sum_{T_j\in\mathcal{N}(i)}\text{area}(T_j) & \text{if } i=k \\  
0 & \text{otherwise}
\end{cases}
,
$$
 where $\alpha_{ij}$ and $\beta_{ij}$ are the angles opposite to the edge $[i,j]$ and $\mathcal{N}(i)$ is the list of adjacent vertices and faces of the vertex $i$. Since $\Delta$ is symmetric positive semi-definite and encodes mesh geometry, we project latent features $X\in \mathbb{R}^{|V|\times d}$ onto its $k$ normalized eigenvectors $F \in \mathbb{R}^{k \times |V|}$ associated to the $k$ smallest eigenvalues, performing spectral-domain mixing:
$$
X \leftarrow \sigma(XW + F^T(R \cdot (FX))), \quad (R \cdot (FX))_{jl} = \sum_{n=1}^{d} R_{jln}(FX)_{jn}
$$
with trainable $W \in \mathbb{R}^{d \times \tilde{d}}$ and $R \in \mathbb{R}^{k \times \tilde{d} \times d}$. This offers efficient global communication with far fewer parameters than attention mechanisms \cite{vaswani2017attention}. Figure \ref{fig: MeshSpectralLayer} shows the architecture of this module.

\subsection{Model architecture and training}
A key principle in combining global operators with local message passing is that each node should carry structural information beyond individual messages. Figure \ref{fig: SBN} shows the complete SBN architecture integrating multiscale message passing blocks and mesh spectral layers.
\begin{figure}[!htbp]
    \centering
    \includegraphics[width=\textwidth]{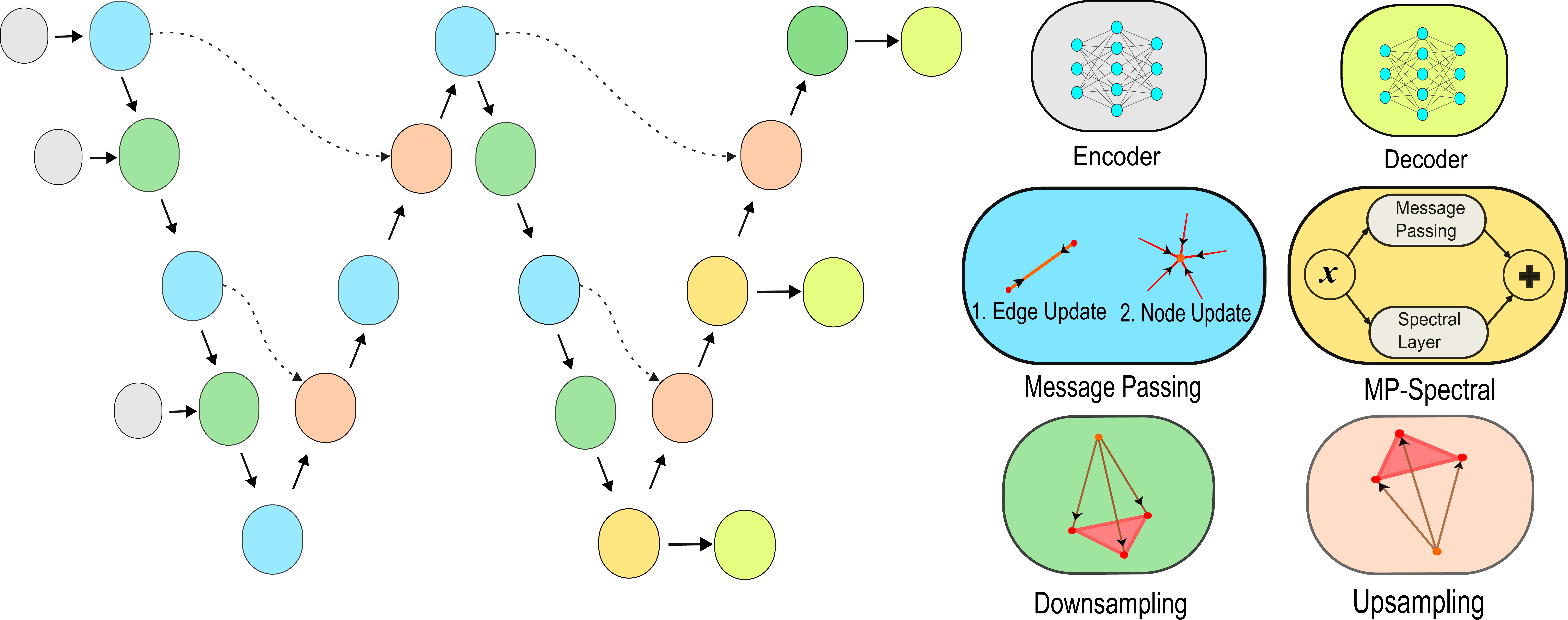}
    \caption{ Architecture of Spectral Beltrami Network.}
    \label{fig: SBN}
\end{figure}

SBN is trained by supervising the predicted mapping $\hat{u}=\mathcal{F}_{\theta}(\mu_{V^{1}},p_1,p_2)$ against ground truth $u =\mathcal{F}_{\theta}(\mu_{F^1},p_1,p_2,p_1,p_2)$ where $\mu_{F^1}=\{\mu_T|\ \mu_T=\frac{1}{3}\sum_{v\in T}\mu_v,T \in F^1\}$, via four loss components:
\begin{enumerate}
    \item \textbf{Boundary-weighted node loss} emphasizes boundary accuracy where free-boundary errors concentrate:
   $$\mathcal{L}_{1}(V^1) = \frac{1}{|V^{1}|}\sum_{v\in V^{1}}e^{\|u(v)\|_{2}} \|u(v)-\hat{u}(v)\|^{2}_{2}.$$
   \item \textbf{Relative edge deformation loss} prevents mesh folding by penalizing inconsistent edge deformations:
   $$\mathcal{L}_2(E^1)=\frac{1}{|E^{1}|}\sum_{[i,j]\in E^{1}} \frac{\|(\hat{u}(v_i)-\hat{u}(v_j))-(u(v_i)-u(v_j)) \|_2}{\|u(v_i)-u(v_j) \|_{2}}.$$
   \item \textbf{Boundary $\ell^1$ loss} enhances boundary fitting:
   $$\mathcal{L}_3(\partial V^1)=\frac{1}{|\partial V^1|}\sum_{v\in\partial V^{1}}\| u(v)-\hat{u}(v)\|_{1}.$$
   \item \textbf{Deep supervision} extends losses to coarser meshes $V^2, V^3$ (with interpolated ground-truth) to enhance gradient flow and representation learning:
   $$\mathcal{L}_1^{ds}(V^1,V^2,V^3) = \mathcal{L}_1(V^1) + \lambda_2(t) \mathcal{L}_1(V^2) + \lambda_3(t) \mathcal{L}_1(V^3),$$
   where $\lambda_j(t) = \lambda \cdot \frac{t}{N} + \lambda_j \cdot (1-\frac{t}{N})$ gradually decays coarse-mesh influence ($\lambda_2=\lambda_3=0.2$, $\lambda=0.1$).
\end{enumerate}

In conclusion, the total training loss is a weighted sum of these loss terms. Once the training of SBN was finished, we obtained a neural network $\mathcal{F}_{\theta^*}$ and during the optimization, its parameters are frozen.

%% file: sections/sec7-SBN-Opt.tex
\section{SBN-Opt: A novel optimization framework of the free-boundary  diffeomorphism problem}
\label{sec: neural lsqc}

A well trained SBN $\mathcal{F}_{\theta^*}$ serves as a differentiable solver for optimization of free-boundary diffeomorphism. We optimize input BC $\{\mu_v\}$, pinned points $\{p_1, p_2\}$, and post-composition similarity parameters $(\phi, s, r)$ to minimize task-specific energies. To ensure valid inputs ($\|\mu_v\|, \|p_i\| < 1$), we apply an activation function:
$$
\mathcal{T}(x,T) = \left( \frac{e^{\frac{|x|}{T}} - e^{-\frac{|x|}{T}}}{e^{\frac{|x|}{T}} + e^{-\frac{|x|}{T}}} \right)e^{i\arg(x)},
$$
setting $\mu_v = \mathcal{T}(\tilde{\mu}_v,T_{\text{BC}})$ and $p_i = \mathcal{T}(\tilde{p}_i, T_{\text{pin}})$ where $(\tilde{\mu}_v, T_{\text{BC}}, \tilde{p}_i, T_{\text{pin}})$ are unconstrained optimization variables. In this way, we introduce the SBN guided optimization framework SBN-Opt to solve the problem \ref{eq: formulation of optimization problem}
\begin{equation}
    \label{neural framework}
\min_{(\tilde{\mu}_v,T_{\text{BC}},\tilde{p}_i, T_{\text{pin}},\phi,s,r)} E_1\Big(g\big(\mathcal{F}_{\theta^*}(\mu_v,p_1,p_2),\phi,s,r\big)\Big) + E_2(\mu), \quad g(x,\phi,s,r) = se^{i\phi}x+r.
\end{equation}

A significant advantage of the SBN-Opt framework is that the pinned points $P = \{p_1, p_2\}$ are treated as learnable parameters rather than unchanged constraints. While the optimization of these points occurs in a continuous coordinate space, the Spectral Beltrami Network requires input to coincide with the discrete vertices of the mesh. During optimization, updated coordinates $\tilde{p}$ are unlikely to align exactly with existing mesh nodes. To resolve this discrepancy while maintaining a differentiable pipeline, we employ a straight-through estimation strategy. In each forward pass, we first identify the mesh vertex $v_{near} \in V$ closest to the current continuous estimate $\tilde{p}$. The input to the network, $p_{in}$, is then constructed using a stop-gradient operator as follows:$$p_{in} = \text{stop\_gradient}(v_{near} - \tilde{p}) + \tilde{p}$$ This formulation ensures that 1. the valid vertex $v_{near}$ is passed as the input to the SBN; 2. during the backward pass, by the stop-gradient operator, the gradient of the loss function $\mathcal{L}$ with respect to the input flows directly to the continuous parameters, i.e. $\frac{\partial \mathcal{L}}{\partial \tilde{p}} = \frac{\partial \mathcal{L}}{\partial p_{in}} $.

Direct optimization on Beltrami coefficients enables explicit distortion control via $E_2(\mu)$. For instance, angle distortion is penalized by $E_{\text{BC}} = \frac{1}{|V^1|}\sum_{v\in V^1}|\mu_v|^2$, while smoothness (assuming piecewise-linear $\mu$) is enforced via $E_{\text{smooth}} = \frac{1}{|F^1|}\sum_{T\in F^1}|\nabla \mu(T)|^2$. Although SBN outputs only on $V^1$, any point $v \in \mathbb{D}$ maps via barycentric interpolation. By resolution-independence (Proposition \ref{prop: resolution independence}), this remains faithful to the continuous solution. For each vertex in a high-resolution mesh $v_p \in V^{\text{orig}}$, we identify the triangle $I(p)=[v_{i}, v_{j}, v_{k}] \in F^1$ that minimizes $|\lambda_i| + |\lambda_j| + |\lambda_k|$ among barycentric representations $v_p = \lambda_i v_i + \lambda_j v_j + \lambda_k v_k$ (with $\lambda_i + \lambda_j + \lambda_k = 1$) in all triangles, and this identifies the containing triangle where all coordinates are non-negative. We precompute a sparse interpolation matrix $R\in\mathbb{R}^{|V^{\text{orig}}|\times|V^{1}|}$ where
$$R_{p,q}= \begin{cases} \lambda_{q}(v_p) & \text{if } v_q \text{ is a vertex of the identified triangle},\\ 0 & \text{otherwise}, \end{cases}$$
enabling per-iteration interpolation $f(V^{\text{orig}}) = R \cdot f(V^1)$ and the computation cost and memory is negligible relative to the forward pass of SBN.

Next, we take two image processing tasks in Section \ref{sec: experiments} as examples to illustrate how to apply the framework \ref{neural framework}. In the \textbf{Equiareal mapping problem}, given a triangular mesh $\mathcal{T}=(V,F)$ of a surface $\mathcal{M}$ and its conformal parameterization $\phi:\mathcal{M}\to\mathbb{D}$ as well as population $\{p(T):T\in F\}$, we seek a homeomorphism $\psi:\mathbb{D}\to\mathbb{C}$ such that the mapping $f=\psi \circ \phi$ preserves the area of $\mathcal{T}$, or to say, makes
density $\rho(f(T))=\frac{\text{Area}(T)}{\text{Area}(f(T))}$ uniform across faces. 
If without setting up a sea surrounding the disk\cite{choi2018density}, the region might arbitrary expand. However, setting up a sea might restrict the movement of the true boundary and make the area change near the true boundary damped. Instead, if needed, we can easily discourage arbitrary expansion of the region by designing a barrier function on the scaling parameter $s$ in the Framework \ref{neural framework}. Then the task loss objective in the formulation \ref{eq: formulation of optimization problem} for this task is
$$
E_1 = \sum_{T\in F} \bigl(\rho(T) - \bar\rho(f)\bigr)^2 + \lambda_{\text{reg}}\,R_\text{barrier}(s), \quad R_\text{barrier}(s)=\max \{0, s-\omega\},
$$
where $\bar\rho(f) = \frac{\sum_{T}\text{Area}(T)}{\sum_{T}\text{Area}(f(T))}$ is the target density. 

In the work \cite{qiu2020inconsistent}, Qiu et al. developed an efficient splitting method based on Alternating Direction Method with Multipliers(ADMM) to apply the numerical LSQC algorithm in the \textbf{inconsistent surface registration problem}. However, the algorithm relies on prescribed landmark point-to-point correspondence. One of advantages in our framework is that the pinned points are also optimization parameters therefore it can be applied to more general inconsistent surface registration problems. In practice, exact point to point correspondence is usually impractical and unreasonable. For one thing, manually labeling point to point correspondence is time-consuming. For another, measurement noise, meshing artifacts and smoothing alter vertex positions by millimeters thus treating those vertices as hard truths introduces error. Therefore, it is more common to relax the correspondence to the relation between point sets and point sets. Given surfaces $S_1, S_2$ with subsets $\{S_{i,j}\}$ and intensity functions $I_i:S_i\rightarrow \mathbb{R}$, we seek a diffeomorphism $f:S_1\rightarrow\mathbb{R}^3$ and optimal regions $\Omega_i^* \subset S_i$ minimizing:
\begin{enumerate}
    \item Intensity mismatch: $E_{\text{I}}(f,\Omega_1) = \int_{f(\Omega_1)}\|(I_1 \circ f^{-1}-I_2)\|_1dA$
    \item Point-set discrepancy (via chamfer distance): 
 
        \[E_{\text{pc}}(f)=\sum_{j}\Big( \frac{1}{|S_{1,j}|}\sum_{v\in S_{1,j}}\min_{w\in S_{2,j}}\|f(v)-w\|^2+ \frac{1}{|S_{2,j}|}\sum_{w\in S_{2,j}}\min_{v\in S_{1,j}}\|w-f(v)\|^2 \Big).\]

\end{enumerate}
After conformally parameterizing $S_i$ to $\mathbb{D}$ by $\phi_i$, the problem reduces to finding the optimal $g=\phi_2\circ f \circ \phi_1^{-1}:\mathbb{D}\rightarrow \mathbb{R}^2$ under the maximality assumption $g\circ \phi_1(\Omega_1)=\phi_2(\Omega_2)=f(S_1)\cap S_2$. We minimize $E_1 = \lambda_{\text{I}}\,E_{\text{I}}+\lambda_{\text{pc}}\,E_{\text{pc}}$ while $E_2(\mu)$ provides distortion regularization.
For the general optimization problem, the framework SBN-Opt is described in Algorithm \ref{algo: neural lsqc}.
\begin{algorithm}[htb]
    \caption{ SBN-Opt}
    \label{algo: neural lsqc}
    \begin{algorithmic}[1]
        \Require
            Data: A triangular mesh $\mathcal{T}=(V,F)$ of the unit disk $\mathbb{D}$. Parameters to be optimized:
            \begin{enumerate}
                \item Optimization parameters $\eta$ include $\tilde{\mu}_{v}$, $T_{\text{BC}}$, $\tilde{p}_{i}(i=1,2)$, $T_{\text{pin}}$, log of scaling  $\tilde{s}$, rotation angle $\phi$ and translation vector $r$.
                \item Loss weights $\lambda_j$, energy functionals $E_j$ and optimizer $\Psi$.
            \end{enumerate} 
        \Ensure
            Deformed mesh $f(\mathcal{T})=(f(V),F)$
        \State 
            Freeze the parameters of SBN $\mathcal{F}_{\theta^*}$.
        \State
            Compute and save the interpolation matrix $R\in \mathbb{R}^{|V|\times|V_1|}$ 
        \State 
            $continue$ = $True$;
        \While{ $continue$ }
            \State
            Prepare input $\mu = \mathcal{T}(\tilde{\mu},T_{\text{BC}}), p_i = \mathcal{T}(\tilde{p}_i,T_{\text{pin}})$ for the model $\mathcal{F}_{\theta^*}$.
            \State
            Obtain the mappings $f(V^1)=g\Big(\mathcal{F}_{\theta^*}\big(\mu_{v},p_1,p_2\big),\phi, e^{\tilde{s}},r\Big)\in \mathbb{C}^{|V^1|\times 1}$ and the deformed mesh $f(V)=R\,f(V^{1})$.
            \State Compute the weighted sum $E$ of losses $E_1\big(f(V)\big)$ and $E_2\big(\mu \big)$ with weights $\lambda_1,\lambda_2$.
            \If{$E$ and $E_j$ meet stop criterion}  
                \State
                    $continue$ = $False$; 
            \Else
                \State 
                Compute $\frac{\partial E}{\partial \eta}$ for all parameters $\eta$ and then update all $\eta$ with the optimizer $\Psi$.
            \EndIf
        \EndWhile
    \end{algorithmic}
\end{algorithm}

\begin{remark}
    In practice, a disk conformal parameterization is not mandatory. To avoid unnecessary conformal distortion caused by the fixed boundary shape, sometimes we use the method in \cite{levy2002least} to compute the free-boundary conformal parameterization of a surface. Again, due to Proposition \ref{prop: resolution independence} and Proposition \ref{prop: independent of similarity transformation}, we can still compute the deformation by interpolation. The idea is that, let $\phi:\mathcal{M}\rightarrow \mathbb{R}^2$ be a conformal parameterization of the surface $\mathcal{M}$, $r_0=\max\limits_{v\in \mathcal{M}}\|\phi(v)\|_2$ and $R$ is the sparse interpolation matrix of $\phi(\mathcal{M})$ with respect to the mesh $r_0\cdot \mathbb{D}=(r_0\cdot V^1,F^1)$, if $f:V^1\rightarrow \mathbb{R}^2$ is a deformation predicted by SBN, then deformation of $\psi(S)$ is $r_0\cdot Rf(V^1)$.
\end{remark}

%% file: sections/sec8-implementation.tex
\section{Implementation}
\label{sec: implementation}
We implemented all experiments, including the training of SBN and downstream applications, using the PyTorch framework  \cite{paszke2019pytorch}. The SBN model was trained in a supervised manner, with numerical solutions from the Least-squares Quasiconformal Energy (LSQC) solver serving as ground-truth mappings. The training dataset comprised 960 images from the ILSVRC2012 dataset \cite{russakovsky2015imagenet}, which provided realistic BCs with both low- and high-frequency components. To generate BCs, we randomly selected two images per training iteration, applied data augmentation techniques (random flipping, noise addition, and blurring), and normalized them to form the real and imaginary parts of the BCs. Two pinned points were randomly chosen from the vertices of the finest mesh for each training sample.

SBN employs a hierarchical architecture. It integrates two core components: multiscale message-passing mechanism and mesh spectral layer, as detailed in Section \ref{sec: Spectral Beltrami Network}. The numbers of message passing numbers per scale depends on the mesh resolution, which are 3,3,5 from high to low, and the number of MSL per scale is 2.  The training process ran for 250 epochs using the AdamW optimizer  \cite{loshchilov2017decoupled} with parameters  $\beta = (0.9, 0.999)$, a weight decay of $1 \times 10^{-2}$, an inital learning rate $3 \times 10^{-4}$. We used three triangular meshes of varying resolutions, generated with the MATLAB DistMesh function \cite{persson2004simple}, with initial edge lengths of 0.015, 0.3, and 0.5. The numbers of vertices and faces in these three meshes with resolution from high to low are (16114, 31803), (4027,7846) and (1452,2774), respectively. 
The loss function for training SBN combined three terms $\mathcal{L}_i^{ds} (i = 1,2,3)$ with weights all 1. Although we did not perform an exhaustive grid search for optimal hyperparameters, this configuration yielded strong performance, as demonstrated in Section \ref{sec: experiments}. All experiments were conducted on an NVIDIA A40 GPU with 44GB memory, ensuring efficient computation. In addition, the numerical algorithms were implemented using MATLAB on an 8-core Intel i7 machine.

%% file: sections/sec9-experiment.tex
\section{Experimental results}
\label{sec: experiments}
Numerical experiments were conducted to evaluate the efficacy and potential of SBN. Given that SBN’s performance depends on its ability to accurately predict mappings for various BCs $\mu$,  we first compare its predicted outputs with numerical solutions. This is followed by several applications in equiareal parameterization and inconsistent surface registration, encompassing both realistic and synthetic scenarios.
\subsection{Performance of Spectral Beltrami Network}
\label{subsec:self ablation}
SBN, trained over 250 epochs, achieves robust simulation of the Beltrami equation. Under typical conditions (BC with norms mostly below 0.4), it achieves the average nodewise $\ell^2$ error smaller than 0.03, relative to numerical LSQC solutions, demonstrating strong generalizability on unseen data (Figure \ref{fig:smaller BC performance}). The model maintains robustness in challenging scenarios involving larger BC norms (Figure \ref{fig:large_mu}) or significant domain distortions (Figure \ref{fig:large_distortion}), effectively handling complex or large deformations.
\begin{figure}
    \centering
    \includegraphics[width=\linewidth]{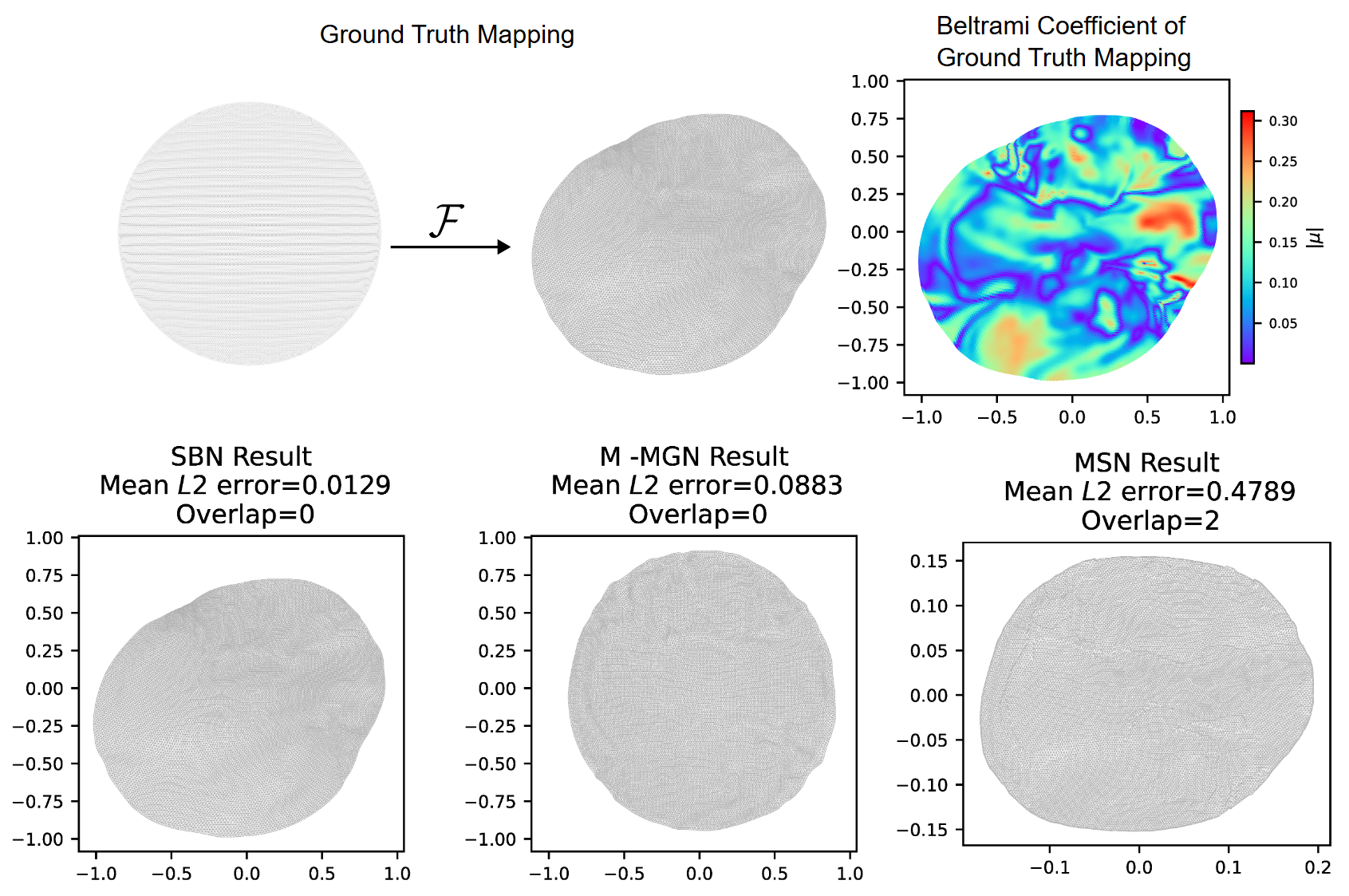}
    \caption{ Comparison of SBN, M-MGN, and MSN for simulating LSQC in a typical case where the maximum Beltrami coefficient norm is less than 0.4. Top row: ground truth mapping (left) and the magnitude of the BC of the ground truth mapping (right). Bottom row: results of SBN, M-MGN, and MSN, respectively. For each model, we report the average nodewise $\ell^2$ error with respect to the numerical LSQC solution and the number of flipped triangles (faces with negative Jacobian) above the corresponding subfigure.  Same format applies to Figure \ref{fig:large_mu} and Figure \ref{fig:large_distortion}.} \label{fig:smaller BC performance} 
\end{figure}
\begin{figure}
    \centering
    \includegraphics[width=\linewidth]{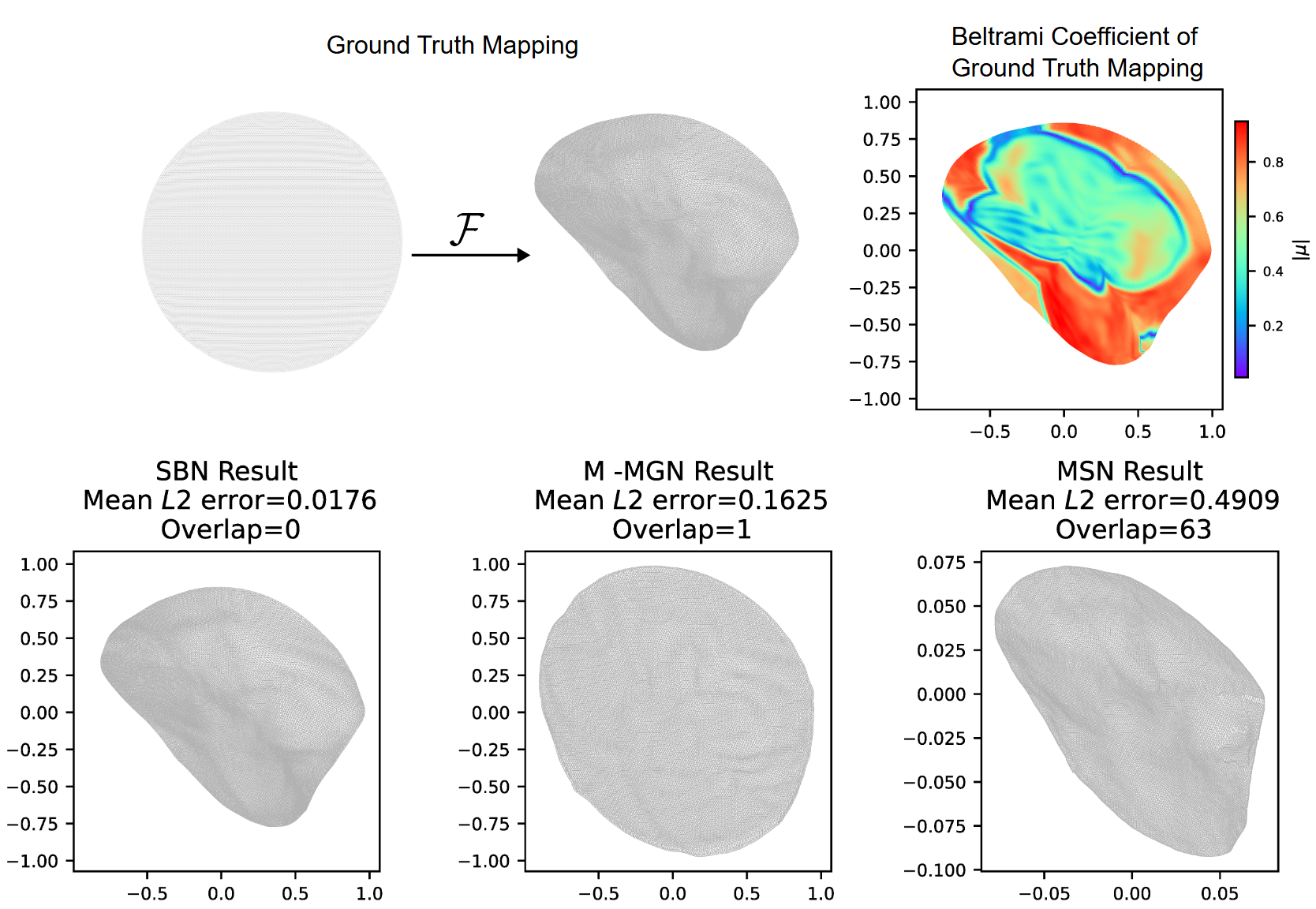}
    \caption{ Comparison of three models (SBN, M-MGN and MSN) for simulating LSQC when the prescribed Beltrami coefficients have a larger range.}
    \label{fig:large_mu}
\end{figure}

\begin{figure}
    \centering
    \includegraphics[width=\linewidth]{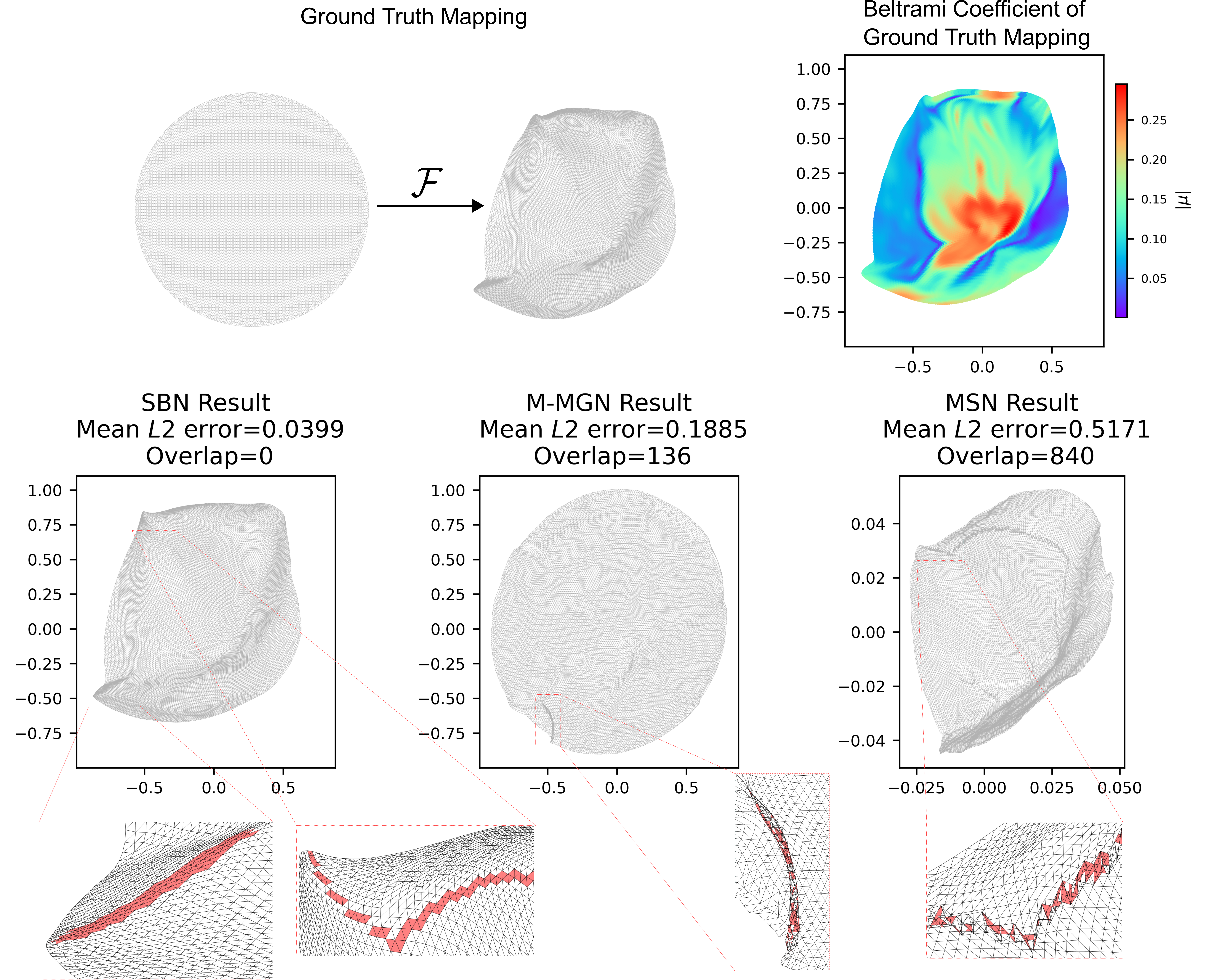}
    \caption{ Comparison of SBN, M-MGN, and MSN for simulating LSQC in a case where the desired mapping exhibits a large distortion. In addition, four red rectangular frames are drawn to highlight regions of interest: the red frames in the M-MGN and MSN results enclose areas containing flipped triangles, visualized as light red triangles in the zoomed-in views; the corresponding red frames in the SBN result mark the same regions on our mapping, where the triangles remain non-flipped and well-shaped.}
    \label{fig:large_distortion}
\end{figure}
A self-ablation study was conducted to investigate individual contributions of the multiscale message-passing mechanism and the mesh spectral layer (MSL) within our architecture.
Removing MSL yields Multiscale MeshGraphNet (M-MGN)  \cite{fortunato2022multiscale}; using only MSL gives Mesh Spectral Net (MSN). Across all three test cases (Figure \ref{fig:smaller BC performance}-Figure \ref{fig:large_distortion}), SBN significantly outperforms both ablations, confirming the synergy between multiscale message passing and spectral layers.

To further evaluate SBN, we compared it with M-MGN, MSN, and the numerical LSQC algorithm, running tests across 10,000 sets of BCs and pinned points. 
Two regimes were studied: (1) maximum BC $\ell^2$ norm less than 0.4, and (2) maximum BC $\ell^2$ norm between 0.4 and 1. We measured: non-injective triangles (faces with negative Jacobian), nodewise $\ell^2$ error, boundary $\ell^2$ error, and BC $\ell^2$ error (vs. target face-averaged BCs). Table \ref{tab:metrics_case1} and Table \ref{tab:metrics_case2} show that SBN achieves the lowest approximation errors among 3 models in both regimes, compared with numerical results.
\begin{table}[h]
\centering
\begin{subtable}[t]{0.48\textwidth}
\caption{ Case 1 :maximum BC norm is smaller than 0.4} \label{tab:metrics_case1}
\resizebox{\textwidth}{!}{%
\begin{tabular}{|l|c|c|c|} \hline \textbf{Method} & Nodewise $\ell^2$ Error & Boundary $\ell^2$ Error & BC $\ell^2$ Error \\ \hline
 SBN & \textbf{0.02239} & \textbf{0.03352} & \textbf{0.02306} \\ \hline
 M-MGN  & 0.09481 & 0.15194 & 0.07031 \\ \hline 
 MSN & 0.49009 & 0.73313 & 0.09514 \\ \hline \end{tabular}
}
\end{subtable}%
\hfill
\begin{subtable}[t]{0.48\textwidth}
\caption{ Case 2: max BC norm is between 0.4 and 1} \label{tab:metrics_case2}
\resizebox{\textwidth}{!}{%
\begin{tabular}{|l|c|c|c|} \hline \textbf{Method} & Nodewise $\ell^2$ Error & Boundary $\ell^2$ Error & BC $\ell^2$ Error \\ \hline
 SBN & \textbf{0.04671} & \textbf{0.07217} & \textbf{0.07679} \\ \hline 
 M-MGN & 0.23172 & 0.37515 & 0.19477 \\ \hline 
 MSN & 0.47382 & 0.70073 & 0.18091 \\ \hline \end{tabular}
}
\end{subtable}
\caption{ Performance of SBN, M-MGN and MSN in two cases, both averaged over 10,000 tests. Lower values indicate better performance.}
\end{table}
\subsection{Equiareal parameterization}
\label{subsec: DEQ}

We demonstrated SBN-Opt's superiority over numerical methods DEM\cite{choi2018density} and DEQ \cite{lyu2024bijective} in both area preserving and conformal distortion control. Problem formulation details are in Section \ref{sec: neural lsqc}. Our assessment was performed in both simply and multiply connected scenarios, systematically examining three aspects: the distribution of BCs, face density uniformity, and angle distortion of 2D parameterization result compared with 3D surface. The angle distortion of a triangle $T_i$ in the 2D parameterization is defined as $A_i=\frac{1}{3}\sum_{j=1}^{3}|\theta_{ij}-\tilde{\theta}_{ij}|$ where $\theta_{ij},\tilde{\theta}_{ij}$ are $j$-th angle of $T_i$ in the parameterization and its correspondence $U_i$ in the 3D surface, respectively.
\subsubsection{Simply Connected Cases} 
We first considered two simply connected examples: a synthetic “peak” model and a lion head surface. In both experiments, we compared our SBN-Opt framework with the DEM method \cite{choi2018density}.

For the peak model, we set the weights of the task loss $E_1$, the angle distortion regularizer $E_{\text{BC}}$, and the smoothness regularizer $E_{\text{smooth}}$ to 1, 5e-2 and 1e-3, respectively. The resulting parameterizations are shown in Figure \ref{fig:peak_mapping_result} with each face colored according to the density error $\frac{|\text{area}(T)-\text{area}(f(T))|}{|\text{area}(f(T))|}$. Indicated by Figure \ref{fig:peak_result_bc_angle_density_comparison}, our SBN-Opt produced an equiareal parameterization whose face-area distribution is highly concentrated around the target density, with the variance of the face density reduced from 0.7382 to 0.0052. At the same time, SBN-Opt achieved lower conformal distortion: the mean BC norm decreased from 0.2455 to 0.2418, and the mean angle distortion decreased from $29.79^\circ$ to $28.57^\circ$. Importantly, our mapping remained a homeomorphism without any overlaps while DEM's result contained 4 flipped triangles. The histograms in Figure \ref{fig:peak_result_bc_angle_density_comparison} illustrate that, compared with DEM, SBN-Opt notably suppressed extreme outliers in the density distribution while improving angular fidelity.

\begin{figure}
    \centering
    \includegraphics[width=\linewidth]{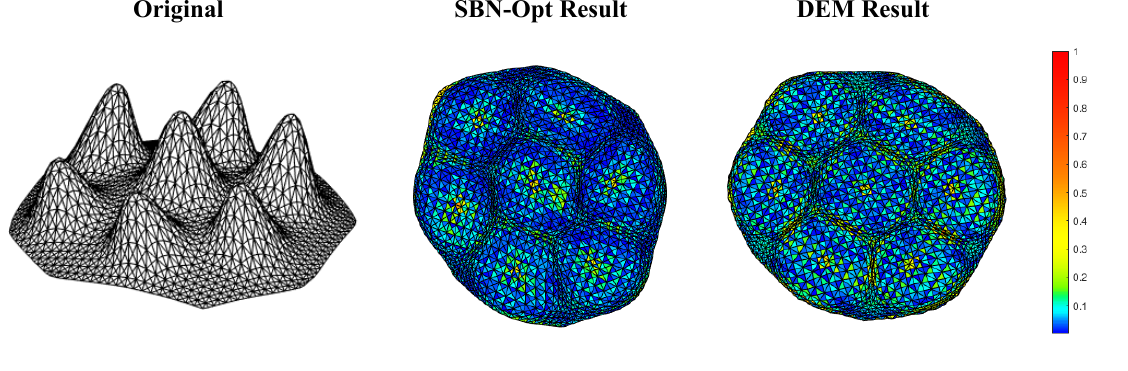}
    \caption{Equiareal parameterization of the "peak" model. From left to right: The original 3D mesh, the parameterization result from our SBN-Opt, and the result from the DEM algorithm. The 2D parameterization results are colored associated with the relative error $\frac{|\text{area}(T)-\text{area}(f(T))|}{|\text{area}(f(T))|}$, and the overall color of SBN-Opt's result is deeper than that of DEM's result, which means SBN-Opt's result is more area-preserving.}
    \label{fig:peak_mapping_result}
\end{figure}
\begin{figure}
    \centering
    \includegraphics[width=\linewidth]{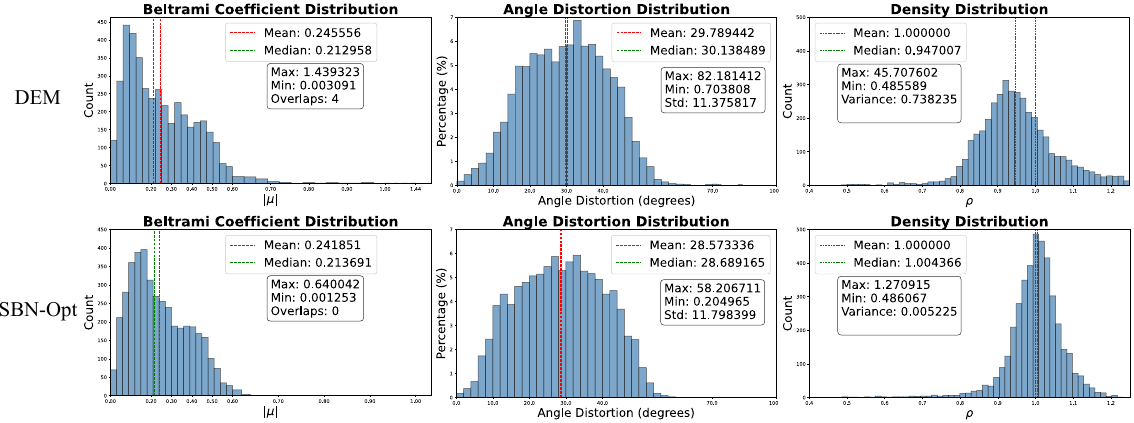}
    \caption{ 
    Statistical comparison of parameterization results for the peak model. Top row: results for DEM; bottom row: results for SBN-Opt.
    From left to right: (a) Histogram of BC magnitudes; (b) Histogram of per-face angle distortion (in degrees); (c) Histogram of the face density distribution. Note that the density histogram of DEM's result is visibly truncated to the value range of SBN-Opt's result so that the two distributions can be compared on the same scale. Same format applies to Figure \ref{fig:lion_result_bc_angle_density_comparison}, Figure \ref{fig:1-hole_BC_angle_density_comparsion} and Figure \ref{fig:2-hole_BC_angle_density_comparsion}. }
    \label{fig:peak_result_bc_angle_density_comparison}
\end{figure}

In the second example, the weights of task loss $E_1$, angle distortion loss $E_{\text{BC}}$ and smoothness loss $E_{\text{smooth}}$ were 1, 1e-2 and 1e-2, respectively. 
Figure \ref{fig:lion_mapping_result} shows the parameterization results, and Figure \ref{fig:lion_result_bc_angle_density_comparison} reports the corresponding statistics. Compared with DEM, SBN-Opt substantially reduced conformal distortion: the average BC norm was reduced by 26.0\% (from 0.4599 to 0.3402), and the mean angle distortion decreased by 10.5\% (from $23.24^\circ$ to $20.81^\circ$). The area preserving performance was also greatly improved, with the variance of the face density reduced to 0.0074, whereas there were faces with significantly distorted densities in the DEM's result. Again, no overlaps were observed in our result. Overall, on simply connected surfaces, SBN-Opt consistently produced maps that are both closer to the target density and more conformal than DEM.
\begin{figure}
    \centering
    \includegraphics[width=\linewidth]{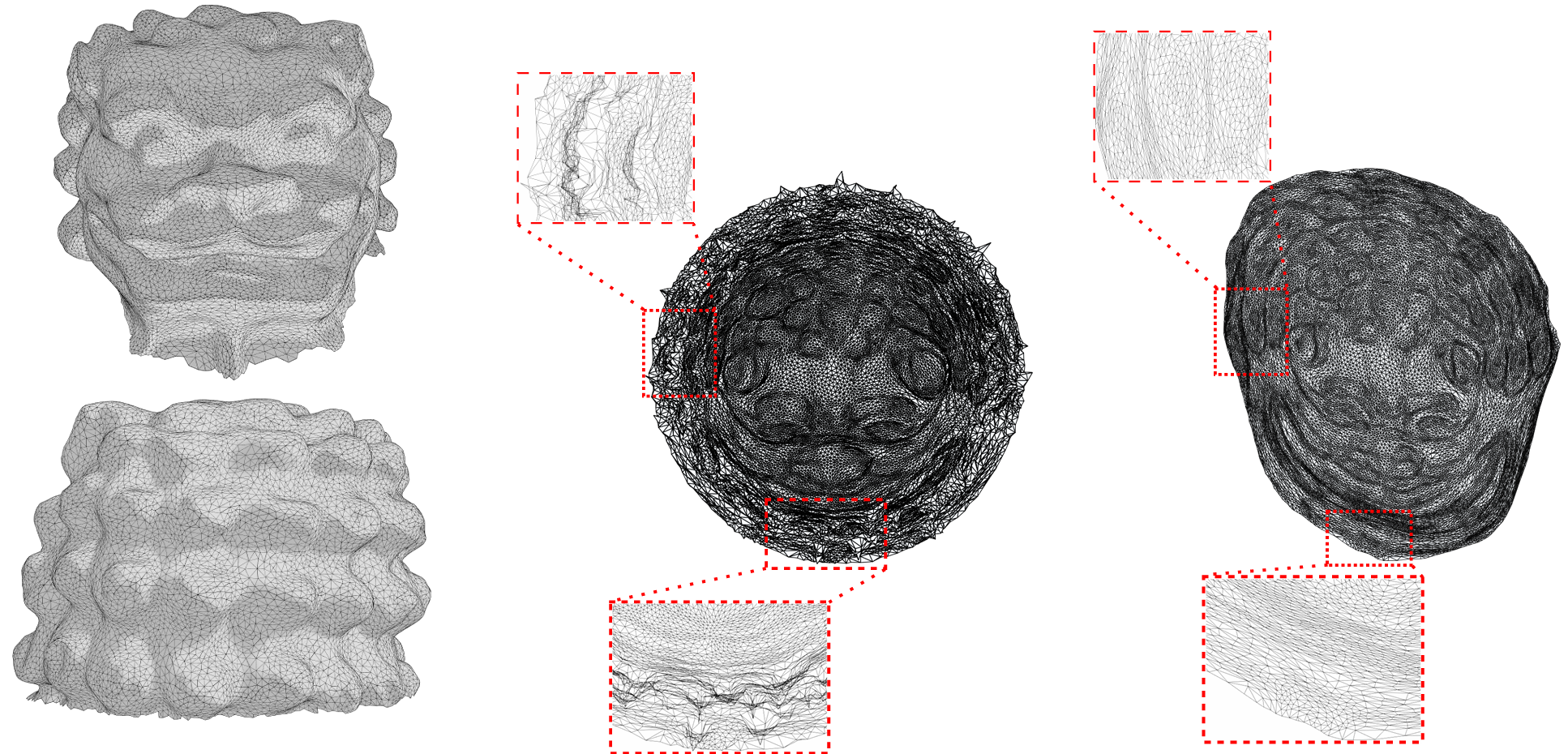}
    \caption{ Left: original lion head mesh shown from front (top) and back (bottom).
Middle: DEM's density-equalizing parameterization. Two red dashed rectangles mark regions are zoomed in where the remeshed triangles are severely distorted and contain local fold-overs (flipping triangles).
Right: SBN-Opt's equiareal parameterization. The red dashed rectangles indicate the regions corresponding to those highlighted in the DEM result; the zoom-in views show that the triangles in these areas are non-flipping and well shaped.}
    \label{fig:lion_mapping_result}
\end{figure}
\begin{figure}
    \centering
    \includegraphics[width=\linewidth]{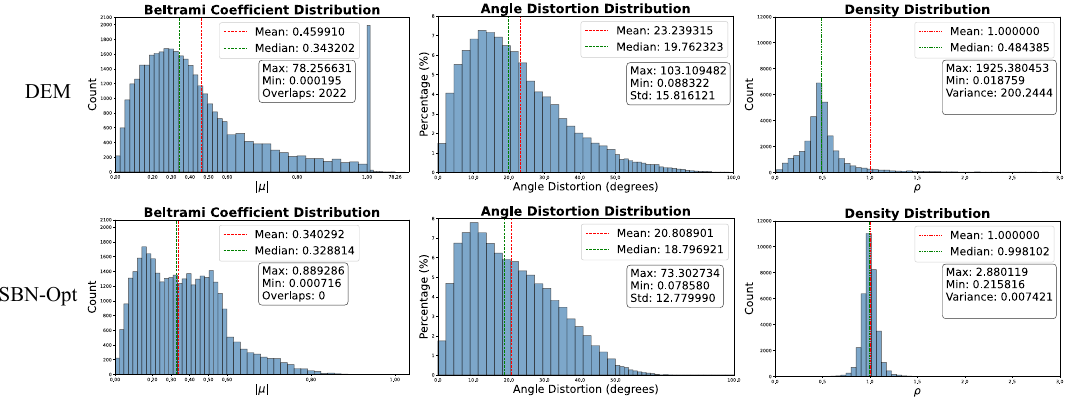}
    \caption{ Statistical comparison of parameterization results for the lion head model.}
    \label{fig:lion_result_bc_angle_density_comparison}
\end{figure}
\begin{figure}
    \centering
    \includegraphics[width=\linewidth]{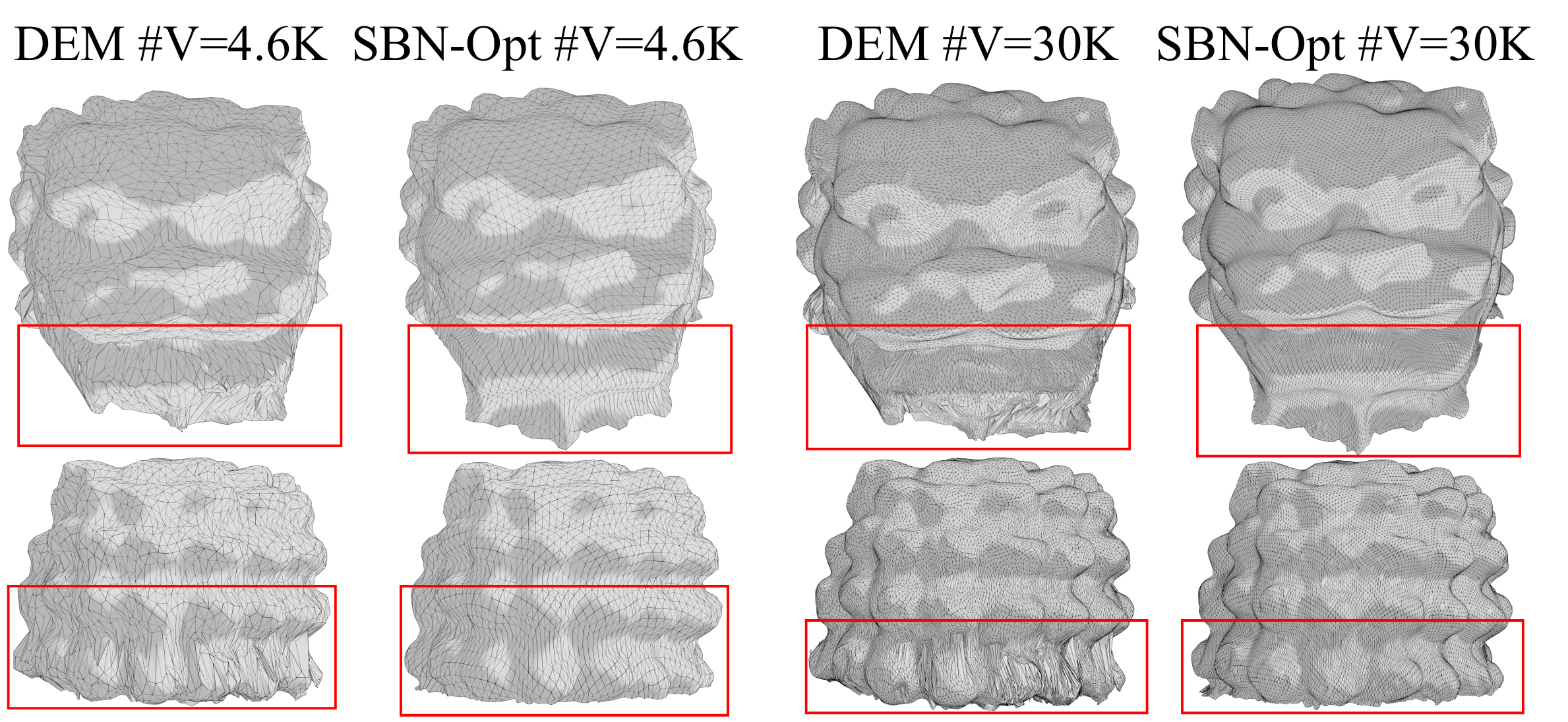}
    \caption{ Remeshing comparison for the lion head surface at two resolutions using DEM and SBN-Opt. From left to right: remeshed result by DEM with 4.6K vertices, remeshed result by SBN-Opt with 4.6K vertices, remeshed result by DEM with 30K vertices, and remeshed result by SBN-Opt with 30K vertices.  For each method and resolution, the top row shows a front view of the remeshed surface and the bottom row shows a back view. Red rectangular boxes mark regions of interest near the lower part of the mesh, which implies that SBN-Opt produces more regular and uniformly sized elements in the highlighted regions than DEM.}
    \label{fig:lion_remeshed_mesh_result_comparison}
\end{figure}
\subsubsection{Multiply Connected Cases}
We next tested SBN-Opt on multiply connected surfaces and compared it with the DEQ method\cite{lyu2024bijective}. We considered two examples: a human face with one hole and a human face with two holes.

For the human face with one hole, we set the weights of $E_1$, $E_{\text{BC}}$ and  $E_{\text{smooth}}$ to 1, 1e-2 and 1e-2, respectively. The parameterization obtained by SBN-Opt and DEQ are shown in Figure \ref{fig:1-hole_BC_angle_density_comparsion}, with quantitative statistics in Figure \ref{fig:1-hole_BC_angle_density_comparsion}. Our method achieved a noticeable improvement in conformal quality: the average BC norm was reduced by 15.8\% (from 0.1416 to 0.1192), and the mean angle distortion decreased by 17.0\% (from $8.208^\circ$ to $6.810^\circ$). At the same time, area preserving performance was significantly enhanced: the variance of the face density dropped from 7.996e-2 for DEQ to 9.47e-4 for SBN-Opt.

In the second example, the weights of task loss $E_1$, $E_{\text{BC}}$ and $E_{\text{smooth}}$ were 1, 1.5e-1 and 0, respectively. The visual comparison is given in Figure \ref{fig:2-hole_mapping_result} and the statistical comparison between our result and DEQ's is shown in Figure \ref{fig:2-hole_BC_angle_density_comparsion}. Again, SBN-Opt yielded a parameterization with substantially reduced conformal distortion: the average Beltrami norm decreased by 15.5\% (from 0.1305 to 0.1095), and the mean angle distortion decreased by 13.5\% (from $7.244^\circ$ to $6.270^\circ$). In terms of area distribution, our method produced a density that is much closer to the target, with its variance reduced from 31.46 to 1.517e-2.

Across both multiply connected examples, SBN-Opt consistently delivered parameterizations that better balance area preserving and conformal distortion than DEQ. By directly optimizing over Beltrami coefficients, our framework avoided the severe density outliers and excessive angular distortion that appear in the numerical baselines.

\begin{figure}
\subfloat[ Equiareal parameterization for the human face with one hole.]   
  {
      \centering
		\includegraphics[width=0.45\textwidth]{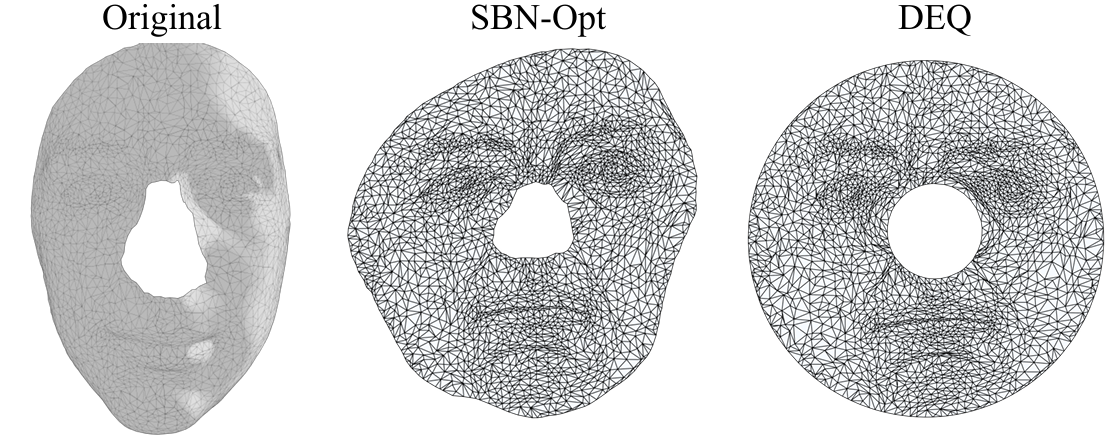}
    \label{fig:1-hole_mapping_result}
  }
  \quad
  \subfloat[ Equiareal parameterization for the human face with two holes.]
  {
      \centering
		\includegraphics[width=0.45\textwidth]{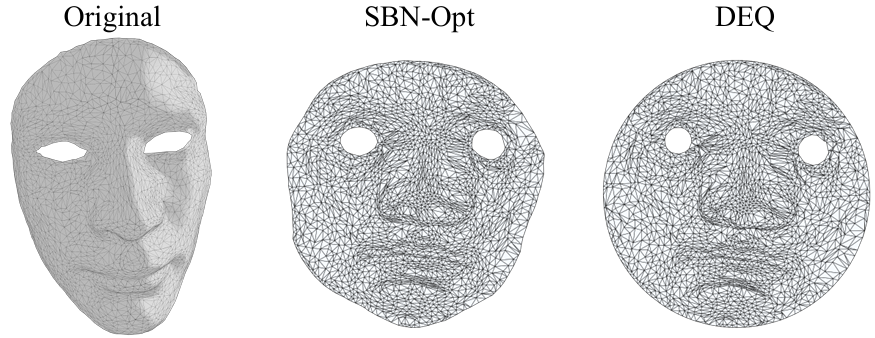}
    \label{fig:2-hole_mapping_result}
  }
  \caption{ (Left) The original 3D mesh; (Middle) The 2D parameterization result from our SBN-Opt; (Right) The remeshed 3D surface generated from the parameterization derived by SBN-Opt.}   
\end{figure}
\begin{figure}
    \centering
    \includegraphics[width=\linewidth]{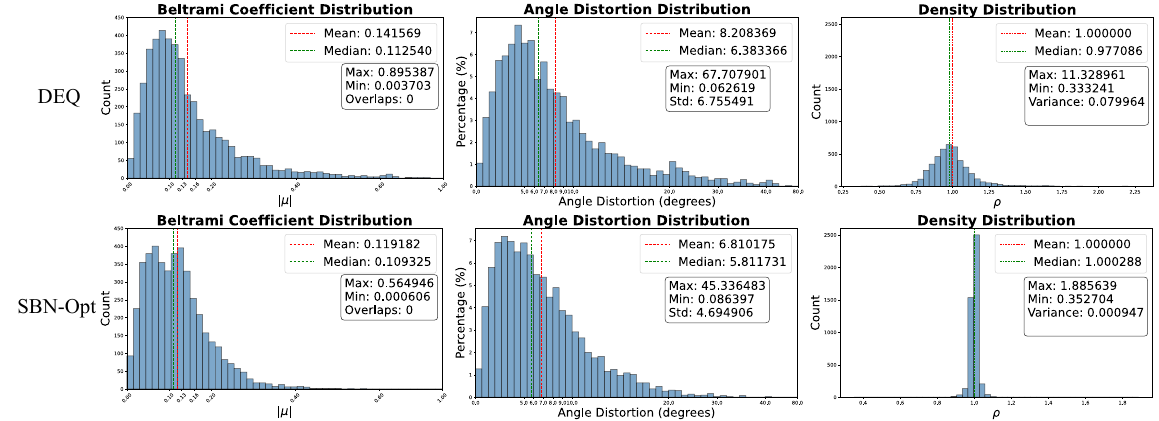}
    \caption{ Statistical comparison of parameterization results for the human face with one hole.}
    \label{fig:1-hole_BC_angle_density_comparsion}
\end{figure}
\begin{figure}
    \centering
    \includegraphics[width=\linewidth]{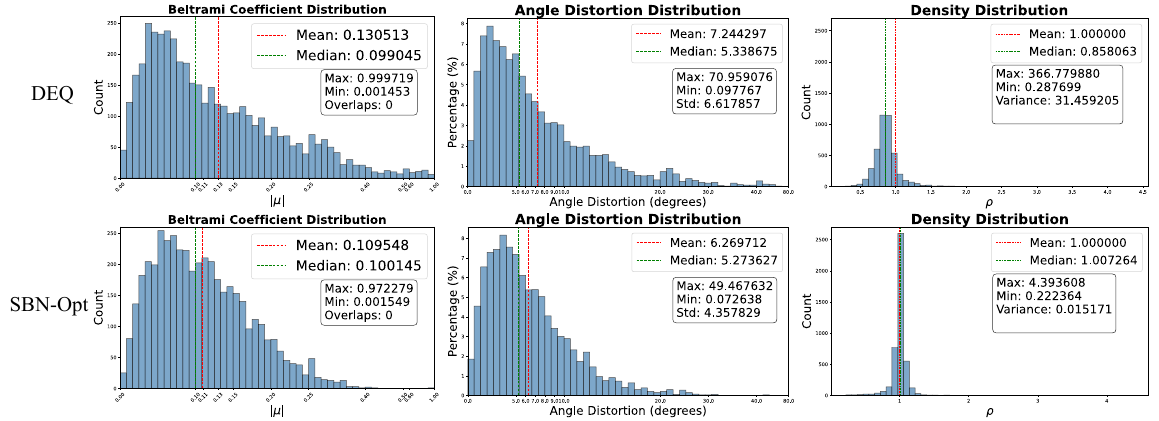}
    \caption{ Statistical comparison of parameterization results for the human face with two holes.}
    \label{fig:2-hole_BC_angle_density_comparsion}
\end{figure}
\begin{figure}
\subfloat[ Remeshed Result Comparison for the human face with one hole.]   
  {
      \centering
		\includegraphics[width=0.46\textwidth]{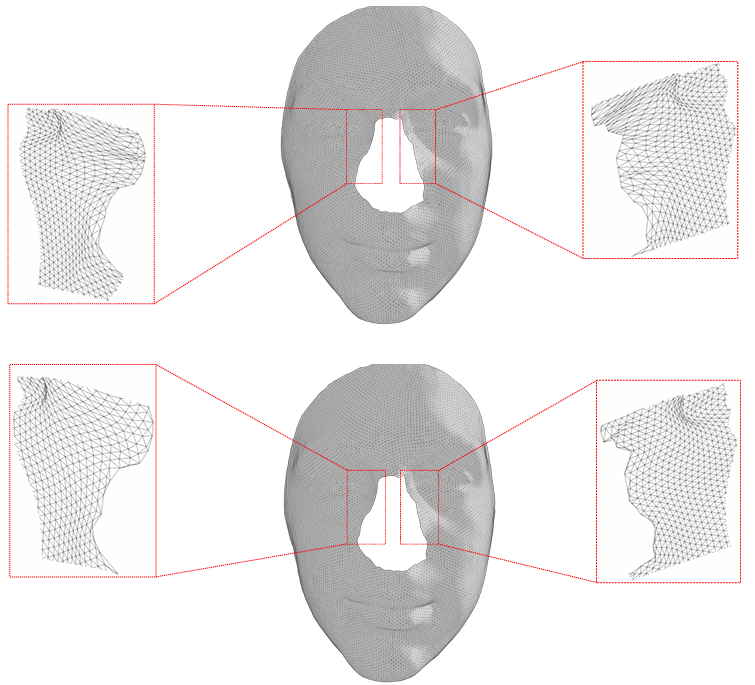}
    \label{fig:1-hole_remeshed_result_comparison}
  }
  \quad
  \subfloat[Remeshed Result Comparison for the human face with two holes.]
  {
      \centering
		\includegraphics[width=0.46\textwidth]{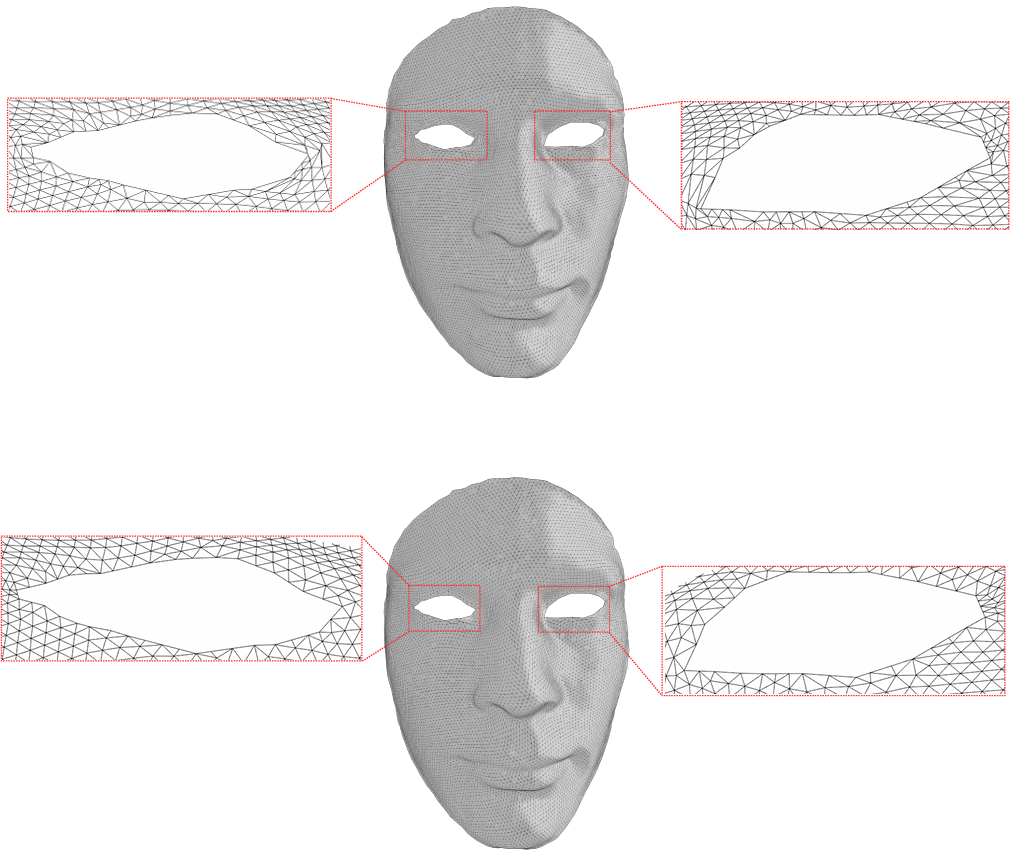}
    \label{fig:2-hole_remeshed_result_comparison}
  }
  \caption{ Remeshing comparison for human face surfaces with one (left) and two holes (right), respectively. For each example, the top row shows the remeshed result obtained from the baseline method (DEQ), and the bottom row shows the remeshed result obtained from the SBN-Opt–based parameterization. Around each face, red dashed rectangles indicate zoom-in windows placed along the eye or nose boundaries; the inset views display the local triangle patterns in these regions.
The side-by-side zoomed comparisons reveal that the SBN-Opt based remeshing yields more regular, well-shaped, and uniformly sized elements.}   
\end{figure}
\subsubsection{Application to Remeshing}
A high-quality equiareal parameterization is not only an end in itself, but also a good tool for downstream geometry processing. In this subsection, we demonstrate how the equiareal parameterization produced by SBN-Opt can be used to generate high-quality remeshed 3D surfaces, and how the improved parameterization quality directly translates into better remeshing results.

Let $S$ be a surface, $f:S\rightarrow \mathbb{C}$ be the equiareal map, and considering a set of uniformly distributed points $\mathcal{P}$ on $f(S)$, we can triangulate them and obtain a new triangulation $\mathcal{T}$. Then, using the inverse mapping $f^{-1}$, we can interpolate $\mathcal{P}$ onto $S$ and obtain a remeshed triangulation  $(f^{-1}(\mathcal{P}),\mathcal{T})$ of the surface $S$.To assess how the quality of the parameterization impacts the remeshing, we applied this procedure to three surfaces in the last experiments and obtain: lion head with 4.6K vertices (see Figure \ref{fig:lion_remeshed_mesh_result_comparison} left), lion head with 30K vertices (see Figure \ref{fig:lion_remeshed_mesh_result_comparison} right), human face with one hole (see Figure \ref{fig:1-hole_remeshed_result_comparison}) and human face with two holes (see Figure \ref{fig:2-hole_remeshed_result_comparison}). For each surface, we constructed remeshed results from the parameterizations produced by SBN-Opt and by the baseline DEM/DEQ methods. In Figure \ref{fig:lion_remeshed_mesh_result_comparison}, Figure \ref{fig:2-hole_remeshed_result_comparison} and Figure \ref{fig:2-hole_remeshed_result_comparison}, we highlight several zoomed-in regions with red triangular frames. Visually, the remeshed surfaces derived from our SBN-Opt parameterizations exhibited more regular triangle shapes and more uniform sampling density, while the DEM/DEQ-based remeshed surfaces contained noticeably skewed and stretched triangles in regions where the underlying parameterizations suffered from larger distortion.

To quantify these observations, in addition to the angle distortion between 2D uniform meshes and 3D remeshed surfaces, we evaluated two more complementary metrics on the remeshed 3D surfaces:
\begin{enumerate}
    \item \textbf{Size variation}. On the remeshed 3D surface \(\widetilde{\mathcal{M}}\), we computed the standard deviation of the triangle areas: \[ \delta_{\text{size}} = \text{std}\bigl(\text{Area}(T)\bigr), \] where \(\text{Area}(T)\) is the area of each face \(T\) in \(\widetilde{\mathcal{M}}\). A smaller \(\delta_{\text{size}}\) reflects a more uniform distribution of triangle sizes, which is consistent with a better equiareal parameterization.
    \item \textbf{Shape variation} For each triangle \(T_i\) in \(\widetilde{\mathcal{M}}\), let \(e_{ij}\) be the length of its three edges. We measured the deviation of the normalized edge lengths from the equilateral configuration: \[ R_i = \sum_{j=1}^3 \left|\frac{e_{ij}}{\sum_{j=1}^3 e_{ij}} - \frac{1}{3}\right|. \] The shape variation is then defined as \[ \delta_{\text{shape}} = \text{mean}(R_i). \] A smaller \(\delta_{\text{shape}}\) value indicates better control of angular distortion in the remeshing.
\end{enumerate}
For each of the four remeshing examples, we computed these three metrics for the remeshed surfaces obtained from SBN-Opt and from DEM/DEQ, and summarized the results in the Table \ref{tab:remesh_metric_comparison}. 
\begin{table}[t]
\centering 
\caption{ The performance of remeshed Results by SBN-Opt and DEM. We recorded the mean angle distortion, size variation and shape variation of each result.} \label{tab:remesh_metric_comparison}
\resizebox{\textwidth}{!}{
\begin{tabular}{|l|c|c|c|c|} \hline Subject & Method & Mean angle distortion & Size variation & Shape variation \\ \hline
\multirow{2}{*}{Lion Head(4.6K)}& DEM & 21.621582 & 0.2355 & 0.6022 \\ \cline{2-5} 
& SBN-Opt & \textbf{21.35} & \textbf{0.1975} & \textbf{0.1182} \\ \cline{1-5}
\multirow{2}{*}{Lion Head(30K)} & DEM & 23.132975 & 0.2219 & 0.1960 \\ \cline{2-5}  
& SBN-Opt & \textbf{22.30} & \textbf{0.2023} & \textbf{0.01620} \\ \cline{1-5}
\multirow{2}{*}{Human face 1 hole(10K)} & DEQ & 7.586 & 0.07717 & 2.2819e-5 \\ \cline{2-5}  
& SBN-Opt & \textbf{7.139} & \textbf{0.07566} & \textbf{1.398e-5} \\ \cline{1-5}
\multirow{2}{*}{Human face 2 holes(10K)} & DEQ & 6.924 & 0.07162 &  2.5748e-5 \\ \cline{2-5}  
& SBN-Opt & \textbf{6.755} & \textbf{0.07124} & \textbf{2.000e-5} \\  
\hline
\end{tabular}
}
\end{table}

Across all cases, the SBN-Opt–based remeshes consistently exhibited lower average angle distortion, smaller area variation, and reduced shape variation, compared with the DEM/DEQ-based remeshes. These quantitative improvements aligned with the visual comparisons in Figure \ref{fig:lion_remeshed_mesh_result_comparison}, Figure \ref{fig:1-hole_remeshed_result_comparison} and Figure \ref{fig:2-hole_remeshed_result_comparison}: triangles generated from our parameterizations were more regular and more uniformly sized, particularly in regions of high curvature or around holes where DEM/DEQ tends to produced elongated or crowded elements.

In summary, these experiments confirmed that a parameterization with lower angular distortion and better area preserving is crucial for high-quality remeshing. By explicitly optimizing over Beltrami coefficients and directly controlling local distortion, SBN-Opt not only outperforms numerical methods at the parameterization level, but also yields superior remeshed 3D surfaces in downstream applications.
\subsection{Inconsistent surface registration}
\label{subsec:ISR}
In this section, we tested our SBN-Opt on the problem of inconsistent surface registration. In the application of inconsistent surface registration, our framework demonstrates consistently superior performance when compared to the landmark-dependent LSQC method proposed by Qiu et al.\cite{qiu2020inconsistent}. Across three distinct registration tasks, our method achieves mappings with lower angular distortion and comparative or slightly better feature alignment. For details of the problem formulation, see Section \ref{sec: neural lsqc}. 

The first example is the registration of tooth surfaces. In Figure \ref{fig:MH_AMNH_3 mesh_highlighted}, the first two columns show the conformal parameterizations of the moving surface and the target static tooth surface, respectively. There were 4 regions selected on each surface for alignment and are highlighted in Figure \ref{fig:MH_AMNH_3 mesh_highlighted}.  In this example, the weights of intensity matching loss $E_{\text{I}}$, landmark matching loss $E_{\text{pc}}$, angle distortion loss $E_{\text{BC}}$ and smoothness loss $E_{\text{smooth}}$ were 1, 1e-1, 5e-1 and 1e-3, respectively. We compared our result with that of the method in Qiu\cite{qiu2020inconsistent} which is denoted by LSQC in the following. And to perform Qiu's method, we selected four landmarks, one in each selected region on both surfaces.
As shown in Figure \ref{fig:MH_AMNH_bc_angle_intensity_comparison}, 
When registering two tooth surfaces using four selected feature regions, our method reduced the angular distortion, measured by the average $\ell^2$ norm of BCs, by 12.2\% (from 0.1135 to 0.09966). Additionally, the angle distortion histogram for SBN-Opt was more concentrated around lower values and exhibits a lighter tail than that of LSQC, with mean angle distortion reduced by 10.7\% ($6.943^\circ$ to $6.202^\circ$). More significantly, the registration accuracy was substantially improved, with the average L1 intensity mismatch error on the overlapping region dropping by 16.9\% (from 0.0794 to 0.0660). The registration result in the 2D domain is shown in the third column of Figure \ref{fig:MH_AMNH_3 mesh_highlighted}. The green mesh is the deformed mesh from the source mesh under the bijective deformation map with the blue highlighted regions overlapping orange highlighted regions in the target mesh. The intersection region of the two meshes is the region of correspondence amongst the two tooth surfaces, see Figure \ref{fig:MH_AMNH_overlap_region_in_two_mesh}.
\begin{figure}
    \centering
    \includegraphics[width=\linewidth]{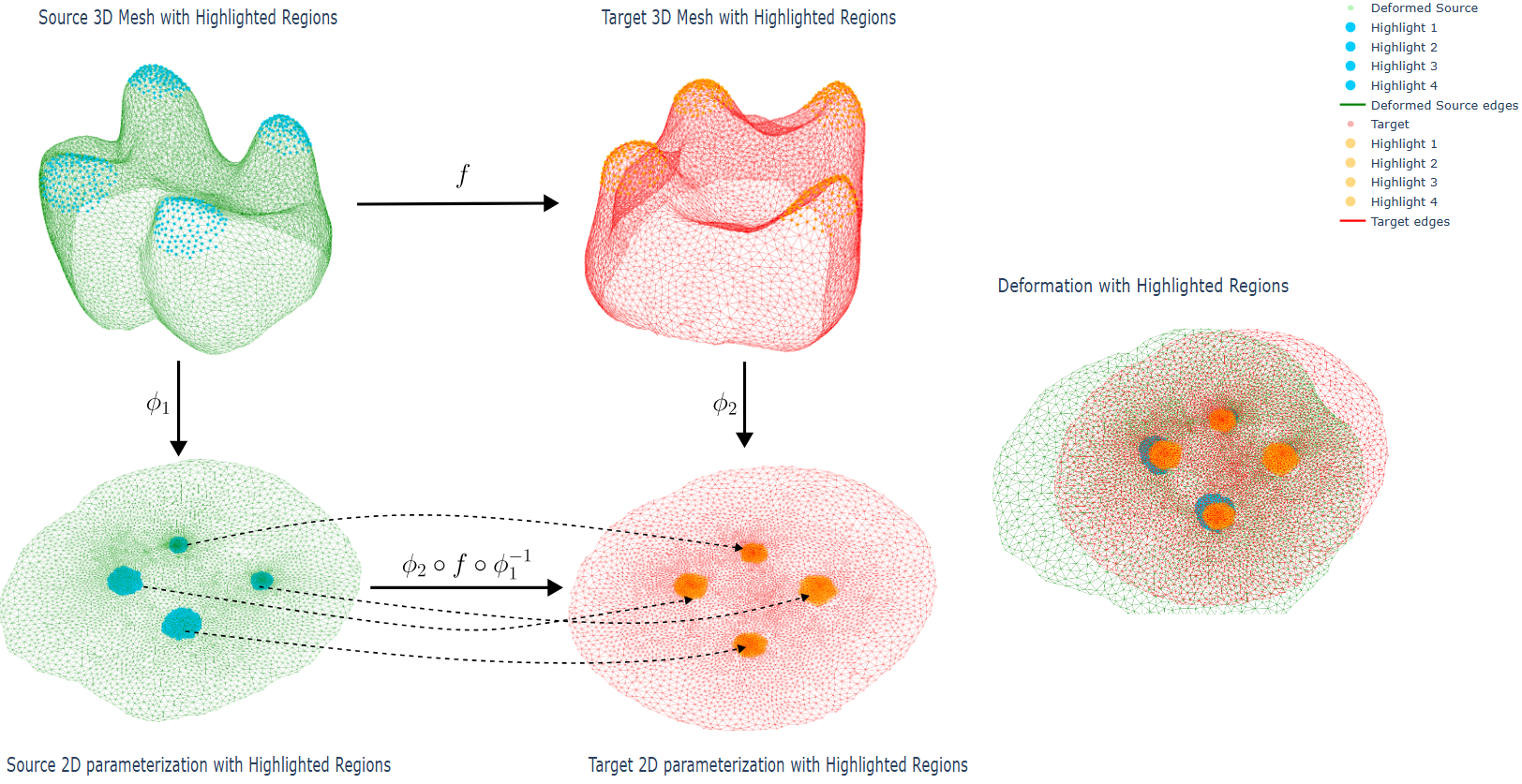}
    \caption{ The first two columns show the conformal parametrizations of the moving and static tooth surfaces, respectively. The region correspondences in the 2D parameter domains are also displayed. The registration result in the 2D domain is shown in the rightmost column.}
    \label{fig:MH_AMNH_3 mesh_highlighted}
\end{figure}
\begin{figure}
    \centering
    \includegraphics[width=\linewidth]{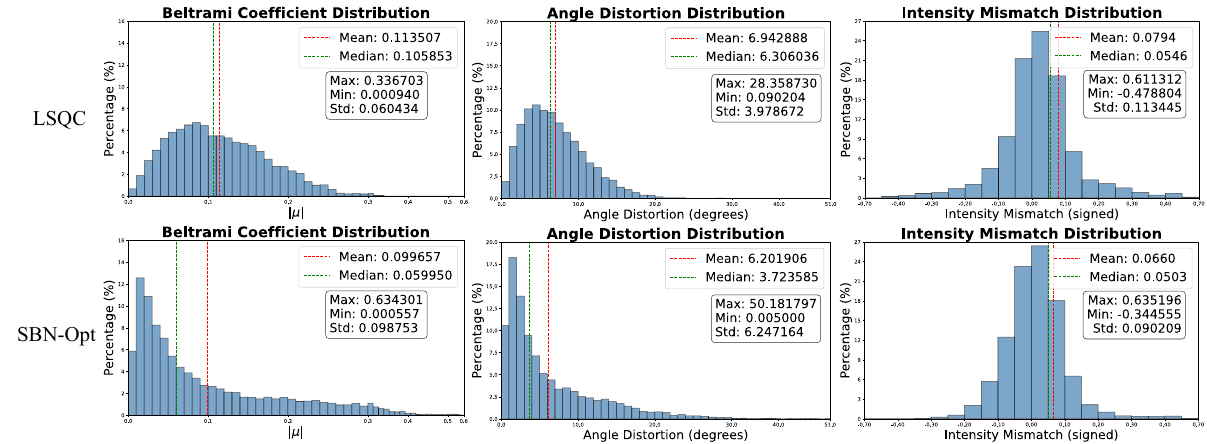}
    \caption{ Statistical comparison of registration results for the  tooth surfaces (Figure \ref{fig:MH_AMNH_3 mesh_highlighted}, Figure \ref{fig:MH_AMNH_overlap_region_in_two_mesh}). This figure format is used for all subsequent comparisons (Figure  \ref{fig:part2_part1_mapping_bc_angle_intensity_comparison} and Figure \ref{fig:source_025_target_120_bc_angle_intensity_comparison}). Top row: results for LSQC; bottom row: results for SBN-Opt.  From left to right: (a) Histogram of BC magnitudes; (b) Histogram of per-face angle distortion (in degrees); (c) Histogram of the intensity mismatch error distribution. }
    \label{fig:MH_AMNH_bc_angle_intensity_comparison}
\end{figure}
\begin{figure}
    \centering
    \includegraphics[width=\linewidth]{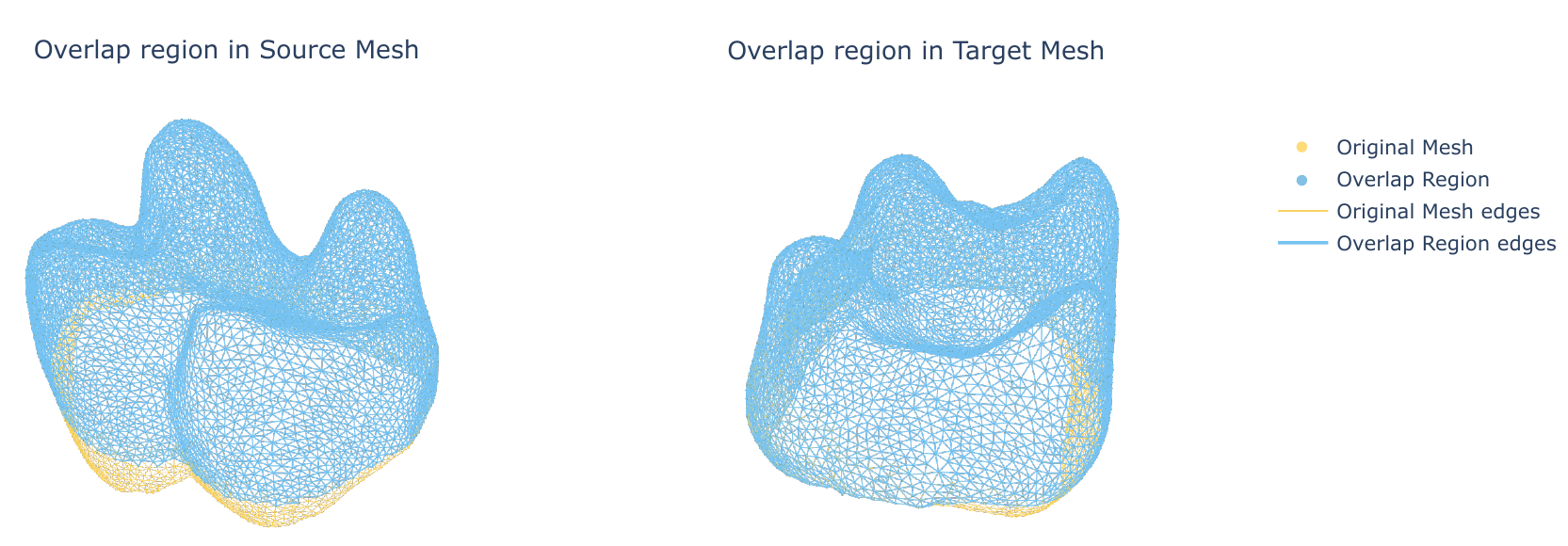}
    \caption{ Visualization of the computed correspondence for the tooth face registration. The identified region of correspondence is highlighted in blue on both the source surface (left) and the target surface (right). The complete surfaces are shown in yellow. This figure illustrates the extensive overlap found by the registration algorithm.}
    \label{fig:MH_AMNH_overlap_region_in_two_mesh}
\end{figure}

The second example is the registration of two partial human faces, which are obtained from the FaMoS Dataset\cite{TMPEH:CVPR:2023}.  Each subject was cropped so that the nose and mouth regions remain while the eyes and surrounding regions were partially removed. In Figure \ref{fig:part1_part2_3 mesh_highlighted}, the first two columns show the conformal parameterizations of the moving and target partial faces, respectively. There were 3 regions (bridge of nose, nose and mouth) selected on each surface for alignment and are highlighted in Figure \ref{fig:part1_part2_3 mesh_highlighted}. And to perform Qiu's method, we selected seven landmarks, two in bridge of nose, three in nose and two in mouth, respectively. In this example, the weights of intensity matching loss $E_{\text{I}}$, landmark matching loss $E_{\text{pc}}$, angle distortion loss $E_{\text{BC}}$ and smoothness loss $E_{\text{smooth}}$ were 1, 5e-1, 1e-2 and 1e-3, respectively.  We compared our result with that of LSQC. As shown in Figure \ref{fig:part2_part1_mapping_bc_angle_intensity_comparison}, our approach demonstrates its most significant advantage in reducing distortion. The average BC norm was reduced by an impressive 23.7\% (from 0.1439 to 0.1098). Consistent with this, the mean angle distortion was also decreased by 23.5\% (from $9.123^\circ$ to $6.983^\circ$)  The average intensity mismatch error was also reduced by 26.0\% (from 0.0327 to 0.0242). Qualitatively, our method produced a more reasonable correspondence; the resulting overlap region in the source mesh correctly avoided the eye region, an anatomical feature not present in the target partial face. In contrast, the LSQC method produced an overlap that incorrectly extended to the eye region, see Figure \ref{fig:part1_part2_overlap_region_in_two_mesh}. The registration result in the 2D domain is shown in the rightmost of Figure \ref{fig:part1_part2_3 mesh_highlighted}. The green mesh is the deformed mesh from the source mesh under the bijective deformation map with the blue highlighted regions overlapping orange highlighted regions in the target mesh. 
\begin{figure}
    \centering
    \includegraphics[width=\linewidth]{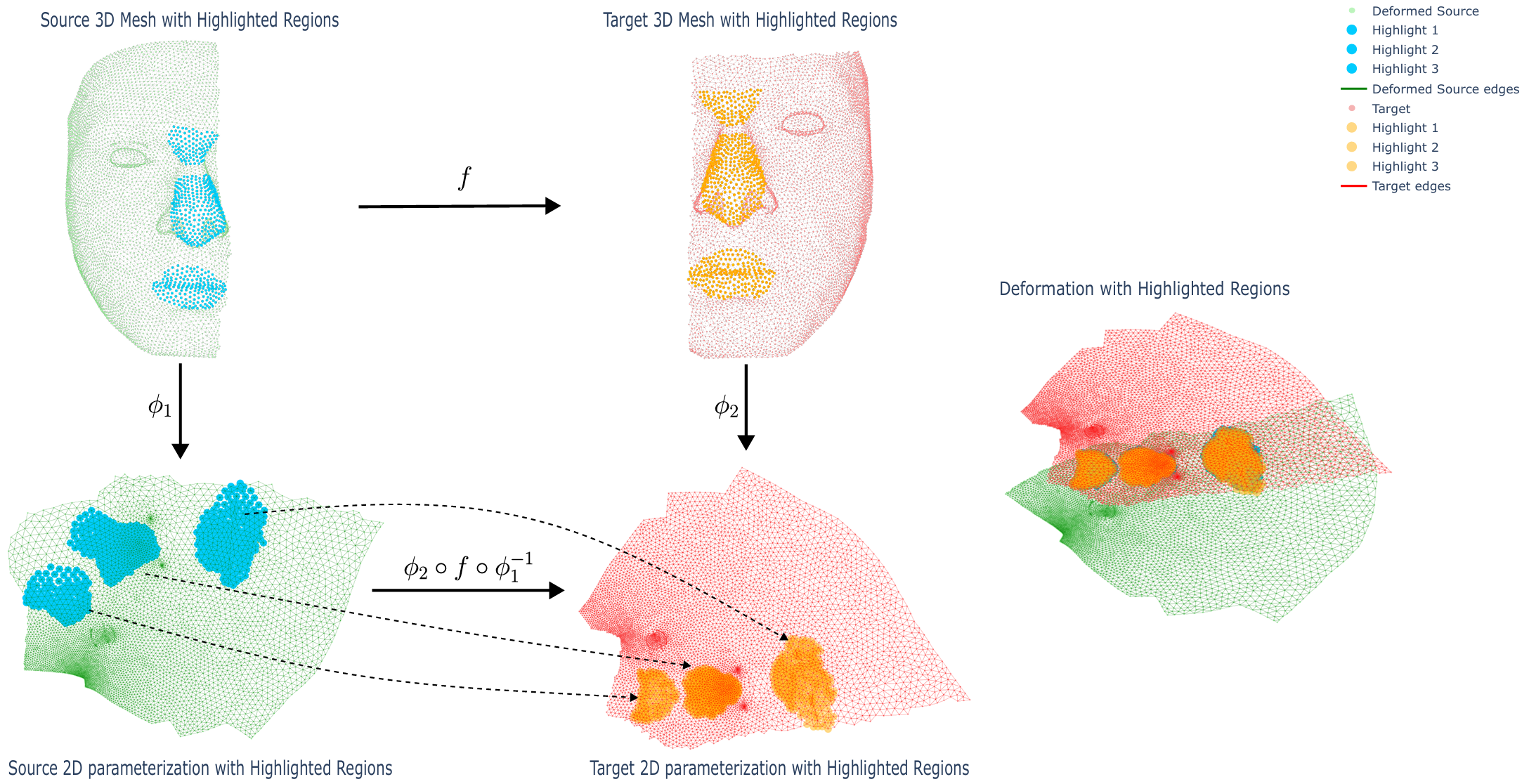}
    \caption{ The first two columns show the conformal parametrizations of the moving and static partial human faces, respectively. The region correspondences in the 2D parameter domains are also displayed. The registration result in the 2D domain is shown in the rightmost column.}
    \label{fig:part1_part2_3 mesh_highlighted}
\end{figure}
\begin{figure}
    \centering
    \includegraphics[width=\linewidth]{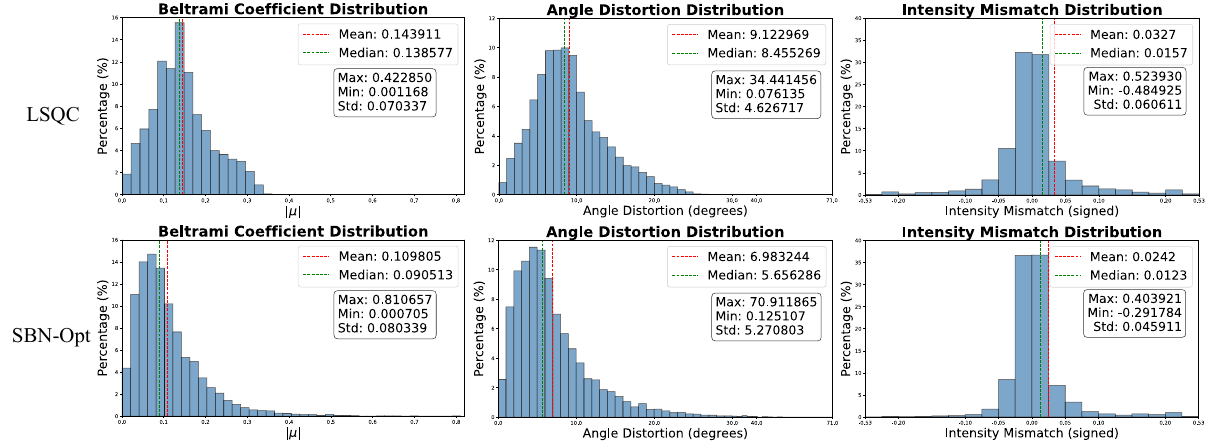}
    \caption{ Quantitative comparison for the partial human face registration shown in Figure \ref{fig:part1_part2_3 mesh_highlighted} The meaning of each histogram corresponds to that in Figure \ref{fig:MH_AMNH_bc_angle_intensity_comparison}. }
    \label{fig:part2_part1_mapping_bc_angle_intensity_comparison}
\end{figure}
\begin{figure}
    \centering
    \includegraphics[width=\linewidth]{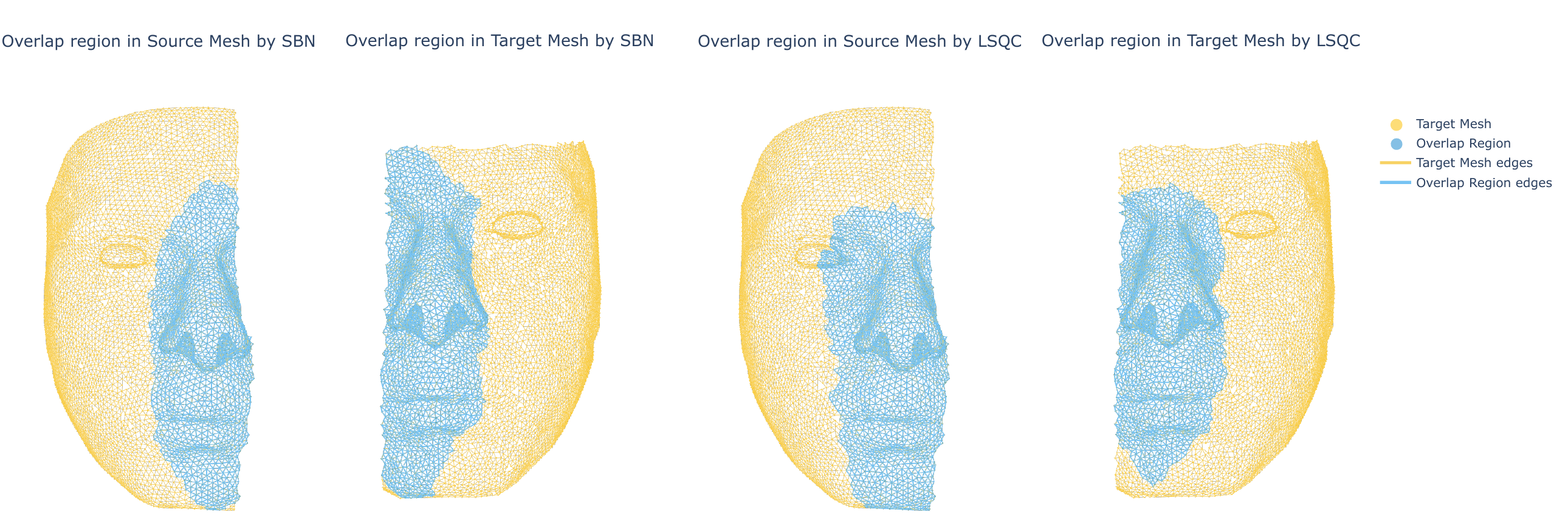}
    \caption{Overlap region comparison for partial human face registration with mouths closed. From left to right: overlap region visualized on the source mesh by SBN-Opt, on the target mesh by SBN-Opt, on the source mesh by LSQC, and on the target mesh by LSQC. Same Format is used in Figure \ref{fig:source_025_target_120_overlap_region_in_two_mesh}. In all subfigures, the full facial surface is shown in yellow, while the overlap region identified by the registration is rendered in blue. Compared with the LSQC's result,  the overlap computed by SBN-Opt is more anatomically plausible, which correctly excludes the eye region on the source mesh and yields a more localized correspondence around the nose and mouth on the target mesh.}
    \label{fig:part1_part2_overlap_region_in_two_mesh}
\end{figure}

The third example further tested our method on partial human faces whose mouths are open, so that both the source and target surfaces are multiply connected. In this setting, the parameter domains contained interior boundary components corresponding to the mouth openings. Figure \ref{fig:source_025_target_120_3 mesh_highlighted} shows the conformal parameterizations of the moving and static open-mouth faces in the first two columns, with highlighted feature regions used to guide registration, and the resulting 2D registration in the rightmost column. A key advantage of SBN-Opt is that it can operate directly on these multiply connected parameter domains without any topological modification. By contrast, the LSQC method of Qiu \cite{qiu2020inconsistent} assumes simply connected domains; to apply it here, one must first manually fill the mouth openings on both surfaces, convert them into simply connected surfaces, and then register the filled meshes. We used the same loss components $E_{\text{I}}$, $E_{\text{pc}}$, $E_{\text{BC}}$ and $E_{\text{smooth}}$ as in the previous examples, with weights 1,5e-1,3e-1 and 0, respectively. Figure \ref{fig:source_025_target_120_bc_angle_intensity_comparison} reports the quantitative comparison between SBN-Opt and LSQC. SBN-Opt again yielded a more concentrated BC histogram with the mean value dropped by 37.1\% (from 0.1078 to 0.0678), and the angle distortion was more concentrated around lower values with the mean value reduced by 30.8\% (from $7.465^\circ$ to $5.169^\circ$). Meanwhile, the intensity matching result is comparable compared with that of LSQC. The corresponding 3D overlap regions are visualized in Figure \ref{fig:source_025_target_120_overlap_region_in_two_mesh}. On both the source and target meshes, the overlap computed by SBN-Opt adhered closely to the facial features and correctly tracked the mouth opening boundaries. In contrast, the LSQC-based registration, which is working on filled, simply connected surrogates of the surfaces, exhibits less faithful alignment around the bridge of nose region and the overlap cannot correctly avoid the eye region in the target mesh. This example highlights that SBN-Opt naturally extends to multiply connected inconsistent surface registration without requiring domain filling or topological surgery, while simultaneously achieving better distortion control, competitive or superior matching accuracy and more reasonable registration result.

\begin{figure}
    \centering
    \includegraphics[width=\linewidth]{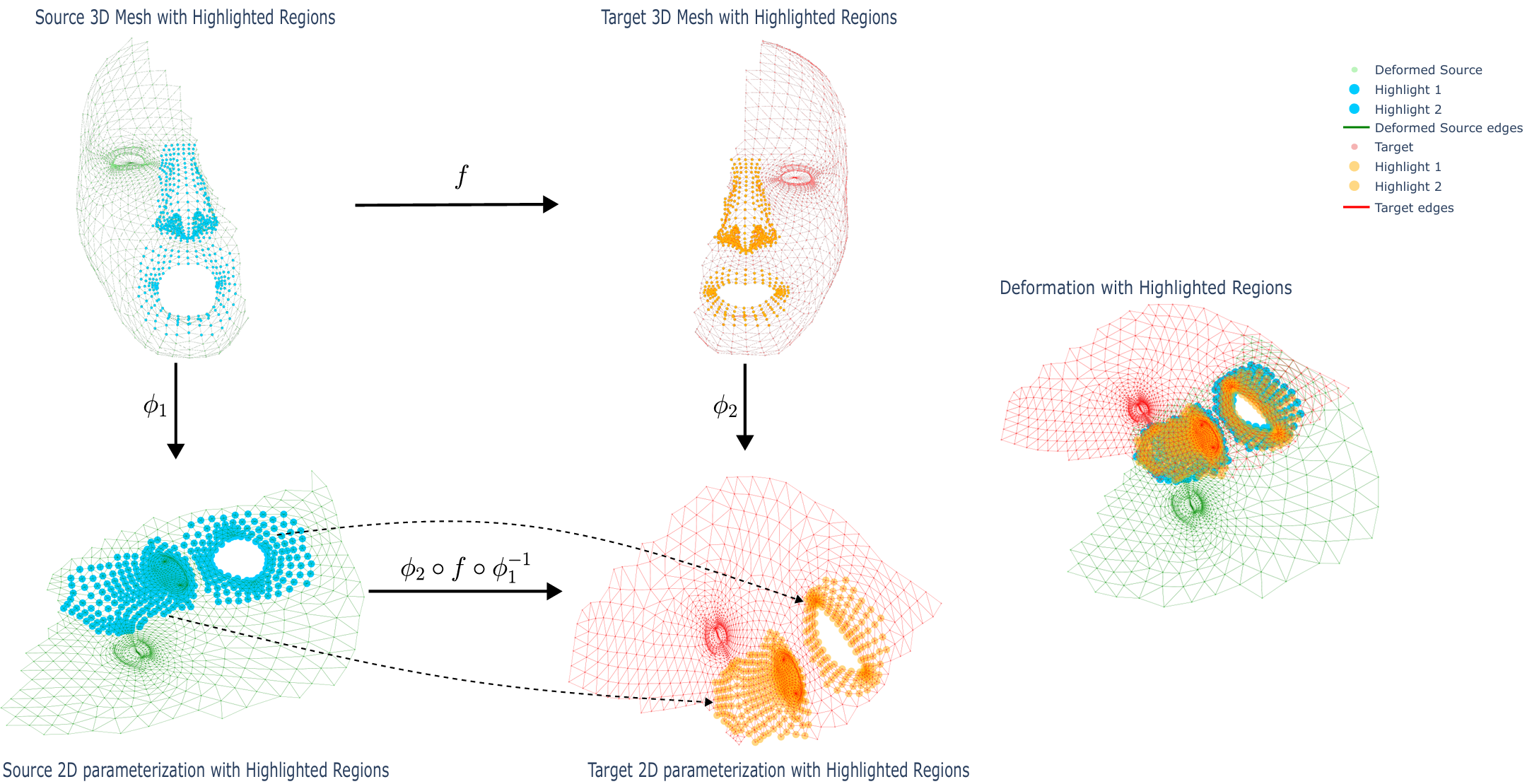}
    \caption{ The first two columns show the conformal parametrizations of the moving and static partial human faces, respectively. The region correspondences in the 2D parameter domains are also displayed. The registration result in the 2D domain is shown in the rightmost column.}
    \label{fig:source_025_target_120_3 mesh_highlighted}
\end{figure}
\begin{figure}
    \centering
    \includegraphics[width=\linewidth]{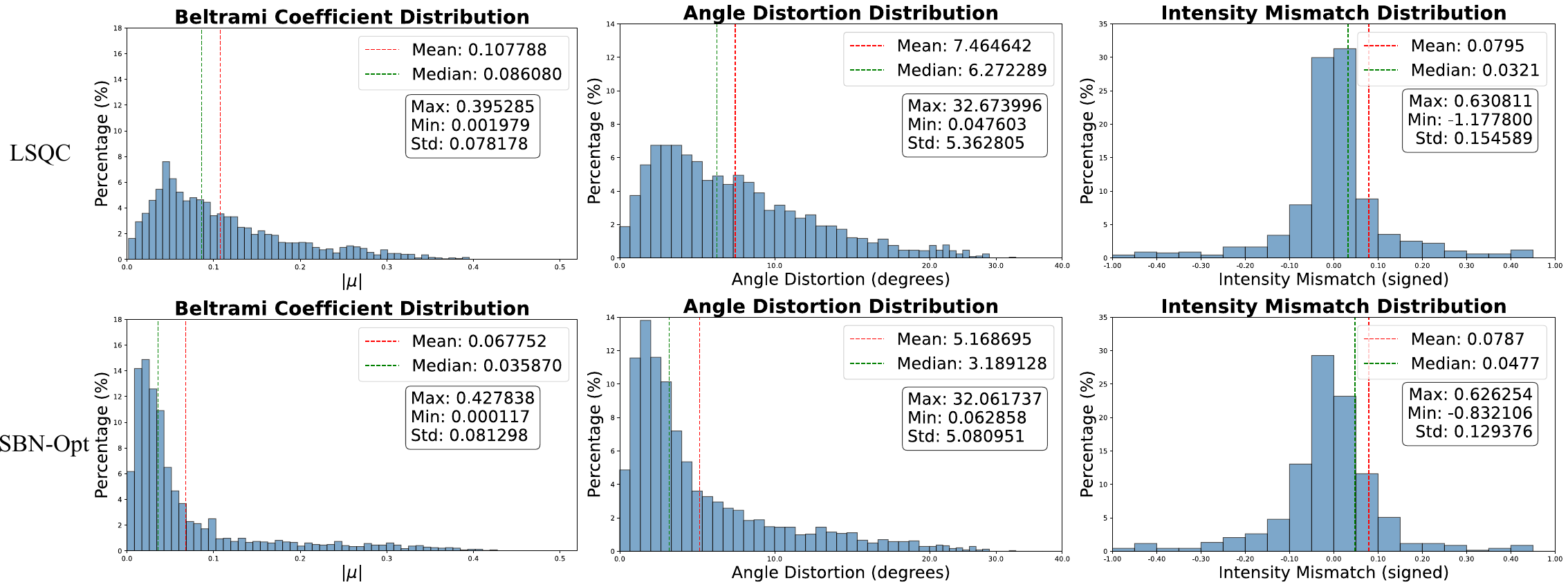}
    \caption{ Quantitative comparison for the partial human face registration shown in Figure \ref{fig:part1_part2_3 mesh_highlighted}. The meaning of each histogram corresponds to that in Figure \ref{fig:MH_AMNH_bc_angle_intensity_comparison}. }
    \label{fig:source_025_target_120_bc_angle_intensity_comparison}
\end{figure}
\begin{figure}
    \centering
    \includegraphics[width=\linewidth]{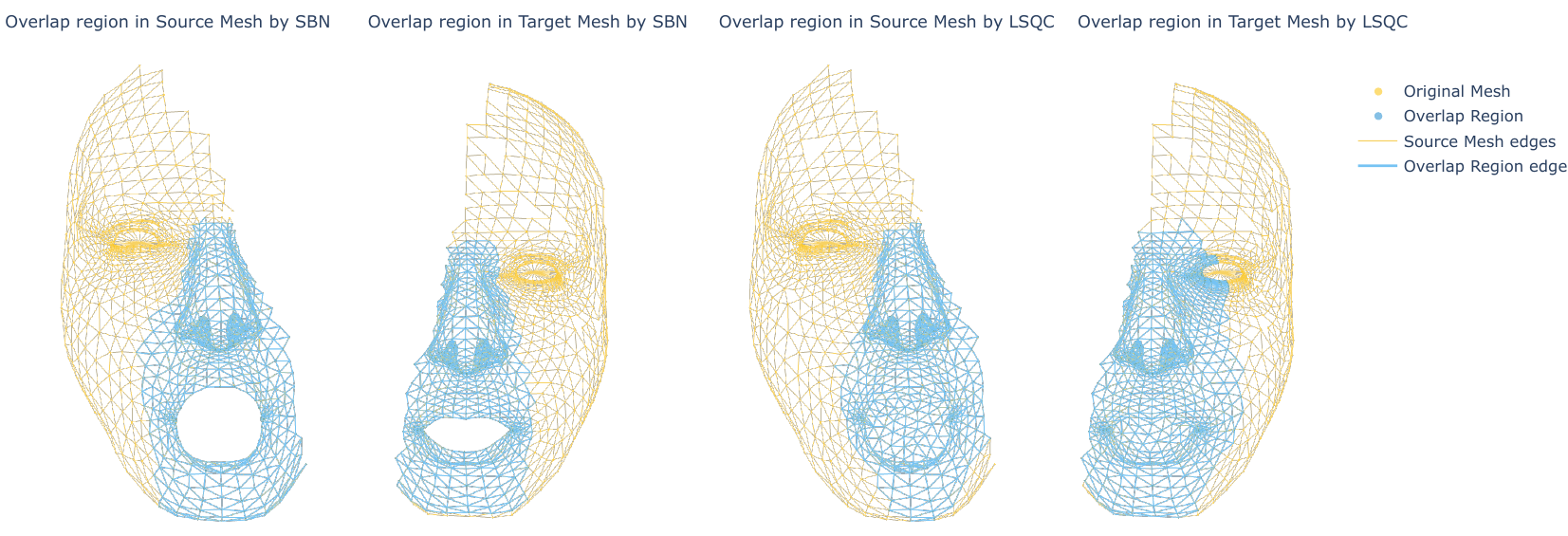}
    \caption{ Visualization of the computed correspondence for the partial human face registration with mouths open. The identified region of correspondence by our method is highlighted in blue on both the source surface (left) and the target surface (center). The complete surfaces are shown in yellow. Compared with the overlap identified by LSQC on the target mesh (right), computed overlap by our method correctly avoids mapping the eye region from the source, which is not present in the target.}
    \label{fig:source_025_target_120_overlap_region_in_two_mesh}
\end{figure}

%% file: sections/sec10-conclusion.tex
\section{Conclusions}
\label{sec: conclusions}
In this work, we introduced SBN-Opt, a neural framework for free-boundary diffeomorphism optimization built upon the Spectral Beltrami Network (SBN). By rigorously establishing the mathematical properties of the Least-squares Quasiconformal (LSQC) energy including existence and uniqueness of minimizers, similarity-invariance, and resolution independence, we laid a solid theoretical foundation for using LSQC in free-boundary settings. Additionally, we analyzed its stability and differentiability, which enables gradient-based optimization via LSQC. Due to the heavy computation cost, we propose the Spectral Beltrami Network. SBN is a neural surrogate that, for the first time, faithfully approximates the free-boundary LSQC solver. Its unique architecture, which synergistically combines multiscale message-passing with mesh spectral layers, overcomes the limitations of prior methods by effectively capturing both local and global geometric dependencies.

A distinctive advantage of our approach is that the two‐pinned‐point condition is treated as part of the input to SBN itself; consequently, both the Beltrami coefficients and the pinned points are optimization variables and the corresponding constraints remain differentiable. This contrasts with many existing quasiconformal geometry based and other diffeomorphic registration methods, which require externally prescribed landmark constraints in order to obtain nontrivial solutions, and thus cannot directly handle general free-boundary problems in a gradient-based pipeline.

Leveraging this design, SBN-Opt reframes the free-boundary diffeomorphism problem as an optimization over the space of admissible Beltrami coefficients and pinned-point conditions, enabling explicit control of local geometric distortion and boundary shapes without artificial scaffolds or fixed landmarks.

Extensive experiments on equiareal parameterization and inconsistent surface registration demonstrate the clear superiority of our approach. SBN-Opt consistently produces mappings with significantly lower distortion and higher accuracy compared to established numerical algorithms. By bridging the gap between the mathematical rigor of quasiconformal geometry and the computational flexibility of deep learning, SBN-Opt provides a powerful and theoretically-grounded tool. This work not only addresses long-standing challenges in surface mapping but also paves the way for future advances in geometry processing, medical imaging, and scientific visualization.